%% file: asyn_Q_pessimism.tex
\documentclass[english]{article}
\usepackage[T1]{fontenc}
\usepackage[latin9]{inputenc}
\usepackage{geometry}
\usepackage{babel}
\usepackage{verbatim}
\usepackage{float}
\usepackage{mathtools}
\usepackage{amsmath}
\usepackage{amssymb}
\usepackage[unicode=true,
 bookmarks=false,
 breaklinks=false,pdfborder={0 0 1},colorlinks=false]
 {hyperref}
\hypersetup{
 colorlinks,linkcolor=red,anchorcolor=blue,citecolor=blue}

\makeatletter

\usepackage{babel}

\usepackage{mathtools}

\usepackage{cite}\usepackage{amsthm}\usepackage{dsfont}\usepackage{array}\usepackage{mathrsfs}\usepackage{comment}\onecolumn
\usepackage{natbib}
\usepackage{color}\usepackage{babel}

\allowdisplaybreaks

\usepackage{enumitem}
\setlist[itemize]{leftmargin=1.5em}
\setlist[enumerate]{leftmargin=1.5em}

\usepackage{babel}
\usepackage[linesnumbered,ruled,vlined]{algorithm2e}
\usepackage{algorithmic}
\usepackage{arydshln}

\newcommand{\cA}{\mathcal{A}}

\newcommand{\cS}{{\mathcal{S}}}

\DeclareMathOperator{\ind}{\mathds{1}}  % Indicator

\newcommand{\indep}{\perp\!\!\!\!\perp} 
\newcommand{\mymid}{\,|\,}

\numberwithin{equation}{section}

\definecolor{yxc}{RGB}{255,0,0}
\definecolor{yjc}{RGB}{125,0,0}
\definecolor{cm}{RGB}{0,0,200}
\definecolor{yly}{RGB}{0,150,0}

\makeatother

\begin{document}
\theoremstyle{plain} \newtheorem{lemma}{\textbf{Lemma}} \newtheorem{prop}{\textbf{Proposition}}\newtheorem{theorem}{\textbf{Theorem}}\setcounter{theorem}{0}
\newtheorem{corollary}{\textbf{Corollary}} \newtheorem{assumption}{\textbf{Assumption}}
\newtheorem{example}{\textbf{Example}} \newtheorem{definition}{\textbf{Definition}}
\newtheorem{fact}{\textbf{Fact}} \newtheorem{condition}{\textbf{Condition}}\theoremstyle{definition}

\theoremstyle{remark}\newtheorem{remark}{\textbf{Remark}}\newtheorem{claim}{\textbf{Claim}}\newtheorem{conjecture}{\textbf{Conjecture}}
\title{The Efficacy of Pessimism in Asynchronous Q-Learning\footnotetext{Corresponding author: Yuxin Chen (Email: \texttt{yuxinc@wharton.upenn.edu}).}}
\author{Yuling Yan\thanks{Department of Operations Research and Financial Engineering, Princeton
		University, Princeton, NJ 08544, USA; Email: \texttt{\{yulingy,jqfan\}@princeton.edu}.} \and Gen Li\thanks{Department of Statistics and Data Science, Wharton School, University
		of Pennsylvania, Philadelphia, PA 19104, USA; Email: \texttt{\{ligen,yuxinc\}@wharton.upenn.edu}.} \and Yuxin Chen\footnotemark[2] \and Jianqing Fan\footnotemark[1]}

\maketitle
\input{abstract.tex}


\tableofcontents{}

\input{intro.tex}\input{formulation.tex}\input{algorithm.tex}\input{related_work.tex}\input{discussion.tex}

\section*{Acknowledgements}

Y.~Chen is supported in part by the Alfred P.~Sloan Research Fellowship, the grants AFOSR YIP award FA9550-19-1-0030,
 ONR N00014-19-1-2120, 
ARO YIP award W911NF-20-1-0097, NSF CCF-2221009, CCF-1907661, IIS-1900140 and IIS-2218773.  
J.~Fan is supported in part by the ONR grant N00014-19-1-2120, 
the NSF grants DMS-1662139, DMS-1712591, DMS-2052926, DMS-2053832,
and the NIH grant 2R01-GM072611-15. Part of this work was done while
Y.~Yan, G.~Li and Y.~Chen were visiting the Simons Institute for
the Theory of Computing.

\appendix
\input{notations.tex}

\input{hoeffding.tex}\input{variance_reduced.tex}

\bibliographystyle{apalike}
\bibliography{bibfileRL}

\end{document}

%% file: abstract.tex
\begin{abstract}
This paper is concerned with the asynchronous form of Q-learning,
which applies a stochastic approximation scheme to Markovian data
samples. Motivated by the recent advances in offline reinforcement
learning, we develop an algorithmic framework that incorporates the
principle of pessimism into asynchronous Q-learning, which penalizes
infrequently-visited state-action pairs based on suitable lower confidence
bounds (LCBs). This framework leads to, among other things, improved
sample efficiency and enhanced adaptivity in the presence of near-expert
data. Our approach permits the observed data in some important scenarios to cover only partial
state-action space, which is in stark contrast to prior theory that
requires uniform coverage of all state-action pairs. When coupled
with the idea of variance reduction, asynchronous Q-learning with
LCB penalization achieves near-optimal sample complexity, provided
that the target accuracy level is small enough. 
In comparison, prior works were suboptimal in terms of the dependency on the effective horizon even when i.i.d.~sampling is permitted. 
Our results deliver
the first theoretical support for the use of pessimism principle in
the presence of Markovian non-i.i.d.~data. 
\end{abstract}

%% file: intro.tex
\section{Introduction}

The asynchronous form of Q-learning, which is a stochastic approximation
paradigm that applies to Markovian non-i.i.d.~samples, has found
applicability in an abundance of reinforcement learning (RL) applications
\citep{watkins1992q,tsitsiklis1994asynchronous,jaakkola1994convergence,even2003learning}.
The input data takes the form of a Markovian sample trajectory induced
by a policy called the \emph{behavior policy}; in each time, asynchronous
Q-learning only updates the Q-function estimate of a single state-action pair
along the trajectory rather than updating all pairs at once --- and hence
the terminology ``asynchronous'' \citep{tsitsiklis1994asynchronous,bertsekas2003parallel}.
This classical algorithm has the virtue of being off-policy, allowing
one to learn the optimal policy even when the behavior policy is suboptimal.
Recent years have witnessed a resurgence of interest in understanding
the performance of asynchronous Q-learning, due to a shift of attention
from classical asymptotic analysis to the non-asymptotic counterpart.
By and large, non-asymptotic results bear important and clear implications
for the impacts of salient parameters (e.g., model capacity, horizon
length) in large-dimensional RL problems. 

\subsection{Motivation}

A central consideration in modern RL applications is data efficiency:
the limited availability of data samples places increasing demands
on sample-efficient RL solutions, and in turn, calls for reexamining
classical algorithms like Q-learning. When it comes to asynchronous
Q-learning, recent theoretical advances have led to sharpened sample
complexity analyses \citep{li2021q,qu2020finite,li2021sample}.
For concreteness, consider a $\gamma$-discounted infinite-horizon
Markov decision process (MDP) and a stationary behavior policy: asynchronous
Q-learning provably yields $\varepsilon$-accuracy as soon as the
sample size exceeds the order of\footnote{Here, the higher-order term $o\big(\frac{1}{\varepsilon^{2}}\big)$
depends also on other parameters of the MDP and of the sample trajectory (e.g., the mixing time, the discount
factor, and $\mu_{\mathsf{min}}$).} \citep{li2021q}
\begin{equation}
\frac{1}{\mu_{\mathsf{min}}(1-\gamma)^{4}\varepsilon^{2}}+o\bigg(\frac{1}{\varepsilon^{2}}\bigg)\label{eq:prior-Q-theory}
\end{equation}
modulo some log factor, where $\mu_{\mathsf{min}}$ stands for the
minimum occupancy probability of the sample trajectory over all state-action
pairs. While this bound \eqref{eq:prior-Q-theory} is tight in a general
sense for vanilla Q-learning, two issues immediately spring into mind. 

\begin{itemize}

\item \emph{Uniform coverage vs.~partial coverage.} The factor $1/\mu_{\mathsf{min}}$
in \eqref{eq:prior-Q-theory} imposes a firm requirement on uniform
coverage of the state-action space, namely, every state-action pair
needs to be visited sufficiently often in order to guarantee reliable
learning. Nevertheless, it is not uncommon for a behavior policy to
provide only partial coverage of the state-action space; for instance,
a behavior policy might elect to rule out several actions that are
clearly underperforming. In truth, partial coverage of the state-action
space results in $\mu_{\mathsf{min}}=0$, thus making the general bound \eqref{eq:prior-Q-theory}
vacuous in this case.

\item \emph{Lack of adaptivity to expert data}. The general bound \eqref{eq:prior-Q-theory}
falls short of reflecting the quality of the sample trajectory (except
for a general uniform coverage parameter $\mu_{\mathsf{min}}$). For instance, if the behavior
policy is adopted by an ``expert'' who is already aware of which
actions are (close to) optimal, then such expert data could
be more informative than a general sample trajectory with the same
$\mu_{\mathsf{min}}$. It is therefore desirable for an algorithm
to adapt automatically to the quality of the data, in the hope of
achieving sample size saving when expert data is available. 
\end{itemize}

\subsection{Main contributions}

This paper seeks to make asynchronous Q-learning adaptive to near-expert
data, allowing for partial coverage of the state-action space in some important scenarios. A key
idea that has been recently proposed to accommodate partial coverage in the presence of near-expert data
is the principle of pessimism (or conservatism) in the face of uncertainty
\citep{jin2021pessimism,rashidinejad2021bridging}, whose benefits
have been established in the context of offline RL (or batch RL).
In a nutshell, the pessimism principle penalizes the Q-function based
on how infrequent a state-action pair is visited, which effectively directs
the attention of an RL algorithm away from the under-covered part
of the state-action space. However, it remains unclear how effective
this idea of pessimism could be in the asynchronous setting when coping with Markovian data. 

In order to address this issue, the current paper revisits asynchronous
Q-learning in the presence of a Markovian sample trajectory generated
by a behavior policy $\pi_{\mathsf{b}}$. We focus on a $\gamma$-discounted
infinite-horizon MDP with $S$ states and $A$ actions, and suppose
that the behavior policy is stationary and satisfies a certain single-policy
concentrability assumption (associated with a test distribution $\rho$)
with coefficient $C^{\star}\geq1$; informally, this means that the
observed sample trajectory effectively becomes expert data as $C^{\star}$
approaches 1, as we shall formalize in Section~\ref{sec:Models-and-assumptions}.
Our contributions are two-fold; here and below, $\widetilde{O}(\cdot)$ stands for the orderwise upper bound while
hiding any logarithmic dependency. 
\begin{itemize}
		
\item \emph{Asynchronous Q-learning with LCB penalization.} We propose a
variant of asynchronous Q-learning by penalizing each Q-learning iteration
based on a lower confidence bound (LCB). This variant of Q-learning
achieves $\varepsilon$-accuracy (w.r.t.~a test distribution $\rho$)
as long as the total sample size is above the order of 
\[
\widetilde{O}\Bigg(\frac{SC^{\star}}{\left(1-\gamma\right)^{5}\varepsilon^{2}}\Bigg),
\]
provided that the accuracy level $\varepsilon$ is small enough. Given
		that $C^{\star}$ can be as small as $O(1)$ and given the trivial bound
$1/\mu_{\mathsf{min}}\geq SA$ (so that $\eqref{eq:prior-Q-theory}\geq \frac{SA}{(1-\gamma)^4\varepsilon^2}$),
our theory leads to sample size benefits in terms of its dependency
on $A$ when the data is near-expert. 

\item \emph{Variance-reduced asynchronous Q-learning with LCB penalization.}
While asynchronous Q-learning with LCB penalization allows for reduced
sample complexity in the presence of near-expert data, the dependency on the effective
horizon $\frac{1}{1-\gamma}$ remains suboptimal. To address this,
we leverage the idea of variance reduction (also called reference-advantage
decomposition) \citep{wainwright2019variance,zhang2020almost,li2021breaking}
to further accelerate convergence of the algorithm, which in turn
yields a sample complexity 
\[
\widetilde{O}\Bigg(\frac{SC^{\star}}{\left(1-\gamma\right)^{3}\varepsilon^{2}}\Bigg)
\]
for sufficiently small accuracy level $\varepsilon$. The
scaling $\frac{1}{(1-\gamma)^{3}}$ is essentially unimprovable even
for the synchronous setting with independent samples \citep{azar2013minimax,rashidinejad2021bridging}. 
Notably, none of the prior works on offline RL were able to achieve the scaling of $\frac{SC^{\star}}{(1-\gamma)^3}$;  
that is, the best-known theory \citep{rashidinejad2021bridging} scales as $\widetilde{O}\big(\frac{SC^{\star}}{\left(1-\gamma\right)^{5}\varepsilon^{2}}\big)$ 
	and relies on i.i.d.~sampling. 
\end{itemize}
Finally, we remark that the algorithmic and theoretical frameworks
put forward herein are suitable for two important scenarios in the
absence of active exploration of the environment: (i) \emph{online
reinforcement learning} with a time-invariant policy (so that the data
arrives on the fly with no policy evolvement), and (ii) \emph{offline
reinforcement learning,} where the data generated by the behavior
policy has been pre-collected. In addition to the appealing sample complexity, model-free algorithms also enjoy the benefits of low memory and low computational complexity. 

%% file: formulation.tex
\section{Models and assumptions\label{sec:Models-and-assumptions}}

\paragraph{Basics of infinite-horizon Markov decision processes.}
In this paper, we consider an infinite-horizon Markov decision process,
denoted by $\mathcal{M}=(\mathcal{S},\mathcal{A},\gamma,P,r)$. Here,
$\mathcal{S}$ represents the state space that contains $S$ distinct
states; $\mathcal{A}$ stands for the action space that contains $A$
distinct actions; $\gamma\in(0,1)$ denotes the discount factor, with
$\frac{1}{1-\gamma}$ representing the effective horizon; $P:\mathcal{S}\times\mathcal{A}\rightarrow\Delta(\mathcal{S})$
stands for the probability transition kernel (with $\Delta(\mathcal{S})$
denoting the probability simplex over the set $\mathcal{S}$), such
that $P(\cdot\mymid s,a)\in\Delta(\mathcal{S})$ denotes the transition
probability from state $s$ when action $a$ is executed; $r:\mathcal{S}\times\mathcal{A}\rightarrow[0,1]$
indicates the deterministic reward function, such that $r(s,a)$ is
the immediate reward gained in state $s$ upon execution of action
$a$. We assume throughout that the immediate rewards fall within
the range $[0,1]$. 

Let $\Delta(\mathcal{A})$ be the probability simplex over the set
$\mathcal{A}$. A policy $\pi:\mathcal{S}\rightarrow\Delta(\mathcal{A})$
is an action selection rule, such that $\pi(\cdot\mymid s)\in\Delta(\mathcal{A})$
specifies the action selection probability in state $s$. When $\pi$
is deterministic, we often overload the notation and let $\pi(s)$
represent the action selected in state $s$. The value function and
the Q-function of policy $\pi$ are defined respectively as
\begin{align*}
\forall s\in\mathcal{S}: & \qquad V^{\pi}(s)\coloneqq\mathbb{E}\Bigg[\sum_{t=0}^{\infty}\gamma^{t}r(s_{t},a_{t})\mid s_{0}=s\Bigg],\\
\forall(s,a)\in\mathcal{S}\times\mathcal{A}: & \qquad Q^{\pi}(s,a)\coloneqq\mathbb{E}\Bigg[\sum_{t=0}^{\infty}\gamma^{t}r(s_{t},a_{t})\mid s_{0}=s,a_{0}=a\Bigg],
\end{align*}
where the expectation is taken over a random trajectory $(s_{0},a_{0},s_{1},a_{1},s_{2},a_{2},\cdots)$
induced by the MDP $\mathcal{M}$ when policy $\pi$ is employed.
For a given initial state distribution $\rho\in\Delta(\mathcal{S})$, we can also overload the notation of the value function to represent a certain average value function: 
\[
V^{\pi}(\rho)\coloneqq{\mathbb{E}}_{s\sim\rho}\big[V^{\pi}(s)\big].
\]
Moreover, it is well known that there exists at least one \emph{deterministic}
policy, denoted by $\pi^{\star}$, that simultaneously maximizes the
value function and the Q-function over all state-action pairs. Therefore, we introduce the following notation
\[
V^{\star}(s)\coloneqq\max_{\pi}V^{\pi}(s),\qquad V^{\star}(\rho)\coloneqq {\mathbb{E}}_{s\sim\rho}\big[V^{\star}(s)\big],\qquad\text{and}\qquad Q^{\star}(s,a)\coloneqq\max_{\pi}Q^{\pi}(s,a)
\]
to represent the optimal value function and the optimal Q-function.
Given a test distribution $\rho\in\Delta(\mathcal{S})$ and a target
accuracy level $\varepsilon\in\big(0,\frac{1}{1-\gamma}\big)$, our
aim is to compute a policy $\widehat{\pi}$ obeying
\[
V^{\star}(\rho)-V^{\widehat{\pi}}(\rho)\leq\varepsilon.
\]

A kind of distributions that plays an important role in our theory is the discounted state-action
occupancy distribution defined as follows:
\begin{align} \label{eq:defn-discounted-occupancy}
\forall(s,a)\in\mathcal{S}\times\mathcal{A}:\qquad 
	d_{\rho}^{\pi}\left(s,a\right) &\coloneqq\left(1-\gamma\right)\sum_{t=0}^{\infty}\gamma^{t}\mathbb{P}\left(s_{t}=s,a_{t}=a\mid\pi,s_{0}\sim\rho\right), \\
	d_{\rho}^{\pi}\left(s\right) &\coloneqq\left(1-\gamma\right)\sum_{t=0}^{\infty}\gamma^{t}\mathbb{P}\left(s_{t}=s\mid\pi,s_{0}\sim\rho\right),
\end{align}
where the trajectory $(s_{0},a_{0},s_{1},a_{1},s_{2},a_{2},\cdots)$
is induced by the MDP under the policy $\pi$ and a given initial
state distribution $\rho$. When $\pi$ coincides with the optimal
policy $\pi^{\star}$, we abbreviate
%introduce the following convenient notation
\begin{align} \label{eq:defn-discounted-occupancy-opt}
\forall(s,a)\in\mathcal{S}\times\mathcal{A}:\qquad 
	%\begin{cases}
		d_{\rho}^{\star}(s,a) &\coloneqq d_{\rho}^{\pi^{\star}}(s,a) \qquad \text{and} \qquad
		d_{\rho}^{\star}(s) \coloneqq d_{\rho}^{\pi^{\star}}(s) =  d_{\rho}^{\pi^{\star}}\big(s, \pi^{\star}(s) \big) .
	%\end{cases}
\end{align}

\paragraph{Sampling mechanism.} Suppose that the observed Markovian
sample trajectory $\big\{(s_{t},a_{t})\big\}_{t\geq0}$ is obtained
by executing a behavior policy $\pi_{\mathsf{b}}$ in the MDP $\mathcal{M}$. We say that the total sample size is $T$ if the algorithm employs $T$ state-action pairs of this trajectory, i.e., $\big\{(s_{t},a_{t})\big\}_{ 0\leq t\leq T}$. 
Assume that $\mu_{\mathsf{b}}(s,a)$ is the stationary distribution
of the this Markov chain generated by $\pi_{\mathsf{b}}$, with the
minimum state-action occupancy probability defined to be
\[
\mu_{\mathsf{min}}\coloneqq\min_{s\in\mathcal{S},\,a\in\mathcal{A}}\mu_{\mathsf{b}}(s,a).
\]
We impose the following assumptions on $\pi_{\mathsf{b}}$ throughout
this paper. 

\begin{assumption}\label{assumption:ergodic}The behavior policy
$\pi_{\mathsf{b}}$ is stationary, and the Markov chain induced by
$\pi_{\mathsf{b}}$ is uniformly ergodic. \end{assumption}\begin{remark}In
words, uniform ergodicity says that for any initial state-action pair,
the total-variation distance between the distribution of $(s_{t},a_{t})$
and the stationary distribution of the chain decays geometrically in $t$; see
\citet[Definition 1.1]{paulin2015concentration} for a precise definition of uniform ergodicity.
\end{remark}

Furthermore, for a given test distribution or initial state distribution
$\rho \in \Delta(\mathcal{S})$, we adopt the following concept as introduced in \citet{rashidinejad2021bridging}. 

\begin{assumption}[Single-policy concentrability]\label{assumption:policy}Suppose
that there exists some constant $C^{\star}\geq1$ such that
\begin{equation}
\forall(s,a)\in\mathcal{S}\times\mathcal{A}:\qquad\frac{d_{\rho}^{\star}\left(s,a\right)}{\mu_{\mathsf{b}}\left(s,a\right)}\leq C^{\star},
	\label{eq:defn-single-policy-concentrability}
\end{equation}
where we define $0/0=0$ by convention. Throughout this paper, $C^{\star}\geq1$
is called the single-policy concentrability coefficient.\end{assumption}

In some sense, the single-policy concentrability coefficient measures
the closeness between the stationary distribution of the observed
data and a certain occupancy distribution induced by the optimal policy.
In particular, if we take $\rho=\mu^{\star}$ to be the stationary
\emph{state distribution} of the MDP under the deterministic policy
$\pi^{\star}$, then it can be easily verified that
$
d_{\mu^{\star}}^{\star}\left(s,a\right)=\mu^{\star}(s) \ind\{ \pi^{\star}(s) = a \}, 
$
allowing us to rewrite (\ref{eq:defn-single-policy-concentrability}) w.r.t.~the density ratio of two stationary distributions as follows: 
\begin{equation}
\forall s\in\mathcal{S}:\qquad\frac{\mu^{\star}\left(s\right)}{\mu_{\mathsf{b}}\big(s,\pi^{\star}(s)\big)}\leq C^{\star}.\label{eq:defn-single-policy-concentrability-1}
\end{equation}
In this paper, the sample data is said to be near-expert if $C^{\star}=O(1)$, as in this case the empirical distribution of the sample
data is not far away from what is induced by the optimal policy. It is worth
noting that the single-policy concentrability coefficient \eqref{eq:defn-single-policy-concentrability} is a function
of the test distribution $\rho$ as well, although we suppress this
dependency in the notation $C^{\star}$ for the sake of conciseness. 

Another important quantity that affects the performance of our model-free
algorithms is the mixing time associated with the sample trajectory. To be precise, for any $0<\delta<1$,
the mixing time of the Markov chain induced by the MDP $\mathcal{M}$
under behavior policy $\pi_{\mathsf{b}}$ is defined as
\[
t_{\mathsf{mix}}\left(\delta\right)\coloneqq\min\left\{ t:\max_{s_{0}\in\mathcal{S},a_{0}\in\mathcal{A}}d_{\mathsf{TV}}\left(P^{t}\left(\cdot\mymid s_{0},a_{0}\right),\mu_{\mathsf{b}}\right)\leq\delta\right\} .
\]
 Here, $P^{t}(\cdot\mymid s_{0},a_{0})$ stands for the distribution
of $(s_{t},a_{t})$ (i.e., the state-action pair in the $t$-th step of the trajectory) when the chain is initialized to $(s_{0},a_{0})$,
whereas $d_{\mathsf{TV}}(\mu,\nu)$ is the total-variation distance
between two distributions $\mu$ and $\nu$ over a discrete space
$\mathcal{X}$ \citep{tsybakov2009introduction}, namely, 
\[
d_{\mathsf{TV}}\left(\mu,\nu\right)=\frac{1}{2}\sum_{x\in\mathcal{X}}\big|\mu(x)-\nu(x)\big|=\sup_{B\subseteq\mathcal{X}}\big|\mu(B)-\nu(B)\big|.
\]
In particular, we shall abbreviate
\[
t_{\mathsf{mix}}\coloneqq t_{\mathsf{mix}}(1/4),
\]
following the convention in prior works like \citet{paulin2015concentration}.
Clearly, this important quantity measures how long it takes for a
Markov chain to decorrelate itself from the initial state. 

\begin{remark} \label{remark:iid}
	Another simpler sampling mechanism studied in prior literature (e.g., \citet{rashidinejad2021bridging}) is i.i.d.~sampling, 
	under which  the observed sample trajectory takes the form of $\{(s_t,a_t,s_t')\}_{1\leq t \leq T}$ with
	\[
		(s_t, a_t) \sim \mu_{\mathsf{b}}
		\qquad \text{and} \qquad
		s_t' \sim P(\cdot\mymid s_t,a_t),
		\qquad \quad 1\leq t\leq T
	\]
	independently generated.  
	It is worth mentioning that the theorems and analysis in the current paper automatically apply to i.i.d.~sampling by taking $t_{\mathsf{mix}}=1$. Clearly, the Markovian sample trajectory studied herein is in general more challenging to cope with, due to the complicated Markovian dependency. 
\end{remark}

%% file: algorithm.tex
\section{Asynchronous Q-learning with LCB penalization\label{sec:Asynchronous-Q-learning-LCB}}

\begin{algorithm}[t]
	\textbf{Input:} number of iterations $T$, initial state $s$. \\
	\textbf{Initialize:}  $Q_{0}\left(s,a\right)=0$,
	$V_{0}(s)=0$, $n_{0}(s,a)=0$ for all $(s,a)\in\mathcal{S}\times\mathcal{A}$, $H=\lceil\frac{4}{1-\gamma}\log\frac{ST}{\delta}\rceil$. \\
	\For{$ t = 1$ \KwTo $T$}{
		Draw $a_{t-1}\sim\pi_{\mathsf{b}}(\cdot\mymid s_{t-1})$, and observe
		$s_{t}\sim P(\cdot\mymid s_{t-1},a_{t-1})$. \\
		Let $n_{t}\left(s_{t-1},a_{t-1}\right)= n_{t-1}(s_{t-1},a_{t-1})+1$; and $n_{t}(s,a)= n_{t-1}(s,a)$, $\forall(s,a)\neq(s_{t-1},a_{t-1})$. \\
		Set $n\leftarrow n_{t}(s,a)$, and take $\eta_{n}=(H+1)/(H+n)$. \\
		Update 
		\[
		Q_{t}\left(s_{t-1},a_{t-1}\right)  =\left(1-\eta_{n}\right)Q_{t-1}\left(s_{t-1},a_{t-1}\right)+\eta_{n}\Big\{ r\left(s_{t-1},a_{t-1}\right)+\gamma V_{t-1}\left(s_{t}\right)-b_{n}\Big\}
		\]
		and $Q_{t}(s,a)=Q_{t-1}(s,a)$ for all $(s,a)\neq(s_{t-1},a_{t-1})$,
		where 
		\[
		b_{n}=C_{\mathsf{b}}\sqrt{\frac{H\log\left(ST/\delta\right)}{n\left(1-\gamma\right)^{2}}}
		\]
		for some sufficiently large constant $C_{\mathsf{b}}>0$. \\
		Update
		\[
		V_{t}\left(s_{t-1}\right)=\max\bigg\{\max_{a\in\mathcal{A}}Q_{t}\left(s_{t-1},a\right),\ V_{t-1}(s_{t-1})\bigg\},
		\]
		and $V_{t}(s)=V_{t-1}(s)$ for all $s\neq s_{t-1}$.
	}
	\textbf{Output:} $\widehat{\pi}$ such that $\widehat{\pi}(s)=\arg\max_{a\in\mathcal{A}}Q_{T}(s,a)$
	for all $s\in\mathcal{S}$. 
	\caption{Asynchronous Q learning with LCB penalization.\label{alg:Q-5}}
\end{algorithm}

In this section, we describe how to incorporate the pessimism principle
into classical asynchronous Q-learning, accompanied by our theoretical performance
guarantees. 

\subsection{Algorithm}

We introduce the key algorithmic ingredients of our
first algorithm: asynchronous Q-learning with LCB penalization. The
complete details can be found in Algorithm \ref{alg:Q-5}. 

\paragraph{Asynchronous Q-learning.}

Let us begin by reviewing the basics of asynchronous Q-learning, which
maintains iterates $\{Q_{t}\}$ as the Q-function estimates. In each
iteration $t$, the algorithm takes action $a_{t-1}\sim\pi_{\mathsf{b}}(\cdot\mymid s_{t-1})$,
observes the next state $s_{t}\sim P(\cdot\mymid s_{t-1},a_{t-1})$, and 
 then updates its Q-function estimate w.r.t.~a single state-action pair
$(s_{t-1},a_{t-1})$ as follows
\begin{align*}
Q_{t}\left(s_{t-1},a_{t-1}\right) & =\left(1-\eta_{n}\right)Q_{t-1}\left(s_{t-1},a_{t-1}\right)+\eta_{n}\Big\{ r\left(s_{t-1},a_{t-1}\right)+\gamma V_{t-1}\left(s_{t}\right)\Big\},\\
Q_{t}\left(s,a\right) & =Q_{t-1}\left(s,a\right),\qquad\forall(s,a)\neq(s_{t-1},a_{t-1}).
\end{align*}
Here, $n$ represents the number of visits to $(s_{t-1},a_{t-1})$
prior to the $t$-th iteration, $0<\eta_{n}<1$ stands for the learning
rate, and the value function estimate is defined to
be $V_{t-1}(s)\coloneqq\max_{a\in\mathcal{A}}Q_{t-1}(s,a)$.

\paragraph{The pessimism principle and LCB penalization. }

In order to accommodate under-coverage of the state-action space in the presence of near-expert
data, a key idea is to penalize the Q-function of those state-action
pairs that are rarely visited (i.e., the ones that are not favored
by the ``expert''), so as to downplay their influence on the Q-estimates. Specifically, in the $t$-th iteration, we  modify the Q-learning update by inserting a penalty term $b_{n}$:\begin{subequations}\label{eq:Q-LCB-update}
\begin{align}
Q_{t}\left(s_{t-1},a_{t-1}\right) & =\left(1-\eta_{n}\right)Q_{t-1}\left(s_{t-1},a_{t-1}\right)+\eta_{n}\Big\{ r\left(s_{t-1},a_{t-1}\right)+\gamma V_{t-1}\left(s_{t}\right)-b_{n}\Big\},\\
Q_{t}\left(s,a\right) & =Q_{t-1}\left(s,a\right),\qquad\forall(s,a)\neq(s_{t-1},a_{t-1}),
\end{align}
\end{subequations}where the penalty term $b_{n}$ is chosen to be
some lower-confidence bound (LCB) determined by the Hoeffding concentration
inequality. More precisely, we shall set
\begin{equation}
b_{n}=C_{\mathsf{b}}\sqrt{\frac{H\log\left(ST/\delta\right)}{n\left(1-\gamma\right)^{2}}}\label{eq:penalty-Hoeffding}
\end{equation}
throughout this paper, where we take $H=\lceil\frac{4}{1-\gamma}\log\frac{ST}{\delta}\rceil$ --- so that $b_{n}$ is on the order of $\widetilde{O}\big(\sqrt{\frac{1}{(1-\gamma)^3 n}}\,\big)$ --- and recall that $n$ is the number of visits to $(s_{t-1}, a_{t-1})$ prior to time $t$.
The rationale behind this specific choice will be made clear in the analysis.

\paragraph{Monotonicity of value function estimates.}

In addition to the above pessimism principle, another consideration
is to ensure that the value function estimate $V_{t}$ always improves
upon (or at least, is no worse than) the previous estimate. Towards
this end, we take
\begin{align*}
V_{t}\left(s_{t-1}\right) & =\max\bigg\{\max_{a\in\mathcal{A}}Q_{t}\left(s_{t-1},a\right),\ V_{t-1}(s_{t-1})\bigg\},\\
V_{t}(s) & =V_{t-1}(s)\qquad\text{for all }s\neq s_{t-1},
\end{align*}
which yields monotonically non-decreasing
value function estimates $\{V_{t}\}_{t\geq0}$. This simple modification facilitates
analysis while ensuring that $V_{t}(s)$ is always non-negative (as long as 
we initialize $V_{t}(s)\geq 0$ for all $s\in\mathcal{S}$).

\paragraph{Computational and memory complexities.} 
The whole algorithm, as summarized in Algorithm~\ref{alg:Q-5} has low runtime  $O(T)$ and low memory complexity $O(\min\{T,SA\})$ (note that if a state-action pair is never visited, we do not need to record/update any quantity related to it).

\subsection{Theoretical guarantees}

Equipped with LCB penalization, asynchronous Q-learning is
capable of achieving appealing sample efficiency, even though the observed
sample trajectory might not provide full coverage of the state-action
space. This is stated in the following theorem, whose proof is
postponed to Section~\ref{sec:Analysis:-asynchronous-Q}. 

\begin{theorem}\label{theorem:5} Suppose that Assumptions \ref{assumption:ergodic}
and \ref{assumption:policy} hold, and recall that $T$ is the total
number of samples. With probability exceeding $1-\delta$, the policy
$\widehat{\pi}$ returned by Algorithm \ref{alg:Q-5} satisfies
\begin{equation}
V^{\star}(\rho)-V^{\widehat{\pi}}(\rho)\lesssim\sqrt{\frac{C^{\star}S\iota^{2}}{T\left(1-\gamma\right)^{5}}}+\frac{C^{\star}St_{\mathsf{mix}}\iota}{T\left(1-\gamma\right)^{2}}+\frac{C^{\star}t_{\mathsf{mix}}\iota^{2}}{T\left(1-\gamma\right)^{3}},\label{eq:Vstar-Vpi-bound-asyncQ}
\end{equation}
where $\iota\coloneqq\log(ST/\delta)$.\end{theorem}

By taking the right-hand side of (\ref{eq:Vstar-Vpi-bound-asyncQ})
to be bounded above by $\varepsilon$, we immediately see that Algorithm
\ref{alg:Q-5} achieves $\varepsilon$-accuracy with high probability,
as long as the total sample size $T$ exceeds
\begin{equation}
\widetilde{O}\left(\frac{SC^{\star}}{\left(1-\gamma\right)^{5}\varepsilon^{2}}+\frac{\big(S+\frac{1}{1-\gamma}\big)t_{\mathsf{mix}}C^{\star}}{\left(1-\gamma\right)^{2}\varepsilon}\right).\label{eq:sample-size-asynQ}
\end{equation}
This also means that the sample complexity of Algorithm \ref{alg:Q-5}
scales as
\begin{equation}
\widetilde{O}\left(\frac{SC^{\star}}{\left(1-\gamma\right)^{5}\varepsilon^{2}}\right)\label{eq:sample-size-asynQ-simplified}
\end{equation}
for any target accuracy level $0<\varepsilon\leq\frac{S}{\big(S+\frac{1}{1-\gamma}\big)\left(1-\gamma\right)^{3}t_{\mathsf{mix}}}$.
When we have near-expert data (so that $C^{\star}=O(1)$), the sample complexity can be as
low as
\[
\widetilde{O}\left(\frac{S}{\left(1-\gamma\right)^{5}\varepsilon^{2}}\right).
\]
In comparison, the general bound (\ref{eq:prior-Q-theory}) developed
in the previous literature requires at least $\frac{SA}{(1-\gamma)^{4}\varepsilon^{2}}$
samples (since $1/\mu_{\mathsf{min}}\geq SA$) regardless of what
behavior policy is employed. As a result, the proposed algorithm enjoys
enhanced adaptivity to near-expert data, particularly in the presence of
large action space and/or partial coverage.

It is worth noting, however, that the bound (\ref{eq:sample-size-asynQ-simplified})
	exhibits a dependency $\frac{1}{(1-\gamma)^{5}}$ on the effective
	horizon as opposed to $\frac{1}{(1-\gamma)^{4}}$, due to the adoption
	of the Hoeffding-style penalty (\ref{eq:penalty-Hoeffding}). This
	is potentially improvable by designing more careful Bernstein-style
	penalty (akin to \citet[Section 3]{jin2018q}). Nevertheless, we do
	not pursue this for two reasons: (a) the Hoeffding-style penalty streamlines analysis; (b) the Bernstein-style penalty alone
	is insufficient to yield optimal sample complexity, and we shall put forward
	another algorithm momentarily
	to achieve sample optimality. 
	
%\end{remark}

\section{Variance-reduced asynchronous Q-learning with LCB penalization\label{sec:Variance-reduced-asynchronous-Q}}

\begin{algorithm}[t]
	\textbf{Input:} number of iterations $T$, initial state $s$. \\
	\textbf{Initialize:} $\overline{V}(s)=0$ for all $s\in\mathcal{S}$, $K=\lfloor\log_{4}(3T/4)\rfloor$.  \\
	\For{$ k = 1$ \KwTo $K$}{
		Set $T_k^{\mathsf{ref}}=4^{k-1}$ and $T_k=3\times 4^{k-1}$. \\
		Call function $\mathsf{Empirical}\text{-}\mathsf{transition}(T_k^{\mathsf{ref}},\overline{V},s)$ (cf.~Algorithm \ref{alg:Q-3-transition}) and return $(\widetilde{P},b^{\mathsf{ref}}, s_1)$.  \\
		Call function $\mathsf{VR}\text{-}\mathsf{Q}\text{-}\mathsf{epoch}(T_k,\overline{V},\widetilde{P},b^{\mathsf{ref}}, s_1)$ (cf.~Algorithm \ref{alg:Q-3-inner}) and return $(Q,V,s_2)$. \\
		Update the reference $\overline{V}=V$, and set the initial state in the next epoch as $s=s_2$.
	}
	\textbf{Output:} $\widehat{\pi}$ such that $\widehat{\pi}(s)=\arg\max_{a\in\mathcal{A}}Q(s,a)$
	for all $s\in\mathcal{S}$. 
	\caption{Variance-reduced asynchronous Q-learning with LCB penalization.\label{alg:Q-3}}
\end{algorithm}

\begin{algorithm}[t]
	
	\textbf{Input:} number of samples $T^{\mathsf{ref}}$, reference $\overline{V}$, initial state $s_0^{\mathsf{ref}}$. \\
	\textbf{Initialize:}  $n^{\mathsf{ref}}(s,a)=0$ for all $(s,a)\in\mathcal{S}\times\mathcal{A}$, $\iota=\log\frac{ST}{\delta}$. \\
	\For{$ t = 1$ \KwTo $T^{\mathsf{ref}}$}{
		Draw $a_{t-1}^{\mathsf{ref}}\sim\pi_{\mathsf{b}}(\cdot\mymid s_{t-1}^{\mathsf{ref}})$, and observe
		$s_{t}^{\mathsf{ref}}\sim P(\cdot\mymid s_{t-1}^{\mathsf{ref}},a_{t-1}^{\mathsf{ref}})$. \\
		Let $n^{\mathsf{ref}}(s_{t-1}^{\mathsf{ref}},a_{t-1}^{\mathsf{ref}})\leftarrow n^{\mathsf{ref}}(s_{t-1}^{\mathsf{ref}},a_{t-1}^{\mathsf{ref}})+1$. Set $n\leftarrow n^{\mathsf{ref}}(s_{t-1}^{\mathsf{ref}},a_{t-1}^{\mathsf{ref}})$.\\
		Update 
			\begin{align*}
			\widetilde{P}\big(s_t^{\mathsf{ref}}\mymid s_{t-1}^{\mathsf{ref}},a_{t-1}^{\mathsf{ref}}\big)&\leftarrow\frac{(n-1)\widetilde{P}\big(s_t^{\mathsf{ref}}\mymid s_{t-1}^{\mathsf{ref}},a_{t-1}^{\mathsf{ref}}\big)+1}{n},\\
			\mu^{\mathsf{ref}}(s_{t-1}^{\mathsf{ref}},a_{t-1}^{\mathsf{ref}}) & \leftarrow \frac{(n-1)\mu^{\mathsf{ref}}(s_{t-1}^{\mathsf{ref}},a_{t-1}^{\mathsf{ref}})+ \overline{V}\left(s_{t}^{\mathsf{ref}}\right)}{n},\\
			\qquad\sigma^{\mathsf{ref}}(s_{t-1}^{\mathsf{ref}},a_{t-1}^{\mathsf{ref}}) & \leftarrow \frac{(n-1)\sigma^{\mathsf{ref}}(s_{t-1}^{\mathsf{ref}},a_{t-1}^{\mathsf{ref}})+ \overline{V}^2\left(s_{t}^{\mathsf{ref}}\right)}{n}.
		\end{align*}
	}
%	\textbf{Compute} an empirical transition kernel		
%	\[
%	\widetilde{P}\left(s'\mymid s,a\right)=\frac{\sum_{i=1}^{T^{\mathsf{ref}}}\ind\left\{ (s_{i-1}^{\mathsf{ref}},a_{i-1}^{\mathsf{ref}})=(s,a),s_{i}^{\mathsf{ref}}=s'\right\} }{\sum_{i=1}^{T^{\mathsf{ref}}}\ind\left\{ (s_{i-1}^{\mathsf{ref}},a_{i-1}^{\mathsf{ref}})=(s,a)\right\} }\qquad\text{for all }s,s'\in\mathcal{S},\,a\in\mathcal{A}.
%	\] \\
%	\textbf{Compute} the moment statistics: for each $(s,a)\in\mathcal{S}\times\mathcal{A}$, 
%	\begin{align*}
%		\mu^{\mathsf{ref}}\left(s,a\right) & =\frac{\sum_{i=1}^{T_{k}^{\mathsf{ref}}}\ind\left\{ (s_{i-1}^{\mathsf{ref}},a_{i-1}^{\mathsf{ref}})=(s,a)\right\} \overline{V}\left(s_{i}\right)}{\sum_{i=1}^{T_{k}^{\mathsf{ref}}}\ind\left\{ (s_{i-1}^{\mathsf{ref}},a_{i-1}^{\mathsf{ref}})=(s,a)\right\} },\\
%		\qquad\sigma^{\mathsf{ref}}\left(s,a\right) & =\frac{\sum_{i=1}^{T_{k}^{\mathsf{ref}}}\ind\left\{ (s_{i-1}^{\mathsf{ref}},a_{i-1}^{\mathsf{ref}})=(s,a)\right\} \overline{V}^2\left(s_{i}\right)}{\sum_{i=1}^{T_{k}^{\mathsf{ref}}}\ind\left\{ (s_{i-1}^{\mathsf{ref}},a_{i-1}^{\mathsf{ref}})=(s,a)\right\} }.
%	\end{align*} \\
	\textbf{Compute} the penalty term: for each $(s,a)\in\mathcal{S}\times\mathcal{A}$, take
	\begin{align*}
		b^{\mathsf{ref}}\left(s,a\right) = C_{\mathsf{b}}\left(\sqrt{\frac{\sigma^{\mathsf{ref}}\left(s,a\right)-\left[\mu^{\mathsf{ref}}\left(s,a\right)\right]^{2}}{n^{\mathsf{ref}}\left(s,a\right)}\iota }+\frac{\iota^{3/4}}{\left(1-\gamma\right)\left[n^{\mathsf{ref}}\left(s,a\right)\right]^{3/4}}+\frac{\iota}{\left(1-\gamma\right)n^{\mathsf{ref}}\left(s,a\right)}\right)
	\end{align*}
	for some sufficiently large constant $C_{\mathsf{b}}>0$. \\
	\textbf{Output:} empirical probability transition $\widetilde{P}$, penalty $b^{\mathsf{ref}}$, last state $s_{T^{\mathsf{ref}}}$.
	\caption{$\mathsf{Empirical}\text{-}\mathsf{transition}(T^{\mathsf{ref}},\overline{V},s_0^{\mathsf{ref}})$ \label{alg:Q-3-transition}}
\end{algorithm}

\begin{algorithm}[t]
	\textbf{Input:} number of iterations $T$, reference $\overline{V}$, transition kernel $\widetilde{P}$, penalty $b^{\mathsf{ref}}$, initial state $s_0$.\\
	\textbf{Initialize:}  $Q_{0}\left(s,a\right)=0$,
	$V_{0}(s)=0$, $n_{0}(s,a)=0$ for all $(s,a)\in\mathcal{S}\times\mathcal{A}$, $\iota=\log\frac{ST}{\delta}$, $H=\lceil\frac{4\iota}{1-\gamma}\rceil$. \\
	\For{$ t = 1$ \KwTo $T$}{
		Draw $a_{t-1}\sim\pi_{\mathsf{b}}(\cdot\mymid s_{t-1})$, and observe
		$s_{t}\sim P(\cdot\mymid s_{t-1},a_{t-1})$. \\
		Let $n_{t}\left(s_{t-1},a_{t-1}\right)\leftarrow n_{t-1}(s_{t-1},a_{t-1})+1$;
		and $n_{t}(s,a)\leftarrow n_{t-1}(s,a)$ for all $(s,a)\neq(s_{t-1},a_{t-1})$. \\
		Set $n\leftarrow n_{t}(s,a)$, and take $\eta_{n}=(H+1)/(H+n)$. \\
		Set $\mu_{n}^{\mathsf{adv}}(s,a)=\mu_{n-1}^{\mathsf{adv}}(s,a)$ and
		$\sigma_{n}^{\mathsf{adv}}(s,a)=\sigma_{n-1}^{\mathsf{adv}}(s,a)$
		for all $(s,a)\neq(s_{t-1},a_{t-1})$; update
		\begin{align*}
			\mu_{n}^{\mathsf{adv}}\left(s_{t-1},a_{t-1}\right) & =\left(1-\eta_{n}\right)\mu_{n-1}^{\mathsf{adv}}\left(s_{t-1},a_{t-1}\right)+\eta_{n}\left[V_{t-1}\left(s_{t}\right)-\overline{V}\left(s_{t}\right)\right],\\
			\sigma_{n}^{\mathsf{adv}}\left(s_{t-1},a_{t-1}\right) & =\left(1-\eta_{n}\right)\sigma_{n-1}^{\mathsf{adv}}\left(s_{t-1},a_{t-1}\right)+\eta_{n}\left[V_{t-1}\left(s_{t}\right)-\overline{V}\left(s_{t}\right)\right]^{2}.
		\end{align*} \\
		Compute $\mathsf{sd}_n^{\mathsf{adv}}(s_{t-1},a_{t-1})=\sigma_{n}^{\mathsf{adv}}\left(s_{t-1},a_{t-1}\right)-\left[\mu_{n}^{\mathsf{adv}}\left(s_{t-1},a_{t-1}\right)\right]^{2}$. \\
		%		\[
		%		\beta_{n}\left(s_{t-1},a_{t-1}\right)  = C_{\beta}\left(\sqrt{\frac{H\iota}{n}\left\{ \sigma_{n}^{\mathsf{adv}}\left(s_{t-1},a_{t-1}\right)-\left[\mu_{n}^{\mathsf{adv}}\left(s_{t-1},a_{t-1}\right)\right]^{2}\right\} }+\frac{H^{3/4}\iota^{3/4}}{n^{3/4}\left(1-\gamma\right)}+\frac{H\iota}{n\left(1-\gamma\right)}\right)
		%		\]
		%		for some sufficiently large constant $C_\beta>0$.  \\
		Update 
		\begin{align*}
			Q_{t}\left(s_{t-1},a_{t-1}\right) & =\left(1-\eta_{n}\right)Q_{t-1}\left(s_{t-1},a_{t-1}\right)\nonumber \\
			& \qquad+\eta_{n}\left\{ r\left(s_{t-1},a_{t-1}\right)+\gamma V_{t-1}\left(s_{t}\right)-\gamma\overline{V}\left(s_{t}\right)+\gamma\big\langle\widetilde{P}(\cdot\mymid s_{t-1},a_{t-1}),\overline{V}\big\rangle-b_{n}\right\} .
		\end{align*}
		and $Q_{t}(s,a)=Q_{t-1}(s,a)$ for all $(s,a)\neq(s_{t-1},a_{t-1})$,
		where $b_t=b^{\mathsf{ref}}(s_{t-1},a_{t-1})+b^{\mathsf{adv}}$ and 
		\[
		b^{\mathsf{adv}}=C_{\mathsf{b}}\left(\sqrt{\frac{H\iota}{n}}\frac{\mathsf{sd}_n^{\mathsf{adv}}(s_{t-1},a_{t-1})-(1-\eta_n)\mathsf{sd}_{n-1}^{\mathsf{adv}}(s_{t-1},a_{t-1})}{\eta_n} +\frac{H^{3/4}\iota^{3/4}}{n^{3/4}\left(1-\gamma\right)}+\frac{H\iota}{n\left(1-\gamma\right)}\right)
		\] 
		for some sufficiently large constant $C_{\mathsf{b}}>0$.\\
		%		\[
		%		b_{n}=\frac{\beta_{n}\left(s_{t-1},a_{t-1}\right)-\left(1-\eta_{n}\right)\beta_{n-1}\left(s_{t-1},a_{t-1}\right)}{\eta_{n}}+b_{\mathsf{ref}}\left(s,a\right).
		%		\] \\
		Update
		\[
		V_{t}\left(s_{t-1}\right)=\max\bigg\{\max_{a\in\mathcal{A}}Q_{t}\left(s_{t-1},a\right),\ V_{t-1}(s_{t-1})\bigg\},
		\]
		and $V_{t}(s)=V_{t-1}(s)$ for all $s\neq s_{t-1}$.
	}
	\textbf{Output:} Q-function estimate $Q_T$, value function estimate $V_T$, last state $s_T$.
	\caption{$\mathsf{VR}\text{-}\mathsf{Q}\text{-}\mathsf{epoch}(T,\overline{V},\widetilde{P},b^{\mathsf{ref}}, s_0)$ \label{alg:Q-3-inner}}
\end{algorithm}

As we have alluded to previously, the algorithm presented in Section~\ref{sec:Asynchronous-Q-learning-LCB}
falls short of achieving optimal dependency on the effective horizon.
To address this issue, a plausible idea is to leverage the variance
reduction technique --- originally introduced in finite-sum stochastic optimization \citep{johnson2013accelerating} and imported to online RL recently \citep{zhang2020almost} --- to further accelerate convergence of the algorithm.
This section is devoted to the development of a new variant of asynchronous
Q-learning that incorporates both pessimism and variance
reduction.

\subsection{Algorithm}

We start by describing the key ideas of a variance-reduced variant
of Algorithm~\ref{alg:Q-5}. 
This algorithm enjoys the same computational cost (i.e., $O(T)$) and memory complexity (i.e., $O(SA)$) 
as Algorithm \ref{alg:Q-5}, 
with  full details are summarized in Algorithm
\ref{alg:Q-3} (in conjunction with Algorithms~\ref{alg:Q-3-transition} and \ref{alg:Q-3-inner}). 

\paragraph{Variance reduction.}

Suppose for the moment that we have access to a ``reference'' value
function estimate $\overline{V}$ that is hopefully not far away from
the true optimal value $V^{\star}$. Let us employ a batch of samples
--- more concretely, a total number of $T^{\mathsf{ref}}$ consecutive
samples $\{(s_{i}^{\mathsf{ref}},a_{i}^{\mathsf{ref}},s_{i+1}^{\mathsf{ref}}):0\leq i<T^{\mathsf{ref}}\}$
--- to compute an empirical estimate $\widetilde{P}:\mathcal{S}\times\mathcal{A}\to\Delta(\mathcal{S})$
of the probability transition kernel $P$. We can then incorporate
variance reduction into the update rule (\ref{eq:Q-LCB-update}) of
Algorithm~\ref{alg:Q-5} as follows:
\begin{align}
	& Q_{t}\left(s_{t-1},a_{t-1}\right)  =\left(1-\eta_{n}\right)Q_{t-1}\left(s_{t-1},a_{t-1}\right) + \nonumber \\
	& ~~~ \eta_{n}\left\{ r\left(s_{t-1},a_{t-1}\right)+\gamma V_{t-1}\left(s_{t}\right)-\gamma\overline{V}\left(s_{t}\right)+\gamma\big\langle\widetilde{P}(\cdot\mymid s_{t-1},a_{t-1}),\overline{V}\big\rangle-b_{n}\left(s_{t-1},a_{t-1}\right)\right\} .\label{eq:VR-update-rule}
\end{align}
Here, the penalty term $b_{n}(s_{t-1},a_{t-1})$ is set to be a certain
data-driven lower confidence bound tailored to this variance-reduced
update rule. In particular, this penalty term is chosen to track the
uncertainty of both the ``advantage term'' $V_{t-1}\left(s_{t}\right)-\overline{V}\left(s_{t}\right)$
and the ``reference term'' $\big\langle\widetilde{P}(\cdot\mymid s_{t-1},a_{t-1}),\overline{V}\big\rangle$,
inspired by the reference-advantage decomposition introduced in \citet{zhang2020almost};
see Algorithm \ref{alg:Q-3-inner} for a precise description. As can
be anticipated, if $\overline{V}$ is a more accurate estimate
of $V^{\star}$ than $V_{t-1}$ (i.e., $\overline{V}\approx V^{\star}$
and $\|\overline{V}-V^{\star}\|_{\infty}\ll\|V_{t-1}-V^{\star}\|_{\infty}$),
then the main stochastic term $\overline{V}(s_t)$ (or  $\overline{V}(s_t) - V^{\star}(s_t) $)
in (\ref{eq:VR-update-rule}) is expected to be much less volatile
than the counterpart $V_{t-1}\left(s_{t}\right)$ in (\ref{eq:Q-LCB-update}),
thus resulting in substantial variance reduction and hence accelerated
convergence. It remains to develop a plausible approach that produces
such reliable ``reference'' value function estimates.

\paragraph{An epoch-based paradigm. }

The proposed algorithm proceeds in an epoch-based manner ($K=\lfloor\log_{4}(3T/4)\rfloor$
epochs in total). In the $k$-th epoch, we use the value function
estimate at the end of the previous epoch as the reference function
estimate $\overline{V}$; the number of samples used to construct
the empirical transition kernel and the number of samples employed
to run the updates (\ref{eq:VR-update-rule}) are denoted respectively
by $T_{k}^{\mathsf{ref}}$ and $T_{k}$, both of which are chosen to
grow exponentially with the epoch number $k$ (more specifically,
we shall choose $T_{k}^{\mathsf{ref}}=4^{k-1}$ and $T_{k}=3\cdot4^{k-1}$).
Such choices allow one to ensure that: (i) the estimation error keeps
improving over epochs; and (ii) the samples used in the latest epoch
always account for roughly $3/4$ of the total sample size used so
far, thus mitigating inefficient use of samples despite the lack of sample reuse.

\subsection{Theoretical guarantees}

Armed with the variance reduction idea, we are able to further improve
the sample complexity in terms of the dependency on $\frac{1}{1-\gamma}$,
as stated below. 

\begin{theorem}\label{theorem:3}Suppose that Assumptions \ref{assumption:ergodic}
	and \ref{assumption:policy} hold, and recall that $T$ is the total
	number of samples. Assume that $1/2\leq \gamma < 1$. Then with probability exceeding $1-\delta$, the
	policy $\widehat{\pi}$ returned by Algorithm \ref{alg:Q-3} satisfies
	\begin{align*}
		V^{\star}(\rho)-V^{\widehat{\pi}}(\rho) & \lesssim\sqrt{\frac{SC^{\star}\iota}{T\left(1-\gamma\right)^{3}}}+\frac{SC^{\star}\iota^{4}}{T\left(1-\gamma\right)^{4}}+\frac{St_{\mathsf{mix}}C^{\star}\iota}{T\left(1-\gamma\right)^{2}}+\frac{t_{\mathsf{mix}}C^{\star}\iota^{2}}{T\left(1-\gamma\right)^{3}},
	\end{align*}
	where $\iota\coloneqq\log\frac{ST}{\delta}$.\end{theorem}

Theorem~\ref{theorem:3} asserts that the sample size needed for
Algorithm \ref{alg:Q-3} to achieve $\varepsilon$-accuracy is at most
\begin{equation}
	\widetilde{O}\left(\frac{SC^{\star}}{\left(1-\gamma\right)^{3}\varepsilon^{2}}+\frac{SC^{\star}}{\left(1-\gamma\right)^{4}\varepsilon}+\frac{SC^{\star}t_{\mathsf{mix}}}{\left(1-\gamma\right)^{2}\varepsilon}+\frac{t_{\mathsf{mix}}C^{\star}}{\left(1-\gamma\right)^{3}\varepsilon}\right).\label{eq:sample-complexity-VR}
\end{equation}
In particular, if the accuracy level $\varepsilon\leq\min\big\{1-\gamma,\,\frac{S}{t_{\mathsf{mix}}},\, \frac{1}{(1-\gamma) t_{\mathsf{mix}}}\big\}$,
then the sample complexity of Algorithm \ref{alg:Q-3} simplifies
to
\begin{equation}
	\widetilde{O}\left(\frac{SC^{\star}}{\left(1-\gamma\right)^{3}\varepsilon^{2}}\right).\label{eq:sample-complexity-simpler-VR}
\end{equation}
%
%\yxc{This bound is essentially unimprovable; in fact, even when the data
%samples are drawn i.i.d.~from the stationary distribution $\mu_{\mathsf{b}}$
%of the sample trajectory (namely, the simpler case in the absence
%of Markovian dependency), a lower bound has been established by \citet{rashidinejad2021bridging} 
% that coincides with (\ref{eq:sample-complexity-simpler-VR})
%when $C^{\star}=O(1)$. 
%In fact, even when i.i.d.~sampling is allowed (which can be viewed as a sample trajectory with $t_{\mathsf{mix}} = 1$), 
%the state-of-the-art result before our work  scales as $\widetilde{O}\big(\frac{SC^{\star}}{\left(1-\gamma\right)^{5}\varepsilon^{2}}\big)$ \citep{rashidinejad2021bridging}, 
%which was suboptimal by a factor of $\frac{1}{(1-\gamma)^2}$. 
%%no prior theory on offline RL was able to achieve the sample complexity \eqref{eq:sample-complexity-simpler-VR} 
%All this confirms the efficacy of the pessimism principle in conjunction with variance reduction when running model-free algorithms. }
%
This bound is essentially unimprovable; in fact, even for the simpler i.i.d.~sampling mechanism described in Remark~\ref{remark:iid} (which can be viewed as a sample trajectory with $t_{\mathsf{mix}} = 1$), a lower bound has been established by \citet{rashidinejad2021bridging} 
	that coincides with (\ref{eq:sample-complexity-simpler-VR})
	when $C^{\star}=O(1)$. 
	In fact, even when i.i.d.~sampling is allowed,
	the state-of-the-art result before our work  scales as $\widetilde{O}\big(\frac{SC^{\star}}{\left(1-\gamma\right)^{5}\varepsilon^{2}}\big)$ \citep{rashidinejad2021bridging}, 
	which was suboptimal by a factor of $\frac{1}{(1-\gamma)^2}$. 
	%no prior theory on offline RL was able to achieve the sample complexity \eqref{eq:sample-complexity-simpler-VR} 
	All this confirms the efficacy of the pessimism principle in conjunction with variance reduction when running model-free algorithms.

%	\yly{Both Algorithm \ref{alg:Q-5} and \ref{alg:Q-3} has low computational complexity  $O(T)$ and low memory cost $O(\min\{T,SA\})$ (note that if a state-action pair is never visited, we do not need to record any quantity related to it.). }

%% file: related_work.tex
\section{Related works }

\paragraph{Offline RL and pessimism. }

The principal of pessimism (or conservatism) in the face of uncertainty
has recently been employed and studied extensively in offline RL (also called
batch RL), e.g., \citet{kumar2020conservative,kidambi2020morel,yu2020mopo,yu2021conservative,yu2021combo,yin2021near_double,rashidinejad2021bridging,jin2021pessimism,xie2021policy,liu2020provably,zhang2021corruption,chang2021mitigating,yin2021towards,uehara2021pessimistic,munos2003error,li2022settling,munos2007performance,yin2021near_b}.
Among these prior works, \citet{rashidinejad2021bridging} studied
offline RL for infinite-horizon MDPs when the offline data are i.i.d.~samples
drawn from some distribution $\mu$ satisfying the single policy concentrability
condition. They showed that a model-based value iteration algorithm
with LCB penalization achieves a sample complexity of $O(\frac{SC^{\star}}{(1-\gamma)^{5}\varepsilon^{2}})$,
which is comparable to our bound for Algorithm~\ref{alg:Q-5} (see \eqref{eq:sample-size-asynQ-simplified}) and is worse than our bound for
Algorithm~\ref{alg:Q-3} (see \eqref{eq:sample-complexity-simpler-VR}) by a factor of $\frac{1}{(1-\gamma)^{2}}$ (ignoring
the $o(\varepsilon^{-2})$ term and logarithm factors). Note that the
setting considered in \citet{rashidinejad2021bridging} is a special
case of our setting by taking  $t_{\mathsf{mix}}=1$. In addition,
\citet{jin2021pessimism} proposed a pessimistic variant of the value
iteration algorithm, which  achieves appealing performance under
the episodic linear MDP setting. 
Furthermore, the recent works \citet{xie2021policy,shi2022pessimistic} proposed several pessimistic variants of RL algorithms for finite-horizon episodic MDPs.  Focusing on offline
RL with episodic data generated using some reference policy
satisfying the single policy concentrability condition, these algorithms achieve a sample complexity of $\widetilde{O}(H^{3}SC^{\star}/\varepsilon^{2})$. Note, however, that none of these algorithms accommodate the asynchronous case with a single Markovian trajectory. 
%which do not apply to the asynchronous case with a single Markovian trajectory.  
%Additionally, \citet{xie2021policy} studied the
%problem of policy fine-tuning in episodic finite-horizon MDPs; 

\paragraph{Q-learning. }

There are at least two basic forms of Q-learning: the synchronous version and the
asynchronous counterpart. Synchronous Q-learning typically assumes access
to a simulator that generates independent samples for all state-action
pairs, and attempts to update all entries of the Q-function estimates simultaneously \citep{even2003learning,beck2012error,chen2020finite,wainwright2019stochastic,bowen2021finite,li2021q,wang2021sample}.
The current paper studies the asynchronous form of Q-learning, which naturally arises when the data is a Markovian
trajectory induced by a behavior policy \citep{jaakkola1994convergence,tsitsiklis1994asynchronous,even2003learning,qu2020finite,li2021sample,chen2021lyapunov,shah2018q,chen2022target,xiong2020finite,li2021q}.
However, most prior works focused on the case when the observed trajectory
is able to cover all state-action pairs with sufficient frequency \citep{beck2012error,even2003learning,qu2020finite,li2021sample,chen2021lyapunov,li2021q}. 
For instance, the recent work \citet{qu2020finite}
demonstrated that the sample complexity of asynchronous Q-learning
is at most $\widetilde{O}\big(\frac{t_{\mathsf{mix}}}{\mu_{\mathsf{min}}^{2}(1-\gamma)^{5}\varepsilon^{2}} \big)$,
which was subsequently sharpened by \citet{li2021q} to $\widetilde{O}\big( \frac{1}{\mu_{\mathsf{min}}(1-\gamma)^{4}\varepsilon^{2}} + \frac{t_{\mathsf{mix}}}{\mu_{\mathsf{mix}}(1-\gamma)} \big)$.  It is also worth noting that some variants of model-free algorithms (e.g., the variant coupled with upper confidence bounds) have proven effective for online exploratory RL \citep{strehl2006pac,pazis2016improving, jin2018q,bai2019provably,yang2021q,dong2019q,menard2021ucb,li2021breaking,zhang2021model}; while online RL is beyond  the scope of the current paper, the analysis framework therein based on the optimism principle shed light on our setting as well. 
In comparison to the model-based approach \citep{agarwal2019optimality,azar2017minimax,li2020breaking,agarwal2020model}, 
model-free algorithms like Q-learning often incur lower memory and computational complexities.

\paragraph{Variance reduction.}

The idea of variance reduction first appeared in the stochastic optimization
literature \citep{johnson2013accelerating} and has been recently employed
in RL to speed up various algorithms \citep{wainwright2019variance,li2021sample,sidford2018near,sidford2018variance,yang2019sample,khamaru2021temporal,khamaru2021instance,du2017stochastic,wai2019variance,xu2019reanalysis,zhang2020almost,li2021breaking,zhang2021model,shi2022pessimistic}. 
Among these works, \citet{wainwright2019variance,yang2019sample} showed
that in the synchronous case, variance-reduced Q-learning is minimax
optimal, both in tabular MDPs and the ones with function approximation. \citet{li2021sample} showed that the sample
complexity of variance-reduced asynchronous Q-learning algorithm scales
as $\widetilde{O}(\frac{t_{\mathsf{mix}}}{\mu_{\mathsf{min}}(1-\gamma)^{3}\varepsilon^{2}})$ for small enough accuracy level $\varepsilon$, thereby
matching the lower bound in the synchronous counterpart. 

%% file: discussion.tex
\section{Discussion}

In this paper, we have revisited the paradigm of asynchronous Q-learning,
which was designed to accommodate Markovian sample trajectories. Noteworthily,
all prior theory for asynchronous Q-learning becomes vacuous when
the observed sample trajectory falls short of providing uniform coverage
of all state-action pairs, even when the observed data is produced
by an expert that intentionally leaves out suboptimal actions. To
address this issue, we have designed two algorithms --- asynchronous
Q-learning algorithms with LCB penalization and its variance-reduced
variant --- based on the principle of pessimism in the face of uncertainty.
The sample complexities of these two algorithms scale as $\widetilde{O}\big(\frac{SC^{\star}}{(1-\gamma)^{5}\varepsilon^{2}}\big)$
and $\widetilde{O}\big(\frac{SC^{\star}}{(1-\gamma)^{3}\varepsilon^{2}}\big)$,
respectively, provided that the target accuracy level $\varepsilon$
is sufficiently small; in particular, the latter one matches the lower
bound established for the case with i.i.d.~data and is hence unimprovable.
Compared to prior literature, we have established the first theory
that supports the use of pessimism principle despite the Markovian
structure of data. Moving forward, there are numerous directions that
are worthy of further exploration. For example, the dependency of
our sample complexity on the mixing time scales as $\widetilde{O}\big(\frac{St_{\mathsf{mix}}C^{\star}}{(1-\gamma)^{2}\varepsilon}+\frac{t_{\mathsf{mix}}C^{\star}}{(1-\gamma)^{2}\varepsilon}\big)$;
it remains unclear what the optimal dependency on $t_{\mathsf{mix}}$
is, as well as how to achieve it. Additionally, the current work focuses
solely on tabular MDPs; it would be of interest to extend the current
analysis to accommodate reduced-dimensional function approximation. 
Going beyond offline RL, our analysis framework might shed light on 
how to improve the sample complexity analysis for discounted infinite-horizon MDPs in {\em online exploratory} RL  (note that the state-of-the-art sample complexity bounds in this case \citep{zhang2021model} remain highly suboptimal except for very small $\varepsilon$ (i.e., $\varepsilon\leq \frac{(1-\gamma)^{14}}{S^2A^2}$)).

%\yxc{might need updates after the lower bound is included} 

%% file: notations.tex
\section{Notation} 
We now introduce several notation that will
be used multiple times throughout this paper. For any positive integer
$n$, we define $[n]\coloneqq\{1,\cdots,n\}$. For any $s\in\mathcal{S}$
and $a\in\mathcal{A}$, define $$P_{s,a} = P(\cdot\mymid s,a)\in\mathbb{R}^{1\times S}$$
to be the $(s,a)$-th row of a probability transition matrix $P\in\mathbb{R}^{SA\times S}$.
For any $t\geq0$, we define $P_{t}\in\mathbb{R}^{SA\times S}$ to
be an empirical probability transition matrix such that
\begin{equation}
P_{t}\big(s'\mymid s,a \big)=\begin{cases}
1, & \text{if }(s,a,s')=(s_{t-1},a_{t-1},s_{t})\\
0, & \text{otherwise}
\end{cases}\label{eq:defn-Pt}
\end{equation}
for all $s,s'\in\mathcal{S}$ and $a\in\mathcal{A}$. For any deterministic
policy $\pi$, we introduce two probability transition kernels $P_{\pi}:\mathcal{S}\to\Delta(\mathcal{S})$
and $P^{\pi}:\mathcal{S}\times\mathcal{A}\to\Delta(\mathcal{S}\times\mathcal{A})$,
defined in a way that\begin{subequations}\label{defn:P-pi}
\begin{align}
P_{\pi}(s'\mymid s) & =P\big(s'\mymid s,\pi(s)\big)\\
P^{\pi}\left(s',a'\mymid s,a\right) & =\begin{cases}
P\left(s'\mymid s,a\right), & \text{if }a'=\pi\left(s'\right)\\
0, & \text{otherwise}
\end{cases}
\end{align}
\end{subequations}for any $(s,a),(s',a')\in\mathcal{S}\times\mathcal{A}$.
In addition, we define $\rho^{\pi^{\star}}$ to be a distribution 
on $\mathcal{S}\times\mathcal{A}$ such that
\begin{equation}
\rho^{\pi^{\star}}\left(s,a\right)=\begin{cases}
\rho\left(s\right), & \text{if }a=\pi^{\star}\left(s\right),\\
0, & \text{otherwise}.
\end{cases}\label{eq:defn-rho-pistar}
\end{equation}
For any two vectors $a=[a_i]_{i=1}^n\in\mathbb{R}^n$ and $b=[b_i]_{i=1}^n\in\mathbb{R}^n$, we define the Hadamard product $a\circ b=[a_i b_i]_{i=1}^n$, as well as the concise notation $a^2=a\circ a$.  We also use $a\leq b$ (resp.~$a\geq b$) to denote $a_i\leq b_i$ (resp.~$a_i\geq b_i$) for all $i\in[n]$. 
Moreover, for two vectors $a = [a_1,\cdots,a_n]$ and $b=[b_1,\cdots,b_n]^{\top}$, we abuse the notation by letting
\[
	\langle a , b \rangle = \sum_{i=1}^n a_i b_i
\]
even when $a$ is a row vector and $b$ is a column vector. For any $s\in\mathcal{S}$, $a\in\mathcal{A}$ and any vector $V\in\mathbb{R}^S$, we define and denote 
	\[
	\mathsf{Var}_{s,a}(V) \coloneqq \mathsf{Var}_{s'\sim P(\cdot\mymid s,a)}\big(V(s')\big)=P_{s,a}\left(V^2\right)-(P_{s,a}V)^2.
	\]

We let $f(n)\lesssim g(n)$ or $f(n)=O(g(n))$ to denote $\vert f(n)\vert\leq Cg(n)$ for some constant
$C>0$ when $n$ is sufficiently large; we use $f(n)\gtrsim g(n)$
to indicate that $f(n)\geq C\vert g(n)\vert$ for some constant $C>0$ when
$n$ is sufficiently large; and we let $f(n)\asymp g(n)$ represent
the condition that $f(n)\lesssim g(n)$ and $f(n)\gtrsim g(n)$ hold simultaneously.
Throughout this paper, we define $0/0=0$. For any sequence $\{a_{i}\}_{i=n_{1}}^{n_{2}}$
and two integers $m_{1}$ and $m_{2}$, we define
\[
\sum_{i=m_{1}}^{m_{2}}a_{i}=\begin{cases}
\sum_{i=\max\{n_{1},m_{1}\}}^{\min\{n_{2},m_{2}\}}a_{i}, & \text{if }\max\{n_{1},m_{1}\}\leq\min\{n_{2},m_{2}\},\\
0, & \text{else}.
\end{cases}
\]

%% file: hoeffding.tex
\section{Analysis for Q-learning with LCB penalization (Theorem~\ref{theorem:5})\label{sec:Analysis:-asynchronous-Q}}

In this section, we present the proof of Theorem \ref{theorem:5},
which consists of several steps to be detailed below. 

\subsection{Preliminary facts and additional notation}

Before proceeding, we first introduce the following quantities regarding
the learning rates:
\begin{equation}
\eta_{0}^{t}\coloneqq\prod_{j=1}^{t}\left(1-\eta_{j}\right)\qquad\text{and}\qquad\eta_{i}^{t}\coloneqq\begin{cases}
\eta_{i}\prod_{j=i+1}^{t}\left(1-\eta_{j}\right), & \text{if }t>i,\\
\eta_{i}, & \text{if }t=i,\\
0, & \text{if }t<i,
\end{cases}\label{eq:defn-eta-it}
\end{equation}
where we recall our choice $\eta_{j}=(H+1)/(H+j)$. We make note of
the following results that have been established in prior works (e.g., \cite[Lemma 4.1]{jin2018q}
and \cite[Lemma 1]{li2021breaking}). 

\begin{lemma}\label{lemma:step-size}The learning rates satisfy the
following properties.
\begin{enumerate}
\item For any integer $t\geq1$, $\sum_{i=1}^{t}\eta_{i}^{t}=1$ and $\eta_{0}^{t}=0$.
\item For any integer $t\geq1$ and any $1/2\leq a\leq1$,
\[
\frac{1}{t^{a}}\leq\sum_{i=1}^{t}\frac{1}{i^{a}}\eta_{i}^{t}\leq\frac{2}{t^{a}}.
\]
\item For any integer $t\geq1$, 
\[
\max_{i\in[t]}\eta_{i}^{t}\leq\frac{2H}{t}\qquad\text{and}\qquad\sum_{i=1}^{t}\left(\eta_{i}^{t}\right)^{2}\leq\frac{2H}{t}.
\]
\item For any integer $i\geq1$,
\[
\sum_{t=i}^{\infty}\eta_{i}^{t}=1+\frac{1}{H}.
\]
\end{enumerate}
\end{lemma} 

For any iteration $t\leq T$, we remind the reader that $n_{t}$ represents
the number of times $(s,a)$ has been visited prior to iteration $t$
(see Algorithm~\ref{alg:Q-5}). For notational simplicity, let $n=n_{t}(s,a)$
when it is clear from the context, and suppose that $(s,a)$ has been
visited during the iterations $k_{1}<\cdots<k_{n}<t$. We also find
it convenient to define the (deterministic) policy estimate $\pi_{t}:\mathcal{S}\to\mathcal{A}$ recursively 
as follows
\begin{equation}
\pi_{t}\left(s\right)\coloneqq\begin{cases}
\arg\max_{a\in\mathcal{A}}Q_{t}\left(s_{t-1},a\right), & \text{if }s=s_{t-1}\text{ and }V_{t}\left(s\right)>V_{t-1}\left(s\right),\\
\pi_{t-1}\left(s\right), & \text{otherwise}.
\end{cases}\label{eq:defn-pit-asyncQ}
\end{equation}
If there are multiple $a\in\mathcal{A}$ that maximize $Q_{t}\left(s_{t-1},a\right)$,
we can pick any of these actions. 

The following lemma provides a useful upper bound on $Q^{\star}-Q_{t}$,
and in the meantime, justifies that the value function estimate $V_{t}$
is always a pessimistic view of $V^{\pi_{t}}$ (and hence $V^{\star}$).
The proof of this lemma is postponed to Appendix~\ref{appendix:proof-lcb-h-basics}. 

\begin{lemma}\label{lemma:lcb-h-basics}With probability exceeding
$1-\delta$, for all $s\in\mathcal{S}$ and $t\in[T]$, it holds that
\[
Q^{\star}\big(s,\pi^{\star}(s)\big)-Q_{t}\big(s,\pi^{\star}(s)\big)\leq\gamma\sum_{i=1}^{n}\eta_{i}^{n}P_{s,\pi^{\star}(s)}\big(V^{\star}-V_{k_{i}}\big)+\beta_{n}\big(s,\pi^{\star}(s)\big),
\]
where $n=n_t(s,\pi^\star(s))$ and
we define
\[
\beta_{n}\big(s,\pi^{\star}(s)\big)\equiv\beta_{n}\coloneqq3C_{\mathsf{b}}\sqrt{\frac{H\iota}{n\left(1-\gamma\right)^{2}}};
\]
in addition, we also have
\[
V_{t}(s)\leq V^{\pi_{t}}(s)\leq V^{\star}(s),\qquad\forall s\in\mathcal{S}.
\]
\end{lemma}

Next, let us define two disjoint sets of state-action pairs, divided
based on the associated occupancy probability induced by the behavior
policy: 
\begin{subequations}
\label{eq:defn-I-Ic}
\begin{align}
\mathcal{I} & \coloneqq\left\{ \big(s,\pi^{\star}(s)\big) \mid s\in\mathcal{S},\mu_{\mathsf{b}}\big(s,\pi^{\star}(s)\big)\geq\frac{\delta}{ST}\right\} ,\\
\mathcal{I}^{c} & \coloneqq\left\{ \big(s,\pi^{\star}(s)\big) \mid s\in\mathcal{S},\mu_{\mathsf{b}}\big(s,\pi^{\star}(s)\big)<\frac{\delta}{ST}\right\} .
\end{align}
\end{subequations}
It turns out that the state-action pairs in $\mathcal{I}^{c}$ are rarely visited, as formalized
by the following lemma. The proof is deferred to Appendix \ref{appendix:proof-lemma-states-small-prob}. 

\begin{lemma}\label{lemma:states-small-prob}With probability exceeding
$1-\delta$, we have
\[
\mathcal{I}^{c}\cap\big\{(s_{t},a_{t})\big\}_{t=t_{\mathsf{mix}}(\delta)}^{T}=\varnothing.
\]
\end{lemma}

\subsection{Step 1: error decomposition}

Before proceeding, let us introduce several quantities that will play
an important role in our analysis:
\begin{align*}
\alpha_{j} & \coloneqq\left[\gamma\left(1+\frac{1}{H}\right)^{3}\right]^{j}\sum_{t=1}^{T}\left\langle \rho(P_{\pi^{\star}})^{j},V^{\star}-V_{t}\right\rangle ,\\
\theta_{j} & \coloneqq\left[\gamma\left(1+\frac{1}{H}\right)^{3}\right]^{j}\sum_{t=1}^{T}\sum_{s\in\mathcal{S}}\left[\rho(P_{\pi^{\star}})^{j}\right]\big(s,\pi^{\star}(s)\big)\min\left\{ \beta_{n_{t}\left(s,\pi^{\star}\left(s\right)\right)}\big(s,\pi^{\star}(s)\big),\frac{1}{1-\gamma}\right\} ,\\
\xi_{j} & \coloneqq\left[\gamma\left(1+\frac{1}{H}\right)^{3}\right]^{j}\sum_{t=1}^{t_{\mathsf{mix}}(\delta)}\left\langle \rho(P_{\pi^{\star}})^{j},V^{\star}-V_{t}\right\rangle +\left[\gamma\left(1+\frac{1}{H}\right)^{3}\right]^{j+1}\left\langle \rho(P_{\pi^{\star}})^{j+1},V^{\star}-V_{0}\right\rangle , \\
\psi_{j} & \coloneqq\left[\gamma\left(1+\frac{1}{H}\right)^{3}\right]^{j}\sum_{t=t_{\mathsf{mix}}(\delta)}^{T}\Biggl[\sum_{s\in\mathcal{S},a\in\mathcal{A}}\left[\rho^{\pi^{\star}}(P^{\pi^{\star}})^{j}\right]\left(s,a\right)\sum_{i=1}^{n_{t}\left(s,a\right)}\eta_{i}^{n_{t}\left(s,a\right)}P_{s,a}\left(V^{\star}-V_{k_{i}\left(s,a\right)}\right)\\
 & \qquad\qquad\qquad\qquad\qquad\quad-\left(1+\frac{1}{H}\right)\frac{\left[\rho^{\pi^{\star}}(P^{\pi^{\star}})^{j}\right]\left(s_{t},a_{t}\right)}{\mu_{\mathsf{b}}\left(s_{t},a_{t}\right)}\sum_{i=1}^{n_{t}\left(s_{t},a_{t}\right)}\eta_{i}^{n_{t}\left(s_{t},a_{t}\right)}P_{s_{t},a_{t}}\left(V^{\star}-V_{k_{i}\left(s_{t},a_{t}\right)}\right)\Biggr], \\
%\intertext{and} \\
\phi_{j} & \coloneqq\gamma^{j+1}\left(1+\frac{1}{H}\right)^{3j+2}\sum_{t=0}^{T}\ind_{\left(s_{t},a_{t}\right)\in\mathcal{I}}\Bigg[\frac{\left[\rho^{\pi^{\star}}(P^{\pi^{\star}})^{j}\right]\left(s_{t},a_{t}\right)}{\mu_{\mathsf{b}}\left(s_{t},a_{t}\right)}P_{s_{t},a_{t}}\left(V^{\star}-V_{t}\right)\\
 & \qquad\qquad\qquad\qquad\qquad\qquad -\left(1+\frac{1}{H}\right)\sum_{s\in\mathcal{S},a\in\mathcal{A}}\left[\rho^{\pi^{\star}}(P^{\pi^{\star}})^{j}\right]\left(s,a\right)P_{s,a}\left(V^{\star}-V_{t}\right)\Bigg], 
\end{align*}
where we recall the definition of $\mathcal{I}_1$ in \eqref{eq:defn-I-Ic}.

Let us begin with the following basic inequality:
\begin{align}
V^{\star}\left(\rho\right)-V^{\widehat{\pi}}\left(\rho\right)=\big\langle \rho,V^{\star}-V^{\widehat{\pi}}\big\rangle  & \overset{\text{(i)}}{\leq}\left\langle \rho,V^{\star}-V_{T}\right\rangle \overset{\text{(ii)}}{\leq}\frac{1}{T}\sum_{t=1}^{T}\left\langle \rho,V^{\star}-V_{t}\right\rangle \overset{\text{(iii)}}{=}\frac{1}{T}\alpha_{0}. 
	\label{eq:Vstar-error-alpha0-Hoeffding}
\end{align}
Here, (i) holds true according to Lemma \ref{lemma:lcb-h-basics};
(ii) follows from the monotonicity of $V_{t}$ in $t$ (by construction);
and (iii) follows simply from the definition of $\alpha_{0}$. We
then turn attention to bounding $\alpha_{0}$, towards which we observe
that
\begin{align*}
\alpha_{0} & =\sum_{t=1}^{t_{\mathsf{mix}}(\delta)-1}\left\langle \rho,V^{\star}-V_{t}\right\rangle +\sum_{t=t_{\mathsf{mix}}(\delta)}^{T}\sum_{s\in\mathcal{S}}\rho\left(s\right)\min\bigg\{ Q^{\star}\big(s,\pi^{\star}(s)\big)-V_{t}(s),\frac{1}{1-\gamma}\bigg\}\\
 & \leq\sum_{t=1}^{t_{\mathsf{mix}}(\delta)-1}\left\langle \rho,V^{\star}-V_{t}\right\rangle +\sum_{t=t_{\mathsf{mix}}(\delta)}^{T}\sum_{s\in\mathcal{S}}\rho\left(s\right)\min\bigg\{ Q^{\star}\big(s,\pi^{\star}(s)\big)-Q_{t}\big(s,\pi^{\star}(s)\big),\frac{1}{1-\gamma}\bigg\}\\
 & \leq\sum_{t=1}^{t_{\mathsf{mix}}(\delta)}\left\langle \rho,V^{\star}-V_{t}\right\rangle +\underbrace{\gamma\sum_{t=t_{\mathsf{mix}}(\delta)}^{T}\sum_{s\in\mathcal{S}}\rho\left(s\right)\sum_{i=1}^{n_{t}\left(s,\pi^{\star}\left(s\right)\right)}\eta_{i}^{n_{t}\left(s,\pi^{\star}(s)\right)}P_{s,\pi^{\star}(s)}\left(V^{\star}-V_{k_{i}}\right)}_{\eqqcolon\,\zeta}\\
 & \quad+\underbrace{\sum_{t=1}^{T}\sum_{s\in\mathcal{S}}\rho\left(s\right)\min\left\{ \beta_{n_{t}\left(s,\pi^{\star}\left(s\right)\right)}\big(s,\pi^{\star}(s)\big),\frac{1}{1-\gamma}\right\} }_{=\,\theta_{0}}.
\end{align*}
Here, the first identity holds since $V^{\star}(s)=Q^{\star}\big(s,\pi^{\star}(s)\big)$
and $0\leq V^{\star}(s)-V_{t}(s)\leq1/(1-\gamma)$ for all $s\in\mathcal{S}$,
the second line relies on the fact that $V_{t}(s)\geq\max_{a}Q_{t}(s,a)\geq Q_{t}(s,\pi^{\star}(s))$,
while the last line invokes Lemma \ref{lemma:lcb-h-basics}. With
probability exceeding $1-\delta$, the first term $\zeta$ can be
upper bounded by
\begin{align*}
\zeta & \leq\gamma\sum_{t=t_{\mathsf{mix}}(\delta)}^{T}\sum_{s\in\mathcal{S}}\rho\left(s\right)\sum_{i=1}^{n_{t}\left(s,\pi^{\star}\left(s\right)\right)}\eta_{i}^{n_{t}\left(s,\pi^{\star}(s)\right)}P_{s,\pi^{\star}(s)}\left(V^{\star}-V_{k_{i}(s,\pi^{\star}(s))}\right)\\
 & =\gamma\sum_{t=t_{\mathsf{mix}}(\delta)}^{T}\sum_{s\in\mathcal{S},a\in\mathcal{A}}\mu_{\mathsf{b}}\left(s,a\right)\frac{\rho^{\pi^{\star}}\left(s,a\right)}{\mu_{\mathsf{b}}\left(s,a\right)}\sum_{i=1}^{n_{t}\left(s,a\right)}\eta_{i}^{n_{t}\left(s,a\right)}P_{s,\pi^{\star}(s)}\left(V^{\star}-V_{k_{i}}\right)\\
 & \overset{\text{(i)}}{\leq}\gamma\left(1+\frac{1}{H}\right)\sum_{t=t_{\mathsf{mix}}(\delta)}^{T} \ind\{\left(s_{t},a_{t}\right)\in\mathcal{I}\} \frac{\rho^{\pi^{\star}}\left(s_{t},a_{t}\right)}{\mu_{\mathsf{b}}\left(s_{t},a_{t}\right)}\sum_{i=1}^{n_{t}\left(s_{t},a_{t}\right)}\eta_{i}^{n_{t}\left(s_{t},a_{t}\right)}P_{s_{t},a_{t}}\left(V^{\star}-V_{k_{i}\left(s_{t},a_{t}\right)}\right)+\psi_{0}\\
 & \overset{\text{(ii)}}{=}\gamma\left(1+\frac{1}{H}\right)\sum_{t=t_{\mathsf{mix}}(\delta)}^{T} \ind\{\left(s_{t},a_{t}\right)\in\mathcal{I}\} \frac{\rho^{\pi^{\star}}\left(s_{t},a_{t}\right)}{\mu_{\mathsf{b}}\left(s_{t},a_{t}\right)}\left(\sum_{j=n_{t}(s_{t},a_{t})}^{n_{T}(s_{t},a_{t})}\eta_{n_{t}(s_{t},a_{t})}^{j}\right)P_{s_{t},a_{t}}\left(V^{\star}-V_{t}\right)+\psi_{0}\\
 & \overset{\text{(iii)}}{\leq}\gamma\left(1+\frac{1}{H}\right)^{2}\sum_{t=0}^{T}\ind\{\left(s_{t},a_{t}\right)\in\mathcal{I}\}\frac{\rho^{\pi^{\star}}\left(s_{t},a_{t}\right)}{\mu_{\mathsf{b}}\left(s_{t},a_{t}\right)}P_{s_{t},a_{t}}\left(V^{\star}-V_{t}\right)+\psi_{0}\\
 & =\gamma\left(1+\frac{1}{H}\right)^{3}\sum_{t=0}^{T}\sum_{s\in\mathcal{S},a\in\mathcal{A}}\rho^{\pi^{\star}}\left(s,a\right)P_{s,a}\left(V^{\star}-V_{t}\right)+\psi_{0}+\phi_{0}\\
 & =\gamma\left(1+\frac{1}{H}\right)^{3}\sum_{t=0}^{T}\left\langle \rho P_{\pi^{\star}},V^{\star}-V_{t}\right\rangle +\psi_{0}+\phi_{0}\\
 & \leq\alpha_{1}+\psi_{0}+\phi_{0}+\gamma\left(1+\frac{1}{H}\right)^{3}\left\langle \rho P_{\pi^{\star}},V^{\star}-V_{0}\right\rangle ,
\end{align*}
where we remind the reader of our notation $\rho^{\pi^{\star}}$ in
\eqref{eq:defn-rho-pistar}. Here, (i) is valid (i.e., $\rho(s_{t},a_{t})/\mu_{\mathsf{b}}(s,a)$
is well defined for $t\geq t_{\mathsf{mix}}(\delta)$) due to Lemma
\ref{lemma:states-small-prob}; (ii) holds by grouping the terms in
the previous line; and (iii) utilizes Lemma~\ref{lemma:step-size} and the property that $V^{\star}\geq V_t$ (cf.~Lemma~\ref{lemma:lcb-h-basics}).
Therefore, we arrive at
\begin{align*}
\alpha_{0} & \leq\sum_{t=1}^{t_{\mathsf{mix}}(\delta)}\left\langle \rho,V^{\star}-V_{t}\right\rangle +\zeta+\theta_{0}\\
 & \leq\sum_{t=1}^{t_{\mathsf{mix}}(\delta)}\left\langle \rho,V^{\star}-V_{t}\right\rangle +\alpha_{1}+\psi_{0}+\phi_{0}+\gamma\left(1+\frac{1}{H}\right)^{3}\left\langle \rho P_{\pi^{\star}},V^{\star}-V_{0}\right\rangle +\theta_{0}\\
 & =\alpha_{1}+\xi_{0}+\theta_{0}+\psi_{0}+\phi_{0},
\end{align*}
where we have used the definition of $\xi_{0}$. Repeat the same argument to reach
\[
\alpha_{j}\leq\alpha_{j+1}+\xi_{j}+\theta_{j}+\psi_{j}+\phi_{j}
\]
for all $j\geq1$. This in turn allows us to conclude that
\begin{equation}
\alpha_{0}\leq\underbrace{\limsup_{j\to\infty}\alpha_{j}}_{\eqqcolon\,\alpha}+\underbrace{\sum_{j=0}^{\infty}\xi_{j}}_{\eqqcolon\,\xi}+\underbrace{\sum_{j=0}^{\infty}\theta_{j}}_{\eqqcolon\,\theta}+\underbrace{\sum_{j=0}^{\infty}\psi_{j}}_{\eqqcolon\,\psi}+\underbrace{\sum_{j=0}^{\infty}\phi_{j}}_{\eqqcolon\,\phi}.\label{eq:alpha0-UB-asyncQ}
\end{equation}
We will then bound the terms $\alpha$, $\xi$, $\theta$, $\psi$
and $\phi$ separately in the subsequent steps. Before continuing,
we make note of a useful result.

\begin{lemma}\label{lemma:useful}Recall that $H=\left\lceil \frac{4}{1-\gamma}\log\frac{ST}{\delta}\right\rceil $
for some $0<\delta<1$. For any vector with non-negative entries $V\in\mathbb{R}^{d}$
, we have
\begin{equation}
\sum_{j=0}^{\infty}\left[\gamma\left(1+\frac{1}{H}\right)^{3}\right]^{j}\left\langle \rho(P_{\pi^{\star}})^{j},V\right\rangle \lesssim\frac{1}{1-\gamma}\left\langle d_{\rho}^{\star},V\right\rangle +\frac{\delta}{ST^{4}\left(1-\gamma\right)}\left\Vert V\right\Vert _{\infty}.\label{eq:useful}
\end{equation}
\end{lemma}\begin{proof}See Appendix \ref{appendix:proof-eq-useful}.\end{proof}

\subsection{Step 2: bounding each term in \eqref{eq:alpha0-UB-asyncQ}}

\paragraph{Step 2.1: bounding $\alpha$.}

It is first observed that
\[
\alpha=\limsup_{j\to\infty}\left[\gamma\left(1+\frac{1}{H}\right)^{3}\right]^{j}\sum_{t=1}^{T}\left\langle \rho(P_{\pi^{\star}})^{j},V^{\star}-V_{t}\right\rangle \overset{\text{(i)}}{\leq}\frac{T}{1-\gamma}\limsup_{k\to\infty}\left[\gamma\left(1+\frac{1}{H}\right)^{3}\right]^{k}\overset{\text{(ii)}}{=}0.
\]
Here, (i) is valid since $\rho(P_{\pi^{\star}})^{j}$ is a probability
distribution over $\mathcal{S}$ and $0\leq V^{\star}-V_{t}\leq1/(1-\gamma)$
holds for all $1\leq t\leq T$; (ii) holds since 
	\begin{align}
		\gamma\left(1+\frac{1}{H}\right)^{3}\leq\gamma\left(1+\frac{1-\gamma}{4}\right)^{2}<1
		\label{eq:prod-1-H3-gamma}
	\end{align}
for all $\gamma < 1$. 

\paragraph{Step 2.2: bounding $\xi$.}

By utilizing (\ref{eq:useful}) and \eqref{eq:prod-1-H3-gamma}, we can demonstrate that
\begin{align*}
 \xi 
%& = \sum_{j=0}^{\infty} \xi_j \\
	&= \sum_{t=1}^{t_{\mathsf{mix}}(\delta)} \left\{ \sum_{j=0}^{\infty}\left[\gamma\left(1+\frac{1}{H}\right)^{3}\right]^{j}  \left\langle \rho P_{\pi^{\star}}^{j},V^{\star}-V_{t}\right\rangle  \right\} +  
\sum_{j=0}^{\infty}\left[\gamma\left(1+\frac{1}{H}\right)^{3}\right]^{j+1} \big\langle \rho(P_{\pi^{\star}})^{j+1},V^{\star}-V_{0}\big\rangle  \\
 & \lesssim\frac{1}{1-\gamma}\sum_{t=0}^{t_{\mathsf{mix}}(\delta)}\left\langle d_{\rho}^{\star},V^{\star}-V_{t}\right\rangle +\frac{1}{ST^{4}\left(1-\gamma\right)}\frac{t_{\mathsf{mix}}(\delta)+1}{1-\gamma}\\
 & \lesssim\frac{t_{\mathsf{mix}}(\delta)}{\left(1-\gamma\right)^{2}}+\frac{t_{\mathsf{mix}}(\delta)}{T^{4}\left(1-\gamma\right)^{2}}\\
 & \lesssim\frac{t_{\mathsf{mix}}}{\left(1-\gamma\right)^{2}}\log\frac{1}{\delta}+\frac{t_{\mathsf{mix}}}{T^{4}\left(1-\gamma\right)^{2}}\log\frac{1}{\delta}. 
\end{align*}
Here, the second line holds due to (\ref{eq:useful}) and the basic fact $0\leq V^{\star}(s)-V_t (s) \leq \frac{1}{1-\gamma}$, 
the penultimate line makes use of the fact $\|V^{\star}-V_{t}\|_{\infty}\leq\frac{1}{1-\gamma}$ once again, 
whereas the last line holds since $t_{\mathsf{mix}}(\delta)\lesssim t_{\mathsf{mix}}\log\frac{1}{\delta}$. 

\paragraph{Step 2.3: bounding $\theta$.}

When it comes to $\theta$, we can deduce that
\begin{align*}
\theta & =\sum_{j=0}^{\infty}\left[\gamma\left(1+\frac{1}{H}\right)^{3}\right]^{j}\sum_{t=1}^{T}\sum_{s\in\mathcal{S}}\left[\rho(P_{\pi^{\star}})^{j}\right]\left(s\right)\min\left\{ \beta_{n_{t}\left(s,\pi^{\star}\left(s\right)\right)},\frac{1}{1-\gamma}\right\} \\
 & =\sum_{t=1}^{T}\sum_{j=0}^{\infty}\left[\gamma\left(1+\frac{1}{H}\right)^{3}\right]^{j}\sum_{s\in\mathcal{S}}\left[\rho(P_{\pi^{\star}})^{j}\right]\left(s\right)\min\left\{ \beta_{n_{t}\left(s,\pi^{\star}\left(s\right)\right)},\frac{1}{1-\gamma}\right\} \\
 & \overset{\text{(i)}}{\lesssim}\frac{1}{1-\gamma}\sum_{t=1}^{T}\sum_{s\in\mathcal{S}}d_{\rho}^{\star}\left(s\right)\min\left\{ \beta_{n_{t}\left(s,\pi^{\star}\left(s\right)\right)},\frac{1}{1-\gamma}\right\} +\frac{1}{ST^{4}\left(1-\gamma\right)}\frac{T}{1-\gamma}\\
 & \lesssim\sum_{s\in\mathcal{S}}\sum_{t=1}^{t_{\mathsf{burn}\text{-}\mathsf{in}}(s)}\frac{d_{\rho}^{\star}\left(s\right)}{\left(1-\gamma\right)^{2}}+\sum_{s\in\mathcal{S}}\sum_{t=t_{\mathsf{burn}\text{-}\mathsf{in}}(s)+1}^{T}d_{\rho}^{\star}\left(s\right)\sqrt{\frac{H\iota}{n_{t}\big(s,\pi^{\star}(s)\big)\left(1-\gamma\right)^{4}}}+\frac{1}{T^{3}\left(1-\gamma\right)^{2}}\\
 & \overset{\text{(ii)}}{\asymp}\sum_{s\in\mathcal{S}}\frac{d_{\rho}^{\star}\left(s\right)}{\mu_{\mathsf{b}}\left(s,\pi^{\star}(s)\right)}\frac{t_{\mathsf{mix}}\iota}{\left(1-\gamma\right)^{2}}+\sum_{s\in\mathcal{S}}\sum_{t=t_{\mathsf{burn}\text{-}\mathsf{in}}(s)+1}^{T}d_{\rho}^{\star}\left(s\right)\sqrt{\frac{H\iota}{t\mu_{\mathsf{b}}\big(s,\pi^{\star}(s)\big)\left(1-\gamma\right)^{4}}}+\frac{1}{T^{3}\left(1-\gamma\right)^{2}}\\
 & \overset{\text{(iii)}}{\lesssim}\frac{C^{\star}St_{\mathsf{mix}}\iota}{\left(1-\gamma\right)^{2}}+\sum_{s\in\mathcal{S}}d_{\rho}^{\star}\big(s,\pi^{\star}(s)\big)\sqrt{\frac{HT\iota}{\mu_{\mathsf{b}}\big(s,\pi^{\star}(s)\big)\left(1-\gamma\right)^{4}}} +\frac{1}{T^{3}\left(1-\gamma\right)^{2}}\\
 & \overset{\text{(iv)}}{\lesssim}\frac{C^{\star}St_{\mathsf{mix}}\iota}{\left(1-\gamma\right)^{2}}+\sqrt{\frac{C^{\star}HT\iota}{\left(1-\gamma\right)^{4}}}\sum_{s\in\mathcal{S}}\sqrt{d_{\rho}^{\star}\big(s,\pi^{\star}(s)\big)}\\
 & \overset{\text{(v)}}{\lesssim}\frac{C^{\star}St_{\mathsf{mix}}\iota}{\left(1-\gamma\right)^{2}}+\sqrt{\frac{C^{\star}ST\iota^{2}}{\left(1-\gamma\right)^{5}}},
\end{align*}
where we define, for each $s\in\mathcal{S}$, 
\[
t_{\mathsf{burn}\text{-}\mathsf{in}}(s)\coloneqq C_{\mathsf{burn\text{-}in}}\frac{t_{\mathsf{mix}}}{\mu_{\mathsf{b}}\big(s,\pi^{\star}(s)\big)}\log\left(\frac{ST}{\delta}\right)
\]
for some sufficiently large constant $C_{\mathsf{burn\text{-}in}}>0$.
Here, (i) relies on (\ref{eq:useful});
%and the fact that $0\leq\beta_{n}\leq1/(1-\gamma)$
%for all $n\geq0$; 
(ii) utilizes \citet[Lemma 8]{li2021sample}; (iii)
follows from the fact that
\begin{equation}
\sum_{t=1}^{T}\frac{1}{\sqrt{t}}\leq1+\int_{1}^{T}\frac{1}{\sqrt{x}}\mathrm{d}x=1+2\left(\sqrt{T}-1\right)\leq2\sqrt{T};\label{eq:sum-sqrt}
\end{equation}
(iv) uses Assumption \ref{assumption:policy}; and (v) invokes the
Cauchy-Schwarz inequality and the fact that $\sum_{s}d_{\rho}^{\star}\big(s,\pi^{\star}(s)\big)=1$. 

\paragraph{Step 2.4: bounding $\psi$. }

Recall that
\begin{align*}
\psi_{j} & \coloneqq\gamma\left[\gamma\left(1+\frac{1}{H}\right)^{3}\right]^{j}\sum_{t=t_{\mathsf{mix}}(\delta)}^{T}\Biggl[\sum_{s\in\mathcal{S},a\in\mathcal{A}}\left[\rho^{\pi^{\star}}(P^{\pi^{\star}})^{j}\right]\left(s,a\right)\sum_{i=1}^{n_{t}\left(s,a\right)}\eta_{i}^{n_{t}\left(s,a\right)}P_{s,a}\left(V^{\star}-V_{k_{i}\left(s,a\right)}\right)\\
 & \qquad\qquad\qquad\qquad\qquad\quad-\left(1+\frac{1}{H}\right)\frac{\left[\rho^{\pi^{\star}}(P^{\pi^{\star}})^{j}\right]\left(s_{t},a_{t}\right)}{\mu_{\mathsf{b}}\left(s_{t},a_{t}\right)}\sum_{i=1}^{n_{t}\left(s_{t},a_{t}\right)}\eta_{i}^{n_{t}\left(s_{t},a_{t}\right)}P_{s_{t},a_{t}}\left(V^{\star}-V_{k_{i}\left(s_{t},a_{t}\right)}\right)\Biggr].
\end{align*}
In order to bound $\psi$, we make the observation that
\begin{align*}
\psi & =\sum_{j=0}^{\infty}\gamma\left[\gamma\left(1+\frac{1}{H}\right)^{3}\right]^{j}\sum_{t=t_{\mathsf{mix}}(\delta)}^{T}\Biggl[\sum_{s\in\mathcal{S},a\in\mathcal{A}}\left[\rho^{\pi^{\star}}(P^{\pi^{\star}})^{j}\right]\left(s,a\right)\sum_{i=1}^{n_{t}\left(s,a\right)}\eta_{i}^{n_{t}\left(s,a\right)}P_{s,a}\left(V^{\star}-V_{k_{i}(s,a)}\right)\\
 & \qquad\qquad\qquad\qquad\qquad\qquad\quad-\left(1+\frac{1}{H}\right)\frac{\left[\rho^{\pi^{\star}}(P^{\pi^{\star}})^{j}\right]\left(s_{t},a_{t}\right)}{\mu_{\mathsf{b}}\left(s_{t},a_{t}\right)}\sum_{i=1}^{n_{t}\left(s_{t},a_{t}\right)}\eta_{i}^{n_{t}\left(s_{t},a_{t}\right)}P_{s_{t},a_{t}}\left(V^{\star}-V_{k_{i}\left(s_{t},a_{t}\right)}\right)\Biggr]\\
 & =\sum_{t=t_{\mathsf{mix}}(\delta)}^{T}\Biggl[\sum_{s\in\mathcal{S},a\in\mathcal{A}}\widetilde{d}\left(s,a\right)\sum_{i=1}^{n_{t}\left(s,a\right)}\eta_{i}^{n_{t}\left(s,a\right)}P_{s,a}\left(V^{\star}-V_{k_{i}(s,a)}\right)\\
 & \qquad\qquad-\left(1+\frac{1}{H}\right)\frac{\widetilde{d}\left(s_{t},a_{t}\right)}{\mu_{\mathsf{b}}\left(s_{t},a_{t}\right)}\sum_{i=1}^{n_{t}\left(s_{t},a_{t}\right)}\eta_{i}^{n_{t}\left(s_{t},a_{t}\right)}P_{s_{t},a_{t}}\left(V^{\star}-V_{k_{i}\left(s_{t},a_{t}\right)}\right)\Biggr].
\end{align*}
Here, $\widetilde{d}(\cdot,\cdot)$ is defined such that
\[
\widetilde{d}\left(s,a\right)\coloneqq\sum_{j=0}^{\infty}\gamma\left[\gamma\left(1+\frac{1}{H}\right)^{3}\right]^{j}\left[\rho^{\pi^{\star}}(P^{\pi^{\star}})^{j}\right]\left(s,a\right)
\]
for any $(s,a)\in\mathcal{S}\times\mathcal{A}$. For any $t_{\mathsf{mix}}(\delta)\leq t\leq T$
and any $(s,a)\in\mathcal{I}$, let us define
\[
f_{t}\left(s,a\right)=\frac{\widetilde{d}\left(s,a\right)}{\mu_{\mathsf{b}}\left(s,a\right)}\sum_{i=1}^{n_{t}\left(s,a\right)}\eta_{i}^{n_{t}\left(s,a\right)}P_{s,a}\left(V^{\star}-V_{k_{i}\left(s,a\right)}\right), 
\]
allowing us to rewrite
\begin{align}
\psi & =\sum_{t=t_{\mathsf{mix}}(\delta)}^{T}\left\{ \mathbb{E}_{(s,a)\sim\mu_{\mathsf{b}}}\left[f_{t}\left(s,a\right)\right]-\left(1+\frac{1}{H}\right)f_{t}\left(s_{t},a_{t}\right)\right\} .
	\label{eq:defn-psi-ft-hoeffding}
\end{align}

Let us take a moment to look at some properties of $f_t$. It is straightforward to check that
\begin{itemize}
	\item[(i)] when $a\neq\pi^{\star}(s)$,
one has $f_{t}(s,a)=0$; 
	\item[(ii)] $f_{t}(s,a)$ is monotonically decreasing
in $t$.
\end{itemize}
The latter property follows from the non-decreasing property of $V_{t}$
in $t$ (by construction), and that $\eta_{i}^{n_{t}(s,a)}$ is decreasing
in $t$, as well as $\sum_{i=1}^{n_{t}(s,a)}\eta_{i}^{n_{t}(s,a)}=1$
(cf.~Lemma \ref{lemma:step-size}). On the other hand, when $a=\pi^{\star}(s)$
and $(s,a)\in\mathcal{I}$, we can invoke (\ref{eq:useful}) to arrive
at
\begin{align}
	\widetilde{d}\big(s, \pi^{\star}(s) \big) 
	\lesssim \frac{1}{1-\gamma} d_{\rho}^{\star}(s) + \frac{\delta}{ST^4(1-\gamma)} ,
\end{align}
and consequently, 
%This taken collectively with $\|V^{\star}-V_{k_{i}\left(s,a\right)}\|_{\infty}\leq \frac{1}{1-\gamma}$  and $\sum_{i=1}^{n_{t}(s,a)}\eta_{i}^{n_{t}(s,a)}=1$ 
% allows one to deduce that
%
\begin{align}
	f_{t}\left(s,a\right) & \lesssim 
	\left\{ \frac{1}{1-\gamma}\frac{d_{\rho}^{\star}\left(s,a\right)}{\mu_{\mathsf{b}}\left(s,a\right)}
	+ \frac{\delta}{ST^{4}\left(1-\gamma\right)}\frac{1}{\mu_{\mathsf{b}}\left(s,a\right)}  \right\} \sum_{i=1}^{n_{t}\left(s,a\right)}\eta_{i}^{n_{t}\left(s,a\right)}P_{s,a}\left(V^{\star}-V_{k_{i}\left(s,a\right)}\right) \nonumber \\
& \lesssim 
	 \frac{1}{1-\gamma}\frac{d_{\rho}^{\star}\left(s,a\right)}{\mu_{\mathsf{b}}\left(s,a\right)}
	+ \frac{\delta}{ST^{4}\left(1-\gamma\right)}\frac{1}{\mu_{\mathsf{b}}\left(s,a\right)}   \nonumber \\	
 & \leq\frac{C^{\star}}{\left(1-\gamma\right)^{2}}+\frac{\delta}{ST^{4}\left(1-\gamma\right)}\frac{ST}{\delta}\nonumber \\
	& \leq\frac{c_{10}C^{\star}}{\left(1-\gamma\right)^{2}}\coloneqq C_{f} . \label{eq:ft-UB-Cf}
\end{align}
for some constant $c_{10}\geq  1$. 
Here, the second line follows from Assumption \ref{assumption:policy},
the properties that $\sum_{i=1}^{n_{t}(s,a)}\eta_{i}^{n_{t}(s,a)}=1$,
$0\leq V^{\star}(s)-V_{t}(s)\leq1/(1-\gamma)$ for all $0\leq t\leq T$; 
the third line is valid since  $\mu_{\mathsf{b}}(s,\pi^{\star}(s))\geq\delta/(ST)$
when $(s,\pi^{\star}(s))\in\mathcal{I}$.

We now proceed to bound \eqref{eq:defn-psi-ft-hoeffding}. 
It is worth noting that both $f_{t}$ and $(s_{t},a_{t})$ depend
on $s_{0},a_{0},s_{1},\ldots,s_{t-1},a_{t-1}$. To handle such
statistical dependency, we define
\[
K\coloneqq\left\lfloor \frac{T}{\tau}\right\rfloor \qquad\text{where}\quad\tau\coloneqq t_{\mathsf{mix}}(\delta/T^{2})\lesssim t_{\mathsf{mix}}\log\frac{T}{\delta}.
\]
 Armed with this notation, one can decompose
\begin{align*}
\psi & =\sum_{t=1}^{\tau}\sum_{k=1}^{K-1}\left\{ \mathbb{E}_{(s,a)\sim\mu_{\mathsf{b}}}\left[f_{k\tau+t}\left(s,a\right)\right]-\left(1+\frac{1}{H}\right)f_{k\tau+t}\left(s_{k\tau+t},a_{k\tau+t}\right)\right\} \\
 & \quad+\left(\sum_{t=t_{\mathsf{mix}}(\delta)}^{\tau}+\sum_{t=K\tau+1}^{T}\right)\left\{ \mathbb{E}_{(s,a)\sim\mu_{\mathsf{b}}}\left[f_{t}\left(s,a\right)\right]-\left(1+\frac{1}{H}\right)f_{t}\left(s_{t},a_{t}\right)\right\} \\
 & =\underbrace{\sum_{i=1}^{\tau}\sum_{k=1}^{K-1}\left\{ \mathbb{E}_{(s,a)\sim\mu_{\mathsf{b}}}\left[f_{(k-1)\tau+i}\left(s,a\right)\right]-\left(1+\frac{1}{H}\right)f_{(k-1)\tau+i}\left(s_{k\tau+i},a_{k\tau+i}\right)\right\} }_{\eqqcolon\kappa_{1}}\\
 & \quad+\underbrace{\sum_{i=1}^{\tau}\sum_{k=1}^{K-1}\left\{ \mathbb{E}_{(s,a)\sim\mu_{\mathsf{b}}}\left[f_{k\tau+i}\left(s,a\right)\right]-\mathbb{E}_{(s,a)\sim\mu_{\mathsf{b}}}\left[f_{(k-1)\tau+i}\left(s,a\right)\right]\right\} }_{\eqqcolon\kappa_{2}}\\
 & \quad+\underbrace{\left(1+\frac{1}{H}\right)\sum_{i=1}^{\tau}\sum_{k=1}^{K-1}\left[f_{(k-1)\tau+i}\left(s_{k\tau+i},a_{k\tau+i}\right)-f_{k\tau+i}\left(s_{k\tau+i},a_{k\tau+i}\right)\right]}_{\eqqcolon\kappa_{3}}\\
 & \quad+\underbrace{\left(\sum_{t=t_{\mathsf{mix}}(\delta)}^{\tau}+\sum_{t=K\tau+1}^{T}\right)\left\{ \mathbb{E}_{(s,a)\sim\mu_{\mathsf{b}}}\left[f_{t}\left(s,a\right)\right]-\left(1+\frac{1}{H}\right)f_{t}\left(s_{t},a_{t}\right)\right\} }_{\eqqcolon\kappa_{4}}.
\end{align*}
In what follows, we bound $\kappa_{1}$, $\kappa_{2}$, $\kappa_{3}$ and $\kappa_{4}$
respectively. 
\begin{itemize}
\item The term $\kappa_{4}$ can be easily bounded using \eqref{eq:ft-UB-Cf}
as follows
\[
\kappa_{4}\leq\left(\sum_{t=t_{\mathsf{mix}}(\delta)}^{\tau}+\sum_{t=K\tau+1}^{T}\right)\mathbb{E}_{(s,a)\sim\mu_{\mathsf{b}}}\left[f_{t}\left(s,a\right)\right]\leq2\tau C_{f}\asymp\frac{C^{\star}t_{\mathsf{mix}}}{\left(1-\gamma\right)^{2}}\log\left(\frac{T}{\delta}\right).
\]
\item With regards to $\text{\ensuremath{\kappa_{3}}}$, we make the observation
that 
\begin{align*}
\kappa_{3} & =\left(1+\frac{1}{H}\right)\sum_{t=\tau+1}^{K\tau}\left[f_{t-\tau}\left(s_{t},a_{t}\right)-f_{t}\left(s_{t},a_{t}\right)\right]\\
 & \overset{\text{(i)}}{\leq}\left(1+\frac{1}{H}\right)\sum_{t=\tau+1}^{K\tau}\sum_{s\in\mathcal{S},a\in\mathcal{A}}\left[f_{t-\tau}\left(s,a\right)-f_{t}\left(s,a\right)\right]\\
 & \overset{\text{(ii)}}{=}\left(1+\frac{1}{H}\right)\sum_{t=\tau+1}^{K\tau}\sum_{s\in\mathcal{S}}\left[f_{t-\tau}\big(s,\pi^{\star}(s)\big)-f_{t}\big(s,\pi^{\star}(s)\big)\right]\\
 & =\left(1+\frac{1}{H}\right)\left\{ \sum_{t=1}^{\tau}\sum_{s\in\mathcal{S}}f_{t}\big(s,\pi^{\star}(s)\big)-\sum_{t=(K-1)\tau+1}^{K\tau}\sum_{s\in\mathcal{S}}f_{t}\big(s,\pi^{\star}(s)\big)\right\} \\
 & \leq2\left(1+\frac{1}{H}\right)\tau SC_{f}\asymp\frac{C^{\star}St_{\mathsf{mix}}}{\left(1-\gamma\right)^{2}}\log\left(\frac{T}{\delta}\right).
\end{align*}
Here, (i) holds since $f_{t}(s,a)$ is monotonically decreasing in
$t$; and (ii) holds since, by definition, $f(s,a)=0$ if $a\neq\pi^{\star}(s)$.
\item Similarly, $\kappa_{2}$ can be bounded by
\begin{align*}
\kappa_{2} & =\sum_{t=\tau+1}^{K\tau}\left\{ \mathbb{E}_{(s,a)\sim\mu_{\mathsf{b}}}\left[f_{t}\left(s,a\right)\right]-\mathbb{E}_{(s,a)\sim\mu_{\mathsf{b}}}\left[f_{t-\tau}(s,a)\right]\right\} \\
 & =\sum_{t=(K-1)\tau+1}^{K\tau}\mathbb{E}_{(s,a)\sim\mu_{\mathsf{b}}}\left[f_{t}\left(s,a\right)\right]-\sum_{t=1}^{\tau}\mathbb{E}_{(s,a)\sim\mu_{\mathsf{b}}}\left[f_{t}\left(s,a\right)\right]\\
 & \lesssim\tau C_{f}\asymp\frac{C^{\star}t_{\mathsf{mix}}}{\left(1-\gamma\right)^{2}}\log\left(\frac{T}{\delta}\right).
\end{align*}
\item Finally, we turn attention to bounding $\kappa_{1}$. For each $1\leq i\leq\tau$,
we will bound
\[
	\xi_{i}\coloneqq\sum_{k=1}^{K-1}\left\{
	\mathop{\mathbb{E}}\limits_{(s,a)\sim\mu_{\mathsf{b}}}\left[f_{(k-1)\tau+i}\left(s,a\right)\right]
	-\left(1+\frac{1}{H}\right)f_{(k-1)\tau+i}\left(s_{k\tau+i},a_{k\tau+i}\right) \right\} 
\]
respectively. We need the following lemma to decouple the complicated
statistical dependency. 
\begin{lemma}\label{lemma:coupling}
One can construct an auxiliary set of random variables $\left\{ \left(s_{k}^{i},a_{k}^{i}\right):1\leq k\leq K-1\right\}$ satisfying
\begin{subequations}
\label{eq:auxiliary-independent-equiv-hoeffding}	
\begin{equation}
\big\{ \left(s_{k}^{i},a_{k}^{i}\right):1\leq k\leq K-1\big\} \overset{\mathsf{i.i.d.}}{\sim}\mu_{\mathsf{b}} ,
	\label{eq:auxiliary-independent-hoeffding}
\end{equation}
%satisfyinge
%
\begin{equation}
	\mathbb{P}\Big\{ \left(s_{k}^{i},a_{k}^{i}\right)=\left(s_{k\tau+i},a_{k\tau+i}\right)\quad\text{for all }1\leq k\leq K-1\Big\}
	\geq 1-\frac{\delta}{T}, 
\label{eq:auxiliary-equivalent-hoeffding}
\end{equation}
and
\begin{equation}
\left(s_{k}^{i},a_{k}^{i}\right)\text{ is independent of }\big\{ \left(s_{t},a_{t}\right):0\leq t\leq\left(k-1\right)\tau+i\big\} .
\label{eq:auxiliary-independent2-hoeffding}
\end{equation}
\end{subequations}
\end{lemma}\begin{proof}See Appendix \ref{appendix:proof-lemma-coupling}.\end{proof}
With the above set of auxiliary random variables $\left\{ \left(s_{k}^{i},a_{k}^{i}\right):1\leq k\leq K-1\right\}$ in place, 
one can obtain
\begin{align*}
	\xi_{i} & =\sum_{k=1}^{K-1}\left\{ \mathop{\mathbb{E}}\limits_{(s,a)\sim\mu_{\mathsf{b}}}\left[f_{(k-1)\tau+i}\left(s,a\right)\right]-\left(1+\frac{1}{H}\right)f_{(k-1)\tau+i}\left(s_{k}^{i},a_{k}^{i}\right)\right\} \\
 & =-\left(1+\frac{1}{H}\right)\sum_{k=1}^{K-1}\left\{ f_{(k-1)\tau+i}\left(s_{k}^{i},a_{k}^{i}\right)-\mathop{\mathbb{E}}\limits_{(s,a)\sim\mu_{\mathsf{b}}}\left[f_{(k-1)\tau+i}\left(s,a\right)\right]\right\} -\frac{1}{H}\sum_{k=1}^{K-1}\mathop{\mathbb{E}}\limits_{(s,a)\sim\mu_{\mathsf{b}}}\left[f_{(k-1)\tau+i}\left(s,a\right)\right]
\end{align*}
with probability exceeding $1-\delta/T$. 
Recognizing the property \eqref{eq:auxiliary-independent2-hoeffding}, we are ready to use
	the Freedman inequality (cf.~\citet[Theorem 3]{li2021breaking}) to bound
$\xi_{i}$. Introduce the random variable 
\begin{equation}
	X_{k}=f_{(k-1)\tau+i}(s_{k}^{i},a_{k}^{i})- \mathop{\mathbb{E}}\limits_{(s,a)\sim\mu_{\mathsf{b}}}\big[f_{(k-1)\tau+i}(s,a) \big], 
\end{equation}
and define a filtration $\mathcal{F}_{0}\subset\mathcal{F}_{1}\subset\cdots\subset\mathcal{F}_{K-1}$
with 
\[
\mathcal{F}_{k-1}=\sigma\left\{ \left\{ \left(s_{k}^{i},a_{k}^{i}\right)\right\} _{k=1}^{k-1},\left\{ \left(s_{t},a_{t}\right)\right\} _{t=0}^{\left(k-1\right)\tau+i}\right\} \quad\text{for}\quad1\leq k\leq K-1.
\]
It is straightforward to verify that 
\[
\left|X_{k}\right|\leq R\coloneqq C_{f},\quad\mathbb{E}\left[X_{k}\mid\mathcal{F}_{k-1}\right]=0\quad\text{for all }1\leq k\leq K-1,
\]
and
\begin{align}
W & \coloneqq\sum_{k=1}^{K-1}\mathbb{E}\left[X_{k}^{2}\mid\mathcal{F}_{k-1}\right]\leq\sum_{k=1}^{K-1}\mathbb{E}\left[f_{(k-1)\tau+i}^{2}(s_{k}^{i},a_{k}^{i})\mid\mathcal{F}_{k-1}\right] \notag\\
	& \leq C_{f}\sum_{k=1}^{K-1}\mathbb{E}\left[f_{(k-1)\tau+i}(s_{k}^{i},a_{k}^{i})\mid\mathcal{F}_{k-1}\right]=C_{f}\sum_{k=1}^{K-1}
	\mathop{\mathbb{E}}\limits_{(s,a)\sim\mu_{\mathsf{b}}}\left[f_{(k-1)\tau+i}\left(s,a\right)\right] \label{eq:W-UB-Ef-hoeffding}\\
 & \leq C_{f}^{2}K.
\end{align}
Invoke the Freedman inequality in \citet[Theorem 3]{li2021breaking} to
show that for any integer $m\geq1$,
\begin{align*}
\left|\sum_{k=1}^{K-1}X_{k}\right| & \leq\sqrt{8\max\left\{ W,\frac{C_{f}^{2}K}{2^{m}}\right\} \log\frac{2Tm}{\delta}}+\frac{4}{3}R\log\frac{2Tm}{\delta}\\
 & \leq\sqrt{8W\log\frac{2Tm}{\delta}}+\sqrt{8\frac{C_{f}^{2}K}{2^{m}}\log\frac{2Tm}{\delta}}+\frac{4}{3}C_{f}\log\frac{2Tm}{\delta}\\
 & \leq\frac{1}{2HC_{f}}W+4HC_{f}\log\frac{2Tm}{\delta}+C_{f}\sqrt{8\frac{K}{2^{m}}\log\frac{2Tm}{\delta}}+\frac{4}{3}C_{f}\log\frac{2Tm}{\delta}\\
 & =\frac{1}{2H}\sum_{k=1}^{K-1}\mathbb{E}_{(s,a)\sim\mu_{\mathsf{b}}}\left[f_{(k-1)\tau+i}\left(s,a\right)\right]+O\left(HC_{f}\log\frac{T}{\delta}\right)
\end{align*}
holds with probability exceeding $1-\delta/T$. Here, the penultimate
line relies on the AM-GM inequality, whereas the last line holds by using \eqref{eq:W-UB-Ef-hoeffding} and
taking $m\asymp\log K\lesssim\log T\lesssim T$. Consequently, we
see that with probability exceeding $1-\delta/T$, 
\begin{align*}
\xi_{i} & =-\left(1+\frac{1}{H}\right)\sum_{k=1}^{K-1}X_{k}-\frac{1}{H}\sum_{k=1}^{K-1}\mathop{\mathbb{E}}\limits_{(s,a)\sim\mu_{\mathsf{b}}}\left[f_{(k-1)\tau+i}\left(s,a\right)\right]\\
 & \leq2\left|\sum_{k=1}^{K-1}X_{k}\right|-\frac{1}{H}\sum_{k=1}^{K-1}\mathop{\mathbb{E}}\limits_{(s,a)\sim\mu_{\mathsf{b}}}\left[f_{(k-1)\tau+i}\left(s,a\right)\right]\\
 & \lesssim HC_{f}\log\frac{T}{\delta}\lesssim\frac{C^{\star}\iota}{\left(1-\gamma\right)^{3}}\log\frac{T}{\delta}.
\end{align*}
As a result, with probability exceeding $1-\delta$ we can guarantee
that
\[
\kappa_{1}\leq\sum_{i=1}^{\tau}\xi_{i}\lesssim\frac{C^{\star}\tau\iota}{\left(1-\gamma\right)^{3}}\log\left(\frac{T}{\delta}\right)\asymp\frac{C^{\star}t_{\mathsf{mix}}\iota}{\left(1-\gamma\right)^{3}}\log^{2}\left(\frac{T}{\delta}\right).
\]
\end{itemize}
The above bounds taken collectively allow us to conclude that
\[
\psi\leq\kappa_{1}+\kappa_{2}+\kappa_{3}+\kappa_{4}\lesssim\frac{C^{\star}t_{\mathsf{mix}}\iota}{\left(1-\gamma\right)^{3}}\log^{2}\left(\frac{T}{\delta}\right)+\frac{C^{\star}St_{\mathsf{mix}}}{\left(1-\gamma\right)^{2}}\log\left(\frac{T}{\delta}\right).
\]

\paragraph{Step 2.5: bounding $\phi$.}

By replacing $f_{t}(s,a)$ in Step 4 with
\[
g_{t}\left(s,a\right)=P_{s,a}\left(V^{\star}-V_{t}\right),
\]
we can employ an analogous argument to show that $\phi$ admits the
same bound as $\psi$, namely,
\[
\phi\lesssim\frac{C^{\star}t_{\mathsf{mix}}\iota}{\left(1-\gamma\right)^{3}}\log^{2}\left(\frac{T}{\delta}\right)+\frac{C^{\star}St_{\mathsf{mix}}}{\left(1-\gamma\right)^{2}}\log\left(\frac{T}{\delta}\right).
\]
We omit this part for the sake of brevity. 

\subsection{Step 3: putting all pieces together}

To finish up, taking the bounds on $\alpha$, $\theta$, $\psi$ and
$\phi$ collectively gives
\begin{align*}
\alpha_{0} & \leq\alpha+\xi+\theta+\psi+\phi\\
 & \lesssim\frac{C^{\star}St_{\mathsf{mix}}\iota}{\left(1-\gamma\right)^{2}}+\sqrt{\frac{C^{\star}ST\iota^{2}}{\left(1-\gamma\right)^{5}}}+\frac{C^{\star}t_{\mathsf{mix}}\iota}{\left(1-\gamma\right)^{3}}\log^{2}\left(\frac{T}{\delta}\right)+\frac{C^{\star}St_{\mathsf{mix}}}{\left(1-\gamma\right)^{2}}\log\left(\frac{T}{\delta}\right)\\
 & \asymp\sqrt{\frac{C^{\star}ST\iota^{2}}{\left(1-\gamma\right)^{5}}}+\frac{C^{\star}St_{\mathsf{mix}}\iota}{\left(1-\gamma\right)^{2}}+\frac{C^{\star}t_{\mathsf{mix}}\iota}{\left(1-\gamma\right)^{3}}\log^{2}\left(\frac{T}{\delta}\right).
\end{align*}
Consequently, we can invoke \eqref{eq:Vstar-error-alpha0-Hoeffding} to conclude that
\begin{align*}
V^{\star}\left(\rho\right)-V^{\widehat{\pi}}\left(\rho\right) & \leq\frac{\alpha_{0}}{T}\lesssim\sqrt{\frac{C^{\star}S\iota^{2}}{T\left(1-\gamma\right)^{5}}}+\frac{C^{\star}St_{\mathsf{mix}}\iota}{T\left(1-\gamma\right)^{2}}+\frac{C^{\star}t_{\mathsf{mix}}\iota^{2}}{T\left(1-\gamma\right)^{3}}.
\end{align*}

\section{Auxiliary lemmas for Theorem \ref{theorem:5}}

\subsection{Proof of Lemma \ref{lemma:lcb-h-basics}\label{appendix:proof-lcb-h-basics}}

Consider any given pair $(s,a)\in\mathcal{S}\times\mathcal{A}$. For
notational simplicity, we write $n=n_{t}(s,a)$, the total number
of times that $(s,a)$ has been visited prior to time $t$. We also
set $k_{0}=-1$, and let
\begin{equation}
k_{i}\coloneqq\min\Big\{\big\{0\leq k<T:k>k_{i-1},\left(s_{k},a_{k}\right)=\left(s,a\right)\big\},T\Big\}\label{eq:k-defn}
\end{equation}
for each $1\leq i\leq T$. Clearly, each $k_{i}$ is a stopping time.
In view of the update rule in Algorithm \ref{alg:Q-5}, we have
\[
Q_{t}\left(s,a\right)=\sum_{i=1}^{n}\eta_{i}^{n}\Big\{ r\left(s,a\right)+\gamma V_{k_{i}}\left(s_{k_{i}+1}\right)-b_{i}\left(s,a\right)\Big\},
\]
which together with the Bellman optimality equation $Q^{\star}=r+\gamma PV^{\star}$
gives
\begin{align}
\left(Q^{\star}-Q_{t}\right)\left(s,a\right) & =r\left(s,a\right)+\gamma P_{s,a}V^{\star}-\sum_{i=1}^{n}\eta_{i}^{n}\text{\ensuremath{\Big\{}}r\left(s,a\right)+\gamma V_{k_{i}}\left(s_{k_{i}+1}\right)-b_{i}\left(s,a\right)\Big\}\nonumber \\
 & =\gamma P_{s,a}V^{\star}-\sum_{i=1}^{n}\eta_{i}^{n}\Big\{\gamma V_{k_{i}}\left(s_{k_{i}+1}\right)-b_{i}\left(s,a\right)\Big\}\nonumber \\
 & =\sum_{i=1}^{n}\eta_{i}^{n}\gamma P_{s,a}\left(V^{\star}-V_{k_{i}}\right)+\sum_{i=1}^{n}\eta_{i}^{n}\gamma\Big(\big(P-P_{k_{i}}\big)V_{k_{i}}\Big)\left(s,a\right)+\sum_{i=1}^{n}\eta_{i}^{n}b_{i}\left(s,a\right),\label{eq:Q-Qt-decompose}
\end{align}
where the last two lines are valid since $\sum_{i=1}^{n}\eta_{i}^{n}=1$
(cf.~Lemma~\ref{lemma:step-size}).

From now on we only focus on
the case where $a=\pi^{\star}(s)$. Define $\mathcal{F}_{i}$ to be
the $\sigma$-field generated by $\{(s_{i},a_{i})\}_{i=0}^{k_{i}}$.
It is straightforward to check that for any $1\leq\tau\leq T$,
\[
\left\{ \ind_{k_{i}<T}\Big(\big(P-P_{k_{i}}\big)V_{k_{i}}\Big)\big( s,\pi^{\star}(s) \big)\right\} _{i=1}^{\tau}
\]
is a martingale difference sequence with respect to $\{\mathcal{F}_{i}\}_{i\geq0}$.
Then, we can invoke the Azuma-Hoeffding inequality together with the basic bound $\|V_{k_i}\|_{\infty}\leq \frac{1}{1-\gamma}$ to show that for
any fixed $s\in\mathcal{S}$ and $\tau\in[T]$, 
\begin{align*}
\left| \sum_{i=1}^{\tau}\ind_{k_{i}<T}\eta_{i}^{\tau}\Big(\big(P-P_{k_{i}}\big)V_{k_{i}}\Big)\big(s,\pi^{\star}(s)\big) \right|
 & \lesssim\frac{1}{1-\gamma}\sqrt{\sum_{i=1}^{\tau}\left(\eta_{i}^{\tau}\right)^{2}\log\frac{ST}{\delta}}\\
 & \lesssim\sqrt{\frac{H}{\tau\left(1-\gamma\right)^{2}}\log\frac{ST}{\delta}}
\end{align*}
holds with probability exceeding $1-\delta/(ST)$. Here, the last
line utilizes Lemma \ref{lemma:step-size}. Taking the union bound over $\tau\leq T$ 
allows us to replace $\tau$ with $n=n_{t}(s,a)$ in the above inequality,
namely, for any fixed $s\in\mathcal{S}$ and $a\in\mathcal{A}$, with
probability exceeding $1-\delta/S$ we have
\begin{equation}
\left| \sum_{i=1}^{n}\eta_{i}^{n}\gamma\Big(\big(P-P_{k_{i}}\big)V_{k_{i}}\Big)\big(s,\pi^{\star}(s)\big) \right|
	\lesssim\sqrt{\frac{H\iota}{n\left(1-\gamma\right)^{2}}}\label{eq:hoeffding}
\end{equation}
holds for all $n=n_{t}(s,\pi^{\star}(s))$ with $1\leq t\leq T$.
In view of Lemma \ref{lemma:step-size}, for any $s\in\mathcal{S}$
and $a\in\mathcal{A}$ we know that
\begin{equation}
C_{\mathsf{b}}\sqrt{\frac{H\iota}{n_{t}(s,a)\left(1-\gamma\right)^{2}}}\leq\sum_{i=1}^{n_{t}(s,a)}\eta_{i}^{n_{t}(s,a)}b_{i}\left(s,a\right)\leq2C_{\mathsf{b}}\sqrt{\frac{H\iota}{n_{t}(s,a)\left(1-\gamma\right)^{2}}}.\label{eq:sum-b-bound}
\end{equation}
Therefore, when $C_{\mathsf{b}}$ is sufficiently large, it follows
that
\[
\left(Q^{\star}-Q_{t}\right)\big(s,\pi^{\star}(s)\big)\leq\gamma\sum_{i=1}^{n}\eta_{i}^{n}P_{s,\pi^{\star}(s)}\left(V^{\star}-V_{k_{i}}\right)+3C_{\mathsf{b}}\sqrt{\frac{H\iota}{n\left(1-\gamma\right)^{2}}}.
\]
Taking the union bound over $s\in\mathcal{S}$ and defining
\[
\beta_{n}\big(s,\pi^{\star}(s)\big)\coloneqq3C_{\mathsf{b}}\sqrt{\frac{H\iota}{n\left(1-\gamma\right)^{2}}},
\]
we can conclude that with probability exceeding $1-\delta$, 
\[
\left(Q^{\star}-Q_{t}\right)\big(s,\pi^{\star}(s)\big)\leq\gamma\sum_{i=1}^{n}\eta_{i}^{n}P_{s,\pi^{\star}(s)}\left(V^{\star}-V_{k_{i}}\right)+\beta_{n}\big(s,\pi^{\star}(s)\big)
\]
for all $s\in\mathcal{S}$ and $t\in[T]$. 

Additionally, observe that $V^{\star}\geq V^{\pi_{t}}$ holds trivially due
to the optimality of $V^{\star}$. We are therefore left with showing $V^{\pi_{t}}\geq V_{t}$.
Suppose for the moment that with probability exceeding $1-\delta$,
for all $s\in\mathcal{S}$, $t\in[T]$ and $j\in[t]$, it holds that
\begin{equation}
\left(Q^{\pi_{t}}-Q_{j}\right)\big(s,\pi_{t}(s)\big)\geq\gamma P_{s,\pi_{t}(s)}\left(V^{\pi_{t}}-V_{j}\right)\ind\left\{ n_{t}\big(s,\pi_{t}(s)\big)\geq1\right\} ;\label{eq:Q-pi-t-Q-t-key}
\end{equation}
the proof of this claim (\ref{eq:Q-pi-t-Q-t-key}) is deferred to Appendix \ref{appendix:proof-Q-pi-t-Q-t-key}.
As a consequence, for every $s\in\mathcal{S}$ and $t\in[T]$, there
exists $j(t)\in[t]$ such that
\begin{align*}
\left(V^{\pi_{t}}-V_{t}\right)\left(s\right) & \overset{\text{(i)}}{=}Q^{\pi_{t}}\big(s,\pi_{t}(s)\big)-Q_{j(t)}\big(s,\pi_{t}(s) \big)\overset{\text{(ii)}}{=}Q^{\pi_{t}}\big(s,\pi_{t}(s)\big)-Q_{j(t)}\big(s,\pi_{j(t)}(s)\big)\\
 & \overset{\text{(iii)}}{\geq}\min\Big\{\gamma P_{s,\pi_{t}(s)}\left(V^{\pi_{t}}-V_{j(t)}\right),0\Big\}\overset{\text{(iv)}}{\geq}\min\left\{ \gamma P_{s,\pi_{j(t)}(s)}\left(V^{\pi_{t}}-V_{t}\right),0\right\} .
\end{align*}
Here, (i) and (ii) hold since the update rule of Algorithm \ref{alg:Q-5}
asserts that there must exist some $j(t)\leq t$ such that $V_{t}(s)=V_{j(t)}(s)=Q_{j(t)}(s,\pi_{j(t)}(s))$
and $\pi_{t}(s)=\pi_{j(t)}(s)$; (iii) utilizes (\ref{eq:Q-pi-t-Q-t-key});
and (iv) follows from the monotonicity of $V_{t}$ in $t$ (by construction).
By setting
\[
s_{\min}\coloneqq\arg\min_{s\in\mathcal{S}}\ \left(V^{\pi_{t}}-V_{t}\right)\left(s\right),
\]
we can deduce that
\begin{align*}
\left(V^{\pi_{t}}-V_{t}\right)\left(s_{\min}\right) & \geq\min\left\{ \gamma P_{s_{\min},\pi_{j(t)}\left(s_{\min}\right)}\left(V^{\pi_{t}}-V_{t}\right),0\right\} \\
 & \geq\min\left\{ \gamma\min_{s\in\mathcal{S}}\left(V^{\pi_{t}}-V_{t}\right)\left(s\right),0\right\} \\
 & =\min\left\{ \gamma\left(V^{\pi_{t}}-V_{t}\right)\left(s_{\min}\right),0\right\} ,
\end{align*}
which together with the assumption $0<\gamma<1$ immediately gives
\[
\left(V^{\pi_{t}}-V_{t}\right)\left(s_{\min}\right)\geq0.
\]
Given that $\left(V^{\pi_{t}}-V_{t}\right)\left(s\right)\geq\left(V^{\pi_{t}}-V_{t}\right)\left(s_{\min}\right)$
for every $s\in\mathcal{S}$, we conclude the proof.

\subsubsection{Proof of inequality (\ref{eq:Q-pi-t-Q-t-key}) \label{appendix:proof-Q-pi-t-Q-t-key}}

First of all, if $n_{t}\big(s,\pi_{t}(s)\big)=0$, then for all $j\in[t]$,
$Q_{j}\big(s,\pi_{t}(s)\big)=0$ since it is never updated; therefore, (\ref{eq:Q-pi-t-Q-t-key})
holds true. From now on, we shall only focus on the case when $n_{t}\big(s,\pi_{t}(s)\big)\geq1$.

Consider any $s\in\mathcal{S}$, $t\in[T]$ and $j\in[t]$. For the
moment, let us define $\{k_{i}\}_{i=1}^{T}$ w.r.t.~the state-action
pair $\big(s,\pi_{t}(s)\big)$ in the same way as (\ref{eq:k-defn}). We can
then repeat the argument in (\ref{eq:Q-Qt-decompose}) to decompose
\begin{align*}
 & \big(Q^{\pi_{t}}-Q_{j}\big)\big(s,\pi_{t}(s)\big)\\
 & =\big(r+\gamma PV^{\pi_{t}}\big)\big(s,\pi_{t}(s)\big)-\sum_{i=1}^{n_{j}(s,\pi_{t}(s))}\eta_{i}^{n_{j}(s,\pi_{t}(s))}\Big\{ r\big(s,\pi_{t}(s)\big)+\gamma V_{k_{i}}\left(s_{k_{i}+1}\right)-b_{i}\big(s,\pi_{t}(s)\big)\Big\}\\
 & =\sum_{i=1}^{n_{j}(s,\pi_{t}(s))}\eta_{i}^{n_{j}(s,\pi_{t}(s))}\gamma\Big\{ P_{s,\pi_{t}(s)}V^{\pi_{t}}-V_{k_{i}}\left(s_{k_{i}+1}\right)\Big\}+\sum_{i=1}^{n_{j}(s,\pi_{t}(s))}\eta_{i}^{n_{j}(s,\pi_{t}(s))}b_{i}\big(s,\pi_{t}(s)\big)\\
 & =\sum_{i=1}^{n_{j}(s,\pi_{t}(s))}\eta_{i}^{n_{j}(s,\pi_{t}(s))}\gamma\Big\{ P_{s,\pi_{t}(s)}\left(V^{\pi_{t}}-V_{k_{i}}\right)+\Big(\big(P-P_{k_{i}}\big)V_{k_{i}}\Big)\big(s,\pi_{t}(s)\big)\Big\}+\sum_{i=1}^{n_{j}(s,\pi_{t}(s))}\eta_{i}^{n_{j}(s,\pi_{t}(s))}b_{i}\big(s,\pi_{t}(s)\big)\\
 & \geq\left(\sum_{i=1}^{n_{j}(s,\pi_{t}(s))}\eta_{i}^{n_{j}(s,\pi_{t}(s))}\right)\gamma\min_{1\leq i\leq n}P_{s,\pi_{t}(s)}\left(V^{\pi_{t}}-V_{k_{i}}\right)+\sum_{i=1}^{n_{j}(s,\pi_{t}(s))}\eta_{i}^{n_{j}(s,\pi_{t}(s))}\gamma\Big(\big(P-P_{k_{i}}\big)V_{k_{i}}\Big)\big(s,\pi_{t}(s)\big)\\
 & \qquad+\sum_{i=1}^{n_{j}(s,\pi_{t}(s))}\eta_{i}^{n_{j}(s,\pi_{t}(s))}b_{i}\big(s,\pi_{t}(s)\big)\\
 & \geq\gamma P_{s,\pi_{t}(s)}\left(V^{\pi_{t}}-V_{t}\right)+\sum_{i=1}^{n_{j}(s,\pi_{t}(s))}\eta_{i}^{n_{j}(s,\pi_{t}(s))}\gamma\Big(\big(P-P_{k_{i}}\big)V_{k_{i}}\Big)\big(s,\pi_{t}(s)\big)+C_{\mathsf{b}}\sqrt{\frac{H\iota}{n_{j}\big(s,\pi_{t}(s)\big)\left(1-\gamma\right)^{2}}}.
\end{align*}
Here, the last inequality follows from (\ref{eq:sum-b-bound}), as
well as the facts that $\sum_{i=1}^{n_{j}(s,\pi_{t}(s))}\eta_{i}^{n_{j}(s,\pi_{t}(s))}=1$
(cf.~Lemma~\ref{lemma:step-size}) and that $V_{t}$ is non-decreasing
in $t$. It thus boils down to showing that for every $s\in\mathcal{S}$,
$t\in[T]$ and $j\in[t]$, 
\begin{align}
\sum_{i=1}^{n_{j}(s,\pi_{t}(s))}\eta_{i}^{n_{j}(s,\pi_{t}(s))}\gamma\Big(\big(P-P_{k_{i}}\big)V_{k_{i}}\Big)\big(s,\pi_{t}(s)\big)\lesssim\sqrt{\frac{H\iota}{n_{j}\big(s,\pi_{t}(s)\big)\left(1-\gamma\right)^{2}}}.
	\label{eq:sum-eta-gamm-P-V-5372}
\end{align}
If this were true and if $C_{\mathsf{b}}$ is sufficiently large, then we could
combine the above two inequalities to conclude the proof of (\ref{eq:Q-pi-t-Q-t-key}).

We then prove the inequality \eqref{eq:sum-eta-gamm-P-V-5372}. 
Notice that for all $(s,\pi_{t}(s))$ such that $n_{t}(s,\pi_{t}(s))\geq1$,
it must appear at least once in the sample trajectory. Therefore it
suffices to show that for all $0\leq l<T$ and $t\in[T]$, it holds
that
\[
\sum_{i=1}^{n_{t}(s_{l},a_{l})}\eta_{i}^{n_{t}(s_{l},a_{l})}\gamma\Big(\big(P-P_{k_{i}}\big)V_{k_{i}}\Big)\left(s_{l},a_{l}\right)\lesssim\sqrt{\frac{H\iota}{n_{t}(s_{l},a_{l})\left(1-\gamma\right)^{2}}},
\]
where we abuse the notation by defining $\{k_{i}\}_{i=1}^{T}$ for
the state-action pair $(s_{l},a_{l})$ in the same way as (\ref{eq:k-defn}).
Furthermore, it suffices to only check those $(s_{l},a_{l})$ in the
sample trajectory that were visited for the first time, i.e., $n_{l}(s_{l},a_{l})=0$
and $n_{l+1}(s_{l},a_{l})=1$. It is straightforward to check that,
for any $1\leq\tau\leq T$,
\[
\left\{ \ind_{k_{i}<T}\Big(\big(P-P_{k_{i}}\big)V_{k_{i}}\Big)\left(s_{l},a_{l}\right)\right\} _{i=1}^{\tau}
\]
is a martingale difference sequence with respect to $\{\mathcal{F}_{i}\}_{i\geq0}$,
where $\mathcal{F}_{i}$ is the $\sigma$-field generated by $\{(s_{i},a_{i})\}_{i=0}^{k_{i}}$.
Then we can invoke the Azuma-Hoeffding inequality to show that: for any
such $(s_{l},a_{l})$ and any $\tau\in[T]$, with probability exceeding
$1-\delta/T^{2}$,
\begin{align*}
\left|\sum_{i=1}^{\tau}\ind_{k_{i}<T}\eta_{i}^{\tau}\Big(\big(P-P_{k_{i}}\big)V_{k_{i}}\Big)\left(s_{l},a_{l}\right)\right| & \lesssim\frac{1}{1-\gamma}\sqrt{\sum_{i=1}^{\tau}\left(\eta_{i}^{\tau}\right)^{2}\log\frac{T}{\delta}}\lesssim\sqrt{\frac{H\iota}{\tau\left(1-\gamma\right)^{2}}}.
\end{align*}
Taking the union bound over $\tau\in [T]$ allows us to replace $\tau$ with $n_{t}(s_{l},a_{l})$
in the above inequality, namely, this shows that for any such $(s_{l},a_{l})$, with
probability exceeding $1-\delta/T$ we have
\begin{align*}
\left|\sum_{i=1}^{n_{t}(s_{l},a_{l})}\eta_{i}^{n_{t}(s_{l},a_{l})}\Big(\big(P-P_{k_{i}}\big)V_{k_{i}}\Big)\left(s_{l},a_{l}\right)\right| & \lesssim\sqrt{\frac{H\iota}{n_{t}(s_{l},a_{l})\left(1-\gamma\right)^{2}}}
\end{align*}
for all $t\in[T]$. Taking the union bound over all such $(s_{l},a_{l})$
(which are concerned with at most $T$ pairs), we see that with probability exceeding
$1-\delta$, 
\begin{align*}
\left|\sum_{i=1}^{n_{t}(s_{l},a_{l})}\eta_{i}^{n_{t}(s_{l},a_{l})}\Big(\big(P-P_{k_{i}}\big)V_{k_{i}}\Big)\left(s_{l},a_{l}\right)\right| & \lesssim\sqrt{\frac{H\iota}{n_{t}(s_{l},a_{l})\left(1-\gamma\right)^{2}}}
\end{align*}
is valid for any $0\leq j<T$ and any $t\in[T]$. This establishes the inequality \eqref{eq:sum-eta-gamm-P-V-5372}, thus concluding the proof.

\subsection{Proof of Lemma \ref{lemma:states-small-prob}\label{appendix:proof-lemma-states-small-prob}}

For each $\big(s,\pi^{\star}(s)\big)\in\mathcal{I}^{c}$, we first
have
\[
\mathbb{P}\Big\{\left(s_{t},a_{t}\right)=\big(s,\pi^{\star}(s)\big)\mid\left(s_{0},a_{0}\right)\sim\mu_{\mathsf{b}}\Big\}=\mu_{\mathsf{b}}\big(s,\pi^{\star}(s)\big)<\frac{\delta}{ST},
\]
given that $\mu_{\mathsf{b}}$ is taken to be the stationary distribution of the sample trajectory. 
By virtue of the union bound, we obtain
\begin{align*}
 & \mathbb{P}\left(\mathcal{I}^{c}\cap\left\{ \left(s_{t},a_{t}\right)\right\} _{t=t_{\mathsf{mix}}(\delta)}^{T}=\varnothing\mid\left(s_{0},a_{0}\right)\sim\mu_{\mathsf{b}}\right)\\
 & \quad\geq1-\sum_{t=t_{\mathsf{mix}}}^{T}\sum_{s: (s,\pi^{\star}(s))\in\mathcal{I}^{c}}\mathbb{P}\Big\{\left(s_{t},a_{t}\right)=\left(s,\pi^{\star}(s)\right)\mid\left(s_{0},a_{0}\right)\sim\mu_{\mathsf{b}}\Big\}\\
 & \quad>1-\delta.
\end{align*}
In addition, for an arbitrary pair $(s,a)\in\mathcal{S}\times\mathcal{A}$, the
definition of the mixing time gives
\[
\left|\mathbb{P}\left(\big\{ \left(s_{t},a_{t}\right)\big\} _{t=t_{\mathsf{mix}}(\delta)}^{T}\subseteq\mathcal{I}\mid\left(s_{0},a_{0}\right)\sim\mu_{\mathsf{b}}\right)-\mathbb{P}\left(\left\{ \left(s_{t},a_{t}\right)\right\} _{t=t_{\mathsf{mix}}}^{T}\subseteq\mathcal{I}\mid\left(s_{0},a_{0}\right)=\left(s,a\right)\right)\right|\leq\delta.
\]
Combine the above results to yield
\[
\mathbb{P}\left(\big\{ \left(s_{t},a_{t}\right)\big\} _{t=t_{\mathsf{mix}}(\delta)}^{T}\subseteq\mathcal{I}\mid\left(s_{0},a_{0}\right)=\left(s,a\right)\right)\geq1-2\delta
\]
for an arbitrary pair $(s,a)\in\mathcal{S}\times\mathcal{A}$.

\subsection{Proof of Lemma \ref{lemma:useful}\label{appendix:proof-eq-useful}}

For any given integer $K>0$, one can decompose
\begin{align*}
\sum_{j=0}^{\infty}\left[\gamma\left(1+\frac{1}{H}\right)^{3}\right]^{j}\left\langle \rho(P_{\pi^{\star}})^{j},V\right\rangle  & =\sum_{j=0}^{K-1}\left[\gamma\left(1+\frac{1}{H}\right)^{3}\right]^{j}\left\langle \rho(P_{\pi^{\star}})^{j},V\right\rangle +\sum_{j=K}^{\infty}\left[\gamma\left(1+\frac{1}{H}\right)^{3}\right]^{j}\left\langle \rho(P_{\pi^{\star}})^{j},V\right\rangle \\
 & \leq\left(1+\frac{1}{H}\right)^{3K}\sum_{j=0}^{K-1}\gamma^{j}\left\langle \rho(P_{\pi^{\star}})^{j},V\right\rangle +\sum_{j=K}^{\infty}\left[\gamma\left(1+\frac{1}{H}\right)^{3}\right]^{j}\left\Vert V\right\Vert _{\infty}\\
 & \leq\underbrace{\left(1+\frac{1}{H}\right)^{3K}\frac{1}{1-\gamma}\left\langle d_{\rho}^{\star},V\right\rangle }_{\eqqcolon\,\alpha_{1}}+\underbrace{\gamma^{K}\left(1+\frac{1}{H}\right)^{3K}\frac{1}{1-\gamma\left(1+\frac{1}{H}\right)^{3}}\left\Vert V\right\Vert _{\infty}}_{\eqqcolon \,\alpha_{2}}.
\end{align*}
Here, the last inequality holds since $d_{\rho}^{\star}=(1-\gamma)\sum_{j=0}^{\infty}\gamma^{j}\rho \big( P_{\pi^{\star}} \big)^{j}$.

By taking 
\[
	K = H=\left\lceil \frac{4}{1-\gamma}\log\frac{ST}{\delta}\right\rceil ,
\]
we can derive
\[
\left(1+\frac{1}{H}\right)^{3K}=\left(1+\frac{1}{H}\right)^{3H}\overset{\text{(i)}}{\leq}e^{3}=O\left(1\right)
\]
and
\[
\gamma^{K}=e^{K\log\left[1-\left(1-\gamma\right)\right]}\overset{\text{(ii)}}{\leq}e^{-K\left(1-\gamma\right)}=\frac{\delta}{ST^{4}}. 
\]
Here, (i) holds since $(1+1/x)^{x}\leq e$ for all $x>0$; (ii) is
valid since $\log(1-x)\leq-x$ for all $x\in(0,1)$. 
It is also worth noting that
\begin{equation}
\frac{1}{1-\gamma\left(1+\frac{1}{H}\right)^{3}}\leq\frac{1}{1-\gamma\left(1+\frac{1-\gamma}{4}\right)^{3}}\overset{\text{(iii)}}{\leq}\frac{1}{1-\gamma\left[1+\frac{61}{64}\left(1-\gamma\right)\right]}=\frac{1}{\left(1-\gamma\right)\left(1-\frac{61}{64}\gamma\right)}\lesssim\frac{1}{1-\gamma}, \label{eq:ratio-bound}
\end{equation}
where (iii) holds
since $(1+x)^{3}\leq1+61x/16$ for all $0<x\leq1/4$. We then immediately
arrive at
\[
\alpha_{1}\lesssim\frac{1}{1-\gamma}\left\langle d_{\rho}^{\star},V\right\rangle 
\]
and
\[
\alpha_{2}\lesssim\frac{\delta}{ST^{4}\left(1-\gamma\right)}\left\Vert V\right\Vert _{\infty}.
\]
Taking the upper bounds on $\alpha_{1}$ and $\alpha_{2}$ collectively
leads to the advertised inequality
\[
\sum_{j=0}^{\infty}\left[\gamma\left(1+\frac{1}{H}\right)^{3}\right]^{j}\left\langle \rho(P_{\pi^{\star}})^{j},V\right\rangle \lesssim\frac{1}{1-\gamma}\left\langle d_{\rho}^{\star},V\right\rangle +\frac{\delta}{ST^{4}\left(1-\gamma\right)}\left\Vert V\right\Vert _{\infty}.
\]

\subsection{Proof of Lemma \ref{lemma:coupling}\label{appendix:proof-lemma-coupling}}

For notational simplicity, we denote 
\[
X_{t}\coloneqq\left(s_{t},a_{t}\right),\qquad1\leq t\leq T; 
\]
clearly, $\{X_{t}\}_{t\geq 0}$ forms a Markov chain on $\mathcal{X}\triangleq\mathcal{S}\times\mathcal{A}$,
with stationary distribution $\mu_{\mathsf{b}}$. In what follows,
we demonstrate how to construct the sequence $$Y_{K-1}^{i}=(s_{K-1}^{i},a_{K-2}^{i}),~~~ Y_{K-2}^{i}=(s_{K-2}^{i},a_{K-2}^{i}),~~~ \cdots,~~~ Y_{1}^{i}=(s_{1}^{i},a_{1}^{i})$$ 
so as to satisfy the desired properties.

Let us start by constructing $Y_{K-1}^{i}$. Recall from the definition of the mixing time that: for any fixed state-action pairs $x_{0},x_{1},\cdots,x_{(K-2)\tau+i}\in\mathcal{X}$, 
one has
\[
\mathsf{TV}\left(\mathcal{L}\left(X_{\left(K-1\right)\tau+i}\mid X_{0}=x_{0},\ldots,X_{(K-2)\tau+i}=x_{(K-2)\tau+i}\right),\mu_{\mathsf{b}}\right)\leq\frac{\delta}{T^{2}}.
\]
where $\mathcal{L}(\cdot)$ denotes the law of the random variable.  
In view of the definition of the total-variation distance, we know
that there exists a random variable $Y_{K-1}^{x_{0},\ldots,x_{(K-2)\tau+i}}$
such that conditional on the event $X_{0}=x_{0},\ldots,X_{(K-2)\tau+i}=x_{(K-2)\tau+i}$, 
\begin{itemize}
	\item[(i)] the law of $Y_{K-1}^{x_{0},\ldots,x_{(K-2)\tau+i}}$ obeys
\[
\mathcal{L}\left(Y_{K-1}^{x_{0},\ldots,x_{(K-2)\tau+i}}\mid X_{0}=x_{0},\ldots,X_{(K-2)\tau+i}=x_{(K-2)\tau+i}\right)=\mu_{\mathsf{b}}
\]
	\item[(ii)] $Y_{K-1}^{x_{0},\ldots,x_{(K-2)\tau+i}}$ is almost identical to $X_{\left(K-1\right)\tau+i}$ in the sense that
\[
	\mathbb{P}\Big\{ X_{\left(K-1\right)\tau+i}\neq Y_{K-1}^{x_{0},\,\ldots,\,x_{(K-2)\tau+i}}\mid X_{0}=x_{0},\ldots,X_{(K-2)\tau+i}=x_{(K-2)\tau+i} \Big\}
		\leq\frac{\delta}{T^{2}}.
\]
\end{itemize}
As a consequence, we can construct $Y_{K-1}^{i}$ as follows
\[
Y_{K-1}^{i}\coloneqq\sum_{x_{0},\ldots,x_{(K-2)\tau+i}\in\mathcal{X}} Y_{K-1}^{x_{0},\ldots,x_{(K-2)\tau+i}} \ind\{X_{0}=x_{0},\ldots,X_{(K-2)\tau+i}=x_{(K-2)t_{\mathsf{mix}}+i}\};
\]
as can be easily verified,  for any $x_{0},x_{1},\cdots,x_{(K-2)t_{\mathsf{mix}}+i}\in\mathcal{X}$ one has
\begin{align*}
 & \mathcal{L}\left(Y_{K-1}^{i}\mid X_{0}=x_{0},\ldots,X_{(K-2)\tau+i}=x_{(K-2)\tau+i}\right)\\
 & \qquad=\mathcal{L}\left(Y_{K-1}^{x_{0},\ldots,x_{(K-2)\tau+i}}\mid X_{0}=x_{0},\ldots,X_{(K-2)\tau+i}=x_{(K-2)\tau+i}\right)=\mu_{\mathsf{b}} .
\end{align*}
All this in turn implies that  
\[
	Y_{K-1}^{i}\sim\mu_{\mathsf{b}}
	\qquad \text{and} \qquad
	Y_{K-1}^{i}\indep\left\{ X_{0},X_{1}\ldots,X_{(K-2)\tau+i}\right\} .
\]
In addition, it is also seen that
\begin{align*}
\mathbb{P}\left(Y_{K-1}^{i}\neq X_{\left(K-1\right)\tau+i}\right) & =\sum_{x_{0},\ldots,x_{(K-2)\tau+i}\in\mathcal{X}}\mathbb{P}\left(X_{0}=x_{0},\ldots,X_{(K-2)\tau+i}=x_{(K-2)\tau+i}\right)\\
	& \qquad\cdot \mathbb{P}\Big\{ X_{\left(K-1\right)\tau+i}\neq Y_{K-1}^{x_{0},\ldots,x_{(K-2)\tau+i}}\mid X_{0}=x_{0},\ldots,X_{(K-2)\tau+i}=x_{(K-2)\tau+i}\Big\} \\
 & \leq\frac{\delta}{T^{2}}\sum_{x_{0},\ldots,x_{(K-2)\tau+i}\in\mathcal{X}}\mathbb{P}\left(X_{0}=x_{0},\ldots,X_{(K-2)\tau+i}=x_{(K-2)\tau+i}\right)\\
 & =\frac{\delta}{T^{2}}.
\end{align*}

Next, we turn to the construction of $Y_{K-2}^{i}$. Consider any fixed $x_{0},x_{1},\cdots,x_{(K-3)\tau+i},y_{K-1}^{i}\in\mathcal{X}$. 
Given that $Y_{K-1}^{i}\indep \{X_{0},X_{1}\ldots,X_{(K-2)\tau+i}\}$,
the conditional law of $X_{\left(K-2\right)\tau+i}$ obeys
\begin{align*}
 & \mathcal{L}\left(X_{\left(K-2\right)\tau+i}\mid X_{0}=x_{0},\ldots,X_{(K-3)\tau+i}=x_{(K-3)\tau+i},Y_{K-1}^{i}=y_{k-1}^{i}\right)\\
 & \qquad=\mathcal{L}\left(X_{\left(K-2\right)\tau+i}\mid X_{0}=x_{0},\ldots,X_{(K-3)\tau+i}=x_{(K-3)\tau+i}\right).
\end{align*}
This in turn allows one to obtain
\begin{align*}
 & \mathsf{TV}\Big(\mathcal{L}\left(X_{\left(K-2\right)\tau+i}\mid X_{0}=x_{0},\ldots,X_{(K-3)\tau+i}=x_{(K-3)\tau+i},Y_{K-1}^{i}=y_{K-1}^{i}\right),\mu_{\mathsf{b}}\Big)\\
 & \quad=\mathsf{TV}\left(\mathcal{L}\left(X_{\left(K-2\right)\tau+i}\mid X_{0}=x_{0},\ldots,X_{(K-3)\tau+i}=x_{(K-3)\tau+i}\right),\mu_{\mathsf{b}}\right)\\
 & \quad\leq\frac{\delta}{T^{2}}.
\end{align*}
According to the definition of the total-variation distance, there exists a random variable $Y_{K-2}^{x_{0},x_{1},\cdots,x_{(K-3)\tau+i},y_{K-1}^{i}}$
such that: conditional on the event $X_{0}=x_{0},\ldots,X_{(K-3)\tau+i}=x_{(K-3)\tau+i},Y_{K-1}^{i}=y_{k-1}^{i}$,
\begin{itemize}
	\item[(i)] the law of $Y_{K-2}^{x_{0},x_{1},\cdots,x_{(K-3)\tau+i},y_{K-1}^{i}}$ obeys
\[
\mathcal{L}\left(Y_{K-2}^{x_{0},x_{1},\cdots,x_{(K-3)\tau+i},y_{K-1}^{i}}\mid X_{0}=x_{0},\ldots,X_{(K-3)\tau+i}=x_{(K-3)\tau+i},Y_{K-1}^{i}=y_{K-1}^{i}\right)=\mu_{\mathsf{b}};
\]
	\item[(ii)] $Y_{K-2}^{x_{0},x_{1},\cdots,x_{(K-3)\tau+i},y_{K-1}^{i}}$ is almost identical to  $X_{\left(K-2\right)\tau+i}$ in the following sense
\[
\mathbb{P}\left(X_{\left(K-2\right)\tau+i}\neq Y_{K-2}^{x_{0},x_{1},\cdots,x_{(K-3)\tau+i},y_{K-1}^{i}}\mid X_{0}=x_{0},\ldots,X_{(K-3)\tau+i}=x_{(K-3)\tau+i},Y_{K-1}^{i}=y_{K-1}^{i}\right)\leq\frac{\delta}{T^{2}}.
\]
\end{itemize}
With the above set of random variables in mind, we can readily construct $Y_{K-2}^{i}$ as follows:
\[
Y_{K-2}^{i}\coloneqq\sum_{x_{0},x_{1},\cdots,x_{(K-3)\tau+i},y_{K-1}^{i}\in\mathcal{X}} Y_{K-2}^{x_{0},x_{1},\cdots,x_{(K-3)\tau+i},y_{K-1}^{i}} \ind\big\{X_{0}=x_{0},\ldots,X_{(K-3)\tau+i}=x_{(K-3)\tau+i},Y_{K-1}^{i}=y_{k-1}^{i} \big\}.
\]
As can be straightforwardly verified,  for any $x_{0},x_{1},\cdots,x_{(K-3)\tau+i},y_{K-1}^{i}\in\mathcal{X}$ we have 
\begin{align*}
 & \mathcal{L}\left(Y_{K-2}^{i}\mid X_{0}=x_{0},\ldots,X_{(K-3)\tau+i}=x_{(K-3)\tau+i},Y_{K-1}^{i}=y_{k-1}^{i}\right)\\
 & \quad=\mathcal{L}\left(Y_{K-2}^{x_{0},x_{1},\cdots,x_{(K-3)\tau+i},y_{K-1}^{i}}\mid X_{0}=x_{0},\ldots,X_{(K-3)\tau+i}=x_{(K-3)\tau+i},Y_{K-1}^{i}=y_{k-1}^{i}\right)=\mu_{\mathsf{b}},
\end{align*}
thus implying that $Y_{K-2}^{i}\sim\mu_{\mathsf{b}}$ and
\[
Y_{K-2}^{i}\indep\left\{ X_{0},X_{1}\ldots,X_{(K-3)\tau+i},Y_{K-1}^{i}\right\} .
\]
This reveals that $Y_{K-1}^{i},Y_{K-2}^{i}\overset{\mathsf{i.i.d.}}{\sim}\mu_{\mathsf{b}}$.
In addition, we can also show that 
\begin{align*}
 & \mathbb{P}\left(Y_{K-2}^{i}\neq X_{\left(K-2\right)\tau+i}\right)\\
 & \quad=\sum_{x_{0},x_{1},\cdots,x_{(K-3)\tau+i},y_{K-1}^{i}\in\mathcal{X}}\mathbb{P}\left(X_{0}=x_{0},\ldots,X_{(K-3)\tau+i}=x_{(K-3)\tau+i},Y_{K-1}^{i}=y_{K-1}^{i}\right)\\
 & \qquad\qquad\cdot \mathbb{P}\left(X_{\left(K-2\right)\tau+i}\neq Y_{K-2}^{x_{0},x_{1},\cdots,x_{(K-3)\tau+i},y_{K-1}^{i}}\mid X_{0}=x_{0},\ldots,X_{(K-3)\tau+i}=x_{(K-3)\tau+i},Y_{K-1}^{i}=y_{K-1}^{i}\right)\\
 & \quad\leq\frac{\delta}{T^{2}}\sum_{x_{0},x_{1},\cdots,x_{(K-3)\tau+i},y_{K-1}^{i}\in\mathcal{X}}\mathbb{P}\left(X_{0}=x_{0},\ldots,X_{(K-3)\tau+i}=x_{(K-3)\tau+i},Y_{K-1}^{i}=y_{K-1}^{i}\right)\\
 & \quad=\frac{\delta}{T^{2}}.
\end{align*}

Moving forward, we can employ similar arguments to construct $Y_{K-3}^{i},\ldots,Y_{1}^{i}$ sequentially 
such that: 
\begin{itemize}
	\item[(i)] $Y_{1}^{i},Y_{2}^{i},\ldots,Y_{K-1}^{i}\overset{\mathsf{i.i.d.}}{\sim}\mu_{\mathsf{b}}$; 
	\item[(ii)] for all $1\leq k\leq K-1$,
\[
Y_{k}^{i}\indep\left\{ X_{0},X_{1},\ldots,X_{(k-1)\tau+i}\right\} 
		\qquad \text{and} \qquad
\mathbb{P}\left(Y_{k}^{i}\neq X_{k\tau+i}\right)\leq\frac{\delta}{T^{2}}.
\]
\end{itemize}
As a result, we arrive at
\[
\mathbb{P}\left(Y_{1}^{i}=X_{\tau+i},\cdots,Y_{K-1}^{i}=X_{(K-1)\tau+i}\right)\geq1-\sum_{k=1}^{K-1}\mathbb{P}\left(Y_{k}^{i}=X_{k\tau+i}\right)\geq1-\frac{\delta}{T}.
\]
%We can thus finish the proof by letting $Y_{k}^{i}=(s_{k}^{i},a_{k}^{i})$
%for all $1\leq k\leq K-1$.
%
This concludes the proof. 

%% file: variance_reduced.tex
\section{Analysis for variance-reduced Q-learning with LCB penalization (Theorem~\ref{theorem:3})}

This section presents the proof of Theorem \ref{theorem:3}, which is concerned with the performance of variance-reduced Q-learning with LCB penalization. Recall that $\overline{V}_{k+1}=V_{T_{k}}$ , that is, the value 
estimate in the last iterate of the $k$-th epoch is also used
as the reference for the $(k+1)$-th epoch. For each
$1\leq k\leq K$, we define 
\begin{equation}
\Lambda_{k} \coloneqq\sum_{s\in\mathcal{S}}\rho\left(s\right)\left(V^{\star}-\overline{V}_{k}\right)\left(s\right)
	\label{eq:defn-Lambda-k-VR}
\end{equation}
Clearly, the proof of Theorem \ref{theorem:3} boils down to bounding $\Lambda_{K}$. As we shall see momentarily, obtaining a tight bound on $\Lambda_k$ relies on bounding another closely related quantity $\Delta_{K-1}$, define for each
$1\leq k\leq K$ as follows:
\begin{equation}
\Delta_{k}  \coloneqq\sum_{s\in\mathcal{S}}\widetilde{\rho}\left(s\right)\left(V^{\star}-\overline{V}_{k}\right)\left(s\right).
	\label{eq:defn-Delta-k-VR}
\end{equation}
Here, we set
\begin{equation}
\widetilde{\rho}\coloneqq\frac{d_{\rho}^{\star}-\left(1-\gamma\right)\rho}{\gamma}.
	\label{eq:defn-tilde-rho-VR}
\end{equation}

The sequence $\{\Delta_k\}_{k=1}^K$ will be bounded by induction in the sequel. We shall present our proof by describing three key steps following some preliminary facts.

\subsection{Preliminary facts about the $k$-th epoch }

Let us first look at what happens in the $k$-th epoch. For notational
simplicity, we will denote $\overline{V}\coloneqq\overline{V}_{k-1}$.
Similar to the proof of Theorem \ref{theorem:5}, for any iterate
$t\leq T_{k}$, let $n=n_{t}(s,a)$ and assume that $(s,a)$ has been visited
during the iterations $k_{1}<\cdots<k_{n}<t$. We also need to define
the policy $\pi_{t}:\mathcal{S}\to\mathcal{A}$ as follows
\[
\pi_{t}\left(s\right)\coloneqq\begin{cases}
\arg\max_{a\in\mathcal{A}}Q_{t}\left(s_{t-1},a\right), & \text{if }s=s_{t-1}\text{ and }V_{t}\left(s\right)>V_{t-1}\left(s\right),\\
\pi_{t-1}\left(s\right), & \text{otherwise}.
\end{cases}
\]
If there are multiple $a\in\mathcal{A}$ that maximize $Q_{t}\left(s_{t-1},a\right)$ simultaneously,
then we can go with any of them. We make note of the following lemma.
\begin{lemma}
\label{lemma:vr-basics}
With probability exceeding $1-\delta$,
for any $s\in\mathcal{S}$ and $t\in[T]$ we have 
\[
\left(Q^{\star}-Q_{t}\right)\big(s,\pi^{\star}(s)\big)\leq\gamma\sum_{i=1}^{n}\eta_{i}^{n}P_{s,a}\left(V^{\star}-V_{k_{i}}\right)+\beta_{n}\big(s,\pi^{\star}(s)\big) ,
\]
where $n=n_{t}\big(s,\pi^{\star}(s)\big)$ and 
\begin{align}
\beta_{n}\left(s,a\right) & \coloneqq3C_{\mathsf{b}}\sqrt{\frac{H\iota}{n}\left\{ \sigma_{n}^{\mathsf{adv}}\left(s,a\right)-\left[\mu_{n}^{\mathsf{adv}}\left(s,a\right)\right]^{2}\right\} }+3C_{\mathsf{b}}\frac{H^{3/4}\iota^{3/4}}{n^{3/4}\left(1-\gamma\right)}+3C_{\mathsf{b}}\frac{H\iota}{n\left(1-\gamma\right)} \notag\\
 & \quad+3C_{\mathsf{b}}\sqrt{\frac{\iota}{n^{\mathsf{ref}}\left(s,a\right)}\left\{ \sigma^{\mathsf{ref}}\left(s,a\right)-\left[\mu^{\mathsf{ref}}\left(s,a\right)\right]^{2}\right\} }+3C_{\mathsf{b}}\frac{\iota^{3/4}}{\left(1-\gamma\right)\left[n^{\mathsf{ref}}\left(s,a\right)\right]^{3/4}} \notag\\
 & \quad+3C_{\mathsf{b}}\frac{\iota}{\left(1-\gamma\right)n^{\mathsf{ref}}\left(s,a\right)}.
	\label{eq:defn-beta-n-VR}
\end{align}
In addition, it holds that 
\[
	V_{t}(s)\leq V^{\pi_{t}}(s)\leq V^{\star}(s)\qquad\text{for all } s\in\mathcal{S} ~\text{ and }~ 1\leq t\leq T_{k}.
\]
%
%for every $$.
\end{lemma}
\begin{proof}See Appendix \ref{appendix:proof-vr-basics}.\end{proof}

Moreover, both $\sigma_{n}^{\mathsf{adv}}(s,a)$
and $\sigma^{\mathsf{ref}}(s,a)-[\mu^{\mathsf{ref}}(s,a)]^{2}$ play an important role in determining the variance of the update,
and we are in need of the following bounds on these two quantities.

\begin{lemma} \label{lemma:variance-bound}With probability exceeding
$1-\delta$, for all $s\in\mathcal{S}$ and $t\in[T_{k}]$ we have
\begin{align*}
\sigma_{n_{t}(s,\pi^{\star}(s))}^{\mathsf{adv}}\big(s,\pi^{\star}(s)\big) & \leq P_{s,\pi^{\star}(s)}\left(V^{\star}-\overline{V}\right)^{2}+O\left(\frac{1}{\left(1-\gamma\right)^{2}}\sqrt{\frac{H\iota}{n_{t}\big(s,\pi^{\star}(s)\big)}}\right)
\end{align*}
and 
\[
	\sigma^{\mathsf{ref}}\big(s,\pi^{\star}\left(s\right)\big)-\left[\mu^{\mathsf{ref}}\big(s,\pi^{\star}\left(s\right)\big)\right]^{2}=\mathsf{Var}_{s,\pi^{\star}(s)}(\overline{V})+O\left(\frac{1}{\left(1-\gamma\right)^{2}}\sqrt{\frac{\iota}{n^{\mathsf{ref}}\big(s,\pi^{\star}(s)\big)}}\right).
\]
In addition, it holds that
\begin{subequations}
\begin{align}
	\sum_{s\in\mathcal{S},a\in\mathcal{A}}d_{\rho}^{\star}\left(s,a\right)\mathsf{Var}_{s,a}(V^{\star}-\overline{V}) &\leq\frac{1}{1-\gamma}\Delta_{k-1}; \\
	\sum_{s\in\mathcal{S},a\in\mathcal{A}}d_{\rho}^{\star}\left(s,a\right)\mathsf{Var}_{s,a}(\overline{V}) &\leq\frac{8}{1-\gamma}+\frac{2}{1-\gamma}\Delta_{k-1}.
\end{align}
\end{subequations}
\end{lemma}

\begin{proof}See Appendix \ref{appendix:proof-variance-bound}.\end{proof}

\subsection{Step 1: connecting $\Lambda_{k}$ with $\Delta_{k-1}$}

In this step, we aim to establish a connection between $\Lambda_{k}$ (cf.~\eqref{eq:defn-Lambda-k-VR}) and $\Delta_{k-1}$ (cf.~\eqref{eq:defn-Delta-k-VR}). 
In view of the monotonicity of $V_t$ in $t$ (by construction) and Lemma~\ref{lemma:vr-basics}, we can derive 
\begin{equation}
	\Lambda_{k}=\big\langle \rho,V^{\star}-V_{T_{k}}\big\rangle \leq\frac{1}{T_{k}}\sum_{t=1}^{T_{k}}\left\langle \rho,V^{\star}-V_{t}\right\rangle .
	\label{eq:Lambda-k-final-UB-135-VR}
\end{equation}
Before continuing, we find it convenient to introduce a set of quantities (similar to our proof for Theorem~\ref{theorem:5}):
\begin{align*}
\alpha_{j} & \coloneqq\left[\gamma\left(1+\frac{1}{H}\right)^{3}\right]^{j}\sum_{t=1}^{T_{k}}\left\langle \rho(P_{\pi^{\star}})^{j},V^{\star}-V_{t}\right\rangle ,\\
	\theta_{j} & \coloneqq\left[\gamma\left(1+\frac{1}{H}\right)^{3}\right]^{j}\sum_{t=1}^{T_{k}}\sum_{s\in\mathcal{S}}\left[\rho(P_{\pi^{\star}})^{j}\right]\big(s,\pi^{\star}(s)\big)\min\left\{ \beta_{n_{t}\left(s,\pi^{\star}(s)\right)}\big(s,\pi^{\star}(s)\big),\frac{1}{1-\gamma}\right\} ,\\
\xi_{j} & \coloneqq\left[\gamma\left(1+\frac{1}{H}\right)^{3}\right]^{j}\sum_{t=1}^{t_{\mathsf{mix}}(\delta)}\left\langle \rho(P_{\pi^{\star}})^{j},V^{\star}-V_{t}\right\rangle +\left[\gamma\left(1+\frac{1}{H}\right)^{3}\right]^{j+1}\left\langle \rho(P_{\pi^{\star}})^{j+1},V^{\star}-V_{0}\right\rangle , \\
\psi_{j} & \coloneqq\left[\gamma\left(1+\frac{1}{H}\right)^{3}\right]^{j}\sum_{t=t_{\mathsf{mix}}(\delta)}^{T}\Biggl[\sum_{s\in\mathcal{S},a\in\mathcal{A}}\left[\rho^{\pi^{\star}}(P^{\pi^{\star}})^{j}\right]\left(s,a\right)\sum_{i=1}^{n_{t}\left(s,a\right)}\eta_{i}^{n_{t}\left(s,a\right)}P_{s,a}\left(V^{\star}-V_{k_{i}\left(s,a\right)}\right)\\
 & \qquad\qquad\qquad\qquad\qquad\quad-\left(1+\frac{1}{H}\right)\frac{\left[\rho^{\pi^{\star}}(P^{\pi^{\star}})^{j}\right]\left(s_{t},a_{t}\right)}{\mu_{\mathsf{b}}\left(s_{t},a_{t}\right)}\sum_{i=1}^{n_{t}\left(s_{t},a_{t}\right)}\eta_{i}^{n_{t}\left(s_{t},a_{t}\right)}P_{s_{t},a_{t}}\left(V^{\star}-V_{k_{i}\left(s_{t},a_{t}\right)}\right)\Biggr], \\
\phi_{j} & \coloneqq\gamma^{j+1}\left(1+\frac{1}{H}\right)^{3j+2}\sum_{t=0}^{T_{k}}\ind_{\left(s_{t},a_{t}\right)\in\mathcal{I}}\Bigg[\frac{\left[\rho^{\pi^{\star}}(P^{\pi^{\star}})^{j}\right]\left(s_{t},a_{t}\right)}{\mu_{\mathsf{b}}\left(s_{t},a_{t}\right)}P_{s_{t},a_{t}}\left(V^{\star}-V_{t}\right)\\
 & \qquad\qquad\qquad\qquad-\left(1+\frac{1}{H}\right)\sum_{s\in\mathcal{S},a\in\mathcal{A}}\left[\rho^{\pi^{\star}}(P^{\pi^{\star}})^{j}\right]\left(s,a\right)P_{s,a}\left(V^{\star}-V_{t}\right)\Bigg].
\end{align*}
Repeat the same analysis as in Step 1 of the proof of Theorem \ref{theorem:5} (which we omit here for brevity) 
to yield
\[
\alpha_{0}\leq\underbrace{\limsup_{j\to\infty}\alpha_{j}}_{\eqqcolon\,\alpha}+\underbrace{\sum_{j=0}^{\infty}\xi_{j}}_{\eqqcolon\,\xi}+\underbrace{\sum_{j=0}^{\infty}\theta_{j}}_{\eqqcolon\,\theta} + 
\underbrace{\sum_{j=0}^{\infty}\psi_{j}}_{\eqqcolon\,\psi}+\underbrace{\sum_{j=0}^{\infty}\phi_{j}}_{\eqqcolon\,\phi},
\]
as well as the properties that $\alpha=0$, 
\begin{align*}
\xi & \lesssim\frac{2t_{\mathsf{mix}}}{1-\gamma}\log\frac{1}{\delta}+\frac{t_{\mathsf{mix}}}{T^{4}\left(1-\gamma\right)^{2}}\log\frac{1}{\delta},\\
\psi & \lesssim\frac{C^{\star}t_{\mathsf{mix}}\iota}{\left(1-\gamma\right)^{3}}\log^{2}\left(\frac{T}{\delta}\right)+\frac{C^{\star}St_{\mathsf{mix}}}{\left(1-\gamma\right)^{2}}\log\left(\frac{T}{\delta}\right),\\
\phi & \lesssim\frac{C^{\star}t_{\mathsf{mix}}\iota}{\left(1-\gamma\right)^{3}}\log^{2}\left(\frac{T}{\delta}\right)+\frac{C^{\star}St_{\mathsf{mix}}}{\left(1-\gamma\right)^{2}}\log\left(\frac{T}{\delta}\right).
\end{align*}

It then comes down to bounding $\theta$, which is different from what has been done 
in the proof of Theorem~\ref{theorem:5}. Towards this, we first invoke Lemma~\ref{lemma:useful} to reach
\begin{align}
	\theta & =\sum_{t=1}^{T_{k}}\sum_{j=0}^{\infty}\left[\gamma\left(1+\frac{1}{H}\right)^{3}\right]^{j}\sum_{s\in\mathcal{S}}\left[\rho(P_{\pi^{\star}})^{j}\right]\big(s,\pi^{\star}(s)\big)\min\left\{ \beta_{n_{t}\left(s,\pi^{\star}\left(s\right)\right)}\big(s,\pi^{\star}(s)\big),\frac{1}{1-\gamma}\right\} \notag\\
 & \lesssim\frac{1}{1-\gamma}\sum_{t=1}^{T_k}\sum_{s\in\mathcal{S}}d_{\rho}^{\star}(s)\min\left\{ \beta_{n_{t}\left(s,\pi^{\star}\left(s\right)\right)}\big(s,\pi^{\star}(s)\big),\frac{1}{1-\gamma}\right\} +\frac{1}{ST^{4}\left(1-\gamma\right)}\frac{T}{1-\gamma}.
	\label{eq:UB-theta-13579-VR}
\end{align}
To proceed, let us use the definition of $\beta_n(s,a)$ (cf.~\eqref{eq:defn-beta-n-VR}) to decompose 
\begin{align}
 & \sum_{s\in\mathcal{S}}\sum_{t=1}^{T_{k}}d_{\rho}^{\star}\left(s\right)\min\left\{ \beta_{n_{t}\left(s,\pi^{\star}(s)\right)}\big(s,\pi^{\star}(s)\big),\frac{1}{1-\gamma}\right\} \notag\\
&\lesssim\underbrace{\sum_{s\in\mathcal{S}}\sum_{t=1}^{t_{\mathsf{burn}\text{-}\mathsf{in}}(s)}d_{\rho}^{\star}\left(s\right)\frac{1}{1-\gamma}}_{\eqqcolon\,\omega_{0}}
 +\underbrace{\sum_{s\in\mathcal{S}}\sum_{t=t_{\mathsf{burn}\text{-}\mathsf{in}}(s)}^{T_{k}}d_{\rho}^{\star}\left(s\right)\sqrt{\frac{H\iota}{n_{t}\big(s,\pi^{\star}(s)\big)}\left\{ \sigma_{n}^{\mathsf{adv}}\big(s,\pi^\star(s)\big)-\left[\mu_{n}^{\mathsf{adv}}\big(s,\pi^\star(s)\big)\right]^{2}\right\} }}_{\eqqcolon\,\omega_{1}} \notag\\
 & \quad+\underbrace{\sum_{s\in\mathcal{S}}\sum_{t=t_{\mathsf{burn}\text{-}\mathsf{in}}(s)}^{T_{k}}d_{\rho}^{\star}\left(s\right)\sqrt{\frac{\iota}{n^{\mathsf{ref}}\big(s,\pi^\star(s)\big)}\left\{ \sigma^{\mathsf{ref}}\big(s,\pi^\star(s)\big)-\left[\mu^{\mathsf{ref}}\big(s,\pi^\star(s)\big)\right]^{2}\right\} }}_{\eqqcolon\,\omega_{2}} \notag\\
	& \quad+\underbrace{\sum_{s\in\mathcal{S}}\sum_{t=t_{\mathsf{burn}\text{-}\mathsf{in}}(s)}^{T_{k}}d_{\rho}^{\star}\left(s\right)\frac{H^{3/4}\iota^{3/4}}{n_{t}^{3/4}\big(s,\pi^{\star}(s)\big)\left(1-\gamma\right)}}_{\eqqcolon\,\omega_{3}}+\underbrace{\sum_{s\in\mathcal{S}}\sum_{t=t_{\mathsf{burn}\text{-}\mathsf{in}}(s)}^{T_{k}}d_{\rho}^{\star}\left(s\right)\frac{H\iota}{n_{t}\big(s,\pi^{\star}(s)\big)\left(1-\gamma\right)}}_{\eqqcolon\,\omega_{4}} \notag\\
 & \quad+\underbrace{\sum_{s\in\mathcal{S}}\sum_{t=t_{\mathsf{burn}\text{-}\mathsf{in}}(s)}^{T_{k}}d_{\rho}^{\star}\left(s\right)\frac{\iota^{3/4}}{\left(1-\gamma\right)\left[n^{\mathsf{ref}}\big(s,\pi^\star(s)\big)\right]^{3/4}}}_{\eqqcolon\,\omega_{5}}+\underbrace{\sum_{s\in\mathcal{S}}\sum_{t=t_{\mathsf{burn}\text{-}\mathsf{in}}(s)}^{T_{k}}d_{\rho}^{\star}(s)\frac{\iota}{\left(1-\gamma\right)n^{\mathsf{ref}}\big(s,\pi^\star(s)\big)}}_{\eqqcolon\,\omega_{6}},
	\label{eq:UB-theta-24680-VR}
\end{align}
where we define, for each $s\in\mathcal{S}$, that
\[
	t_{\mathsf{burn}\text{-}\mathsf{in}}(s)\coloneqq C_{\mathsf{burn\text{-}in}}\frac{t_{\mathsf{mix}}}{\mu_{\mathsf{b}}\big(s,\pi^{\star}(s)\big)}\log\left(\frac{ST}{\delta}\right)
\]
for some sufficiently large constant $C_{\mathsf{burn\text{-}in}}>0$.
%This motivates us to bound the terms $\omega_{0},\ldots,\omega_{6}$ separately. 

Before continuing, we first collect a few immediate and useful results of \citet[Lemma 8]{li2021sample}:
with probability exceeding $1-\delta$, we have
\begin{equation}
	n_{t}\big(s,\pi^{\star}\left(s\right)\big)\gtrsim t\mu_{\mathsf{b}}\big(s,\pi^{\star}(s)\big),\qquad\forall s\in\mathcal{S}~\text{ and }~t_{\mathsf{burn}\text{-}\mathsf{in}}(s)\leq t\leq T_{k}\label{eq:nt-lower-bound}
\end{equation}
and when $T_k^{\mathsf{ref}}\asymp T_k \geq t_{\mathsf{burn}\text{-}\mathsf{in}}(s)$, one has
\begin{equation}
n^{\mathsf{ref}}\big(s,\pi^{\star}(s)\big)\gtrsim T_{k}^{\mathsf{ref}}\mu_{\mathsf{b}}\big(s,\pi^{\star}(s)\big),\qquad\forall s\in\mathcal{S},\label{eq:n-ref-lower-bound}
\end{equation}
provided that $C_{\mathsf{burn\text{-}in}}$ is large enough. We then bound the terms $\omega_{0},\ldots,\omega_{6}$ separately. 

\begin{itemize}
\item The first bound $\omega_{0}$ can be easily bounded by 
\begin{align*}
\omega_{0} & \leq\sum_{s\in\mathcal{S}}t_{\mathsf{burn}\text{-}\mathsf{in}}(s)d_{\rho}^{\star}\left(s\right)\frac{1}{1-\gamma}\lesssim\sum_{s\in\mathcal{S}}\frac{t_{\mathsf{mix}}\iota}{\mu_{\mathsf{b}}\big(s,\pi^{\star}(s)\big)}d_{\rho}^{\star}\big(s,\pi^{\star}(s)\big)\frac{1}{1-\gamma}\\
 & \leq\sum_{s\in\mathcal{S}}\frac{C^{\star}t_{\mathsf{mix}}\iota}{1-\gamma}\asymp\frac{C^{\star}St_{\mathsf{mix}}\iota}{1-\gamma}, 
\end{align*}
where the last line follows from Assumption~\ref{assumption:policy}.

\item To control $\omega_{1}$, we observe that 
\begin{align*}
	\omega_{1} & \lesssim\sum_{s\in\mathcal{S}}\sum_{t=t_{\mathsf{burn}\text{-}\mathsf{in}}(s)}^{T_{k}}d_{\rho}^{\star}(s)\sqrt{\frac{H\iota}{n_{t}\big(s,\pi^{\star}(s)\big)}\sigma_{n}^{\mathsf{adv}}\big(s,\pi^{\star}(s)\big)}\\
 & \lesssim\underbrace{\sum_{s\in\mathcal{S}}\sum_{t=t_{\mathsf{burn}\text{-}\mathsf{in}}(s)}^{T_{k}}d_{\rho}^{\star}\left(s\right)\sqrt{\frac{H\iota}{n_{t}\big(s,\pi^{\star}(s)\big)} 
	 P_{s,\pi^{\star}(s)}\left(V^{\star}-\overline{V}\right)^{2} }}_{\eqqcolon\,\omega_{1,1}}+\underbrace{\sum_{s\in\mathcal{S}}\sum_{t=t_{\mathsf{burn}\text{-}\mathsf{in}}(s)}^{T_{k}}d_{\rho}^{\star}(s)\frac{H^{3/4}\iota^{3/4}}{\left(1-\gamma\right)n_{t}^{3/4}\big(s,\pi^{\star}(s)\big)}}_{\eqqcolon\,\omega_{1,2}},
\end{align*}
where the last inequality follows from Lemma \ref{lemma:variance-bound}.
The first term $\omega_{1,1}$ admits the following bound
\begin{align*}
	\omega_{1,1} & \overset{\text{(i)}}{\asymp}\sum_{s\in\mathcal{S}}\sum_{t=t_{\mathsf{burn}\text{-}\mathsf{in}}(s)}^{T_{k}}d_{\rho}^{\star}\big(s,\pi^{\star}(s)\big)\sqrt{\frac{H\iota}{t\mu_{\mathsf{b}}\big(s,\pi^{\star}(s)\big)}P_{s,\pi^{\star}(s)}\left(V^{\star}-\overline{V}\right)^{2}}\\
 & \overset{\text{(ii)}}{\lesssim}\sum_{s\in\mathcal{S}}\sum_{t=1}^{T_{k}}\sqrt{\frac{C^{\star}H\iota}{t}d_{\rho}^{\star}\big(s,\pi^{\star}(s)\big)P_{s,\pi^{\star}(s)}\left(V^{\star}-\overline{V}\right)^{2}}\\
 & \overset{\text{(iii)}}{\lesssim}\sqrt{C^{\star}H\iota T_{k}}\sum_{s\in\mathcal{S}}\sqrt{d_{\rho}^{\star}\big(s,\pi^{\star}(s)\big)P_{s,\pi^{\star}(s)}\left(V^{\star}-\overline{V}\right)^{2}}\\
 & \overset{\text{(iv)}}{\lesssim}\sqrt{C^{\star}SH\iota T_{k}}\sqrt{\sum_{s\in\mathcal{S}}d_{\rho}^{\star}\big(s,\pi^{\star}(s)\big)P_{s,\pi^{\star}(s)}\left(V^{\star}-\overline{V}\right)^{2}}\\
 & \asymp\sqrt{C^{\star}SH\iota T_{k}}\sqrt{\sum_{s\in\mathcal{S},a\in\mathcal{A}}d_{\rho}^{\star}\left(s,a\right)P_{s,a}\left(V^{\star}-\overline{V}\right)^{2}}\\
 & \overset{\text{(v)}}{\lesssim}\sqrt{\frac{C^{\star}S\iota^{2}T_{k}}{\left(1-\gamma\right)^{2}}}\sqrt{\Delta_{k-1}}.
\end{align*}
Here, (i) follows from (\ref{eq:nt-lower-bound}); (ii) utilizes Assumption
\ref{assumption:policy}; (iii) arises from (\ref{eq:sum-sqrt});
(iv) applies the Cauchy-Schwarz inequality; and (v) comes from
Lemma~\ref{lemma:variance-bound} and  the definition of $H$ (i.e., $H\asymp \frac{\iota}{1-\gamma}$). The other term $\omega_{1,2}$ is identical to $\omega_{3}$, which we shall bound momentarily.

\item When it comes to $\omega_{2}$, we apply Lemma~\ref{lemma:variance-bound} to reach 
\begin{align*}
\omega_{2} & \lesssim\underbrace{\sum_{s\in\mathcal{S}}\sum_{t=t_{\mathsf{burn}\text{-}\mathsf{in}}(s)}^{T_{k}}d_{\rho}^{\star}\left(s\right)\sqrt{\frac{\iota}{n^{\mathsf{ref}}\big(s,\pi^{\star}(s)\big)}\mathsf{Var}_{s,\pi^{\star}(s)}(\overline{V})}}_{\eqqcolon\,\omega_{2,1}}+\underbrace{\sum_{s\in\mathcal{S}}\sum_{t=t_{\mathsf{burn}\text{-}\mathsf{in}}(s)}^{T_{k}}d_{\rho}^{\star}\left(s\right)\frac{\iota^{3/4}}{\left(1-\gamma\right)\left[n^{\mathsf{ref}}\big(s,\pi^{\star}(s)\big)\right]^{3/4}}}_{\eqqcolon\,\omega_{2,2}}.
\end{align*}
Regarding $\omega_{2,1}$, we make the observation that 
\begin{align*}
	\omega_{2,1} & \overset{\text{(i)}}{\asymp}\sum_{s\in\mathcal{S}}\sum_{t=t_{\mathsf{burn}\text{-}\mathsf{in}}(s)}^{T_{k}}d_{\rho}^{\star}\big(s,\pi^{\star}(s)\big)\sqrt{\frac{\iota}{T_{k}\mu_{\mathsf{b}}\big(s,\pi^{\star}(s)\big)}\mathsf{Var}_{s,\pi^{\star}(s)}(\overline{V})}\\
 & \overset{\text{(ii)}}{\lesssim}\sqrt{C^{\star}\iota T_{k}}\sum_{s\in\mathcal{S}}\sqrt{d_{\rho}^{\star}\big(s,\pi^{\star}(s)\big)\mathsf{Var}_{s,\pi^{\star}(s)}(\overline{V})}\\
 & \asymp\sqrt{C^{\star}\iota T_{k}}\sum_{s\in\mathcal{S},a\in\mathcal{A}}\sqrt{d_{\rho}^{\star}\left(s,a\right)\mathsf{Var}_{s,a}(\overline{V})}\\
 & \overset{\text{(iii)}}{\lesssim}\sqrt{C^{\star}S\iota T_{k}}\sqrt{\sum_{s\in\mathcal{S},a\in\mathcal{A}}d_{\rho}^{\star}\left(s,a\right)\mathsf{Var}_{s,a}(\overline{V})}\\
 & \overset{\text{(iv)}}{\lesssim}\sqrt{C^{\star}S\iota T_{k}}\sqrt{\frac{1}{1-\gamma}+\frac{\Delta_{k-1}}{1-\gamma}}\\
 & \asymp\sqrt{\frac{C^{\star}S\iota T_{k}}{1-\gamma}}+\sqrt{\frac{C^{\star}S\iota T_{k}}{1-\gamma}}\sqrt{\Delta_{k-1}}.
\end{align*}
Here, (i) relies on (\ref{eq:n-ref-lower-bound}); (ii) invokes Assumption
\ref{assumption:policy}; (iii) utilizes the Cauchy-Schwarz inequality;
and (iv) follows from Lemma \ref{lemma:variance-bound}. The other term
$\omega_{2,2}$ is the same as $\omega_{5}$, which will be bounded
momentarily. 

\item Regarding $\omega_{3}$, we have 
\begin{align*}
	\omega_{3} & \overset{\text{(i)}}{\asymp}\sum_{s\in\mathcal{S}}\sum_{t=t_{\mathsf{burn}\text{-}\mathsf{in}}(s)}^{T_{k}}d_{\rho}^{\star}\big(s,\pi^{\star}(s)\big)\frac{H^{3/4}\iota^{3/4}}{\left(1-\gamma\right)t^{3/4}\mu_{\mathsf{b}}^{3/4}\big(s,\pi^{\star}(s)\big)}\\
 & \overset{\text{(ii)}}{\lesssim}\frac{(C^{\star})^{3/4}H^{3/4}\iota^{3/4}}{1-\gamma}\sum_{s\in\mathcal{S}}\sum_{t=1}^{T_{k}}\frac{\left[d^{\star}\big(s,\pi^{\star}(s)\big)\right]^{1/4}}{t^{3/4}}\\
 & \overset{\text{(iii)}}{\lesssim}T_{k}^{1/4}\frac{(C^{\star})^{3/4}H^{3/4}\iota^{3/4}}{1-\gamma}\sum_{s\in\mathcal{S}}\left[d^{\star}\big(s,\pi^{\star}(s)\big)\right]^{1/4}\\
 & \overset{\text{(iv)}}{\lesssim}T_{k}^{1/4}\frac{(C^{\star})^{3/4}H^{3/4}\iota^{3/4}}{1-\gamma}\left(\sum_{s\in\mathcal{S}}1\right)^{3/4}\left(\sum_{s\in\mathcal{S}}d^{\star}\big(s,\pi^{\star}(s)\big)\right)^{1/4} \\
 & \asymp T_{k}^{1/4}\frac{(C^{\star})^{3/4}S^{3/4}\iota^{3/2}}{\left(1-\gamma\right)^{7/4}}. 
\end{align*}
Here, (i) follows from (\ref{eq:nt-lower-bound}); (ii) utilizes Assumption
\ref{assumption:policy}; (iii) follows from the fact that for any
$T\geq1$,
\begin{equation}
\sum_{t=1}^{T}\frac{1}{t^{3/4}}\leq1+\int_{1}^{T}x^{-3/4}\mathrm{d}x=1+4\left(T^{1/4}-1\right)\leq4T^{1/4};\label{eq:sum-3/4}
\end{equation}
(iv) follows from H{\"o}lder's inequality; and the last line holds since $\sum_s d_{\rho}^{\star}\big(s,\pi^{\star}(s)\big) = 1$.

\item Regarding $\omega_{4}$, we have 
\begin{align*}
\omega_{4} & \overset{\text{(i)}}{\asymp}\sum_{s\in\mathcal{S}}\sum_{t=t_{\mathsf{burn}\text{-}\mathsf{in}}(s)}^{T_{k}}d_{\rho}^{\star}\big(s,\pi^{\star}(s)\big)\frac{H\iota}{\left(1-\gamma\right)t\mu_{\mathsf{b}}\big(s,\pi^{\star}(s)\big)}\\
 & \overset{\text{(ii)}}{\lesssim}\frac{C^{\star}H\iota}{1-\gamma}\sum_{s\in\mathcal{S}}\sum_{t=1}^{T_{k}}\frac{1}{t}\overset{\text{(iii)}}{\lesssim}\frac{C^{\star}S\iota^{2}\log T_{k}}{\left(1-\gamma\right)^{2}}.
\end{align*}
Here, (i) utilizes (\ref{eq:nt-lower-bound}); (ii) relies on Assumption
\ref{assumption:policy}; and (iii) follows from the fact that for
any $T\geq1$,
\begin{equation}
\sum_{t=1}^{T}\frac{1}{t}\leq1+\int_{1}^{T}x^{-1}\mathrm{d}x=1+\left(\log T-1\right)\leq\log T;\label{eq:sum-1}
\end{equation}

\item Moving on to $\omega_{5}$, we develop the following upper bound:  
\begin{align*}
\omega_{5} & \overset{\text{(i)}}{\asymp}\sum_{s\in\mathcal{S}}\sum_{t=t_{\mathsf{burn}\text{-}\mathsf{in}}(s)}^{T_{k}}d_{\rho}^{\star}\big(s,\pi^{\star}(s)\big)\frac{\iota^{3/4}}{\left(1-\gamma\right)(T_{k}^{\mathsf{ref}})^{3/4}\mu_{\mathsf{b}}^{3/4}\big(s,\pi^{\star}(s)\big)}\\
 & \overset{\text{(ii)}}{\lesssim}\frac{(C^{\star})^{3/4}\iota^{3/4}}{1-\gamma}\sum_{s\in\mathcal{S}}\sum_{t=1}^{T_{k}}\frac{\left[d_{\rho}^{\star}\big(s,\pi^{\star}(s)\big)\right]^{1/4}}{(T_{k}^{\mathsf{ref}})^{3/4}}\\
 & \overset{\text{(iii)}}{\asymp}T_{k}^{1/4}\frac{(C^{\star})^{3/4}\iota^{3/4}}{1-\gamma}\sum_{s\in\mathcal{S}}\left[d_{\rho}^{\star}\big(s,\pi^{\star}(s)\big)\right]^{1/4}\\
 & \overset{\text{(iv)}}{\lesssim}T_{k}^{1/4}\frac{(C^{\star})^{3/4}S^{3/4}\iota^{3/4}}{1-\gamma}.
\end{align*}
Here, (i) follows from (\ref{eq:n-ref-lower-bound}); (ii) results from 
Assumption \ref{assumption:policy}; (iii) holds since $T_{k}^{\mathsf{ref}}\asymp T_{k}$;
and (iv) invokes H{\"o}lder's inequality and $\sum_s d_{\rho}^{\star}\big(s,\pi^{\star}(s)\big) = 1$ once again.

\item We are left with bounding the last term $\omega_{6}$, towards which we observe 
\begin{align*}
\omega_{6} & \overset{\text{(i)}}{\asymp}\sum_{s\in\mathcal{S}}\sum_{t=t_{\mathsf{burn}\text{-}\mathsf{in}}(s)}^{T_{k}}d_{\rho}^{\star}\big(s,\pi^{\star}(s)\big)\frac{\iota}{\left(1-\gamma\right)T_{k}^{\mathsf{ref}}\mu_{\mathsf{b}}\big(s,\pi^{\star}(s)\big)}\\
 & \overset{\text{(ii)}}{\lesssim}\frac{C^{\star}\iota}{1-\gamma}\sum_{s\in\mathcal{S}}\frac{T_{k}}{T_{k}^{\mathsf{ref}}}\\
 & \overset{\text{(iii)}}{\asymp}\frac{C^{\star}S\iota}{1-\gamma}.
\end{align*}
Here, (i) follows from (\ref{eq:n-ref-lower-bound}); (ii) utilizes
Assumption \ref{assumption:policy}; and (iii) holds since $T_{k}^{\mathsf{ref}}\asymp T_{k}$. 
\end{itemize}
Taking the preceding bounds on $\omega_{0}$, $\omega_{1}$, $\omega_{2}$,
$\omega_{3}$, $\omega_{4}$, $\omega_{5}$ and $\omega_{6}$ together with \eqref{eq:UB-theta-13579-VR} and \eqref{eq:UB-theta-24680-VR} yields
\begin{align*}
	\theta & \lesssim\frac{1}{1-\gamma}\sum_{t=1}^{T_{k}}\sum_{s\in\mathcal{S}}d_{\rho}^{\star}(s) \min\left\{ \beta_{n_{t}(s,\pi^{\star}(s))}, \frac{1}{1-\gamma} \right\}
	+\frac{1}{ST^{4}\left(1-\gamma\right)}\frac{T}{1-\gamma}\\
 & \lesssim\frac{1}{1-\gamma}\left(\omega_{0}+\omega_{1}+\omega_{2}+\omega_{3}+\omega_{4}+\omega_{5}+\omega_{6}\right)+\frac{1}{ST^{4}\left(1-\gamma\right)}\frac{T}{1-\gamma}\\
 & \lesssim\frac{C^{\star}St_{\mathsf{mix}}\iota}{\left(1-\gamma\right)^{2}}+\sqrt{\frac{C^{\star}S\iota^{2}T_{k}}{\left(1-\gamma\right)^{4}}}\sqrt{\Delta_{k-1}}+\sqrt{\frac{C^{\star}S\iota T_{k}}{\left(1-\gamma\right)^{3}}}+T_{k}^{1/4}\frac{(C^{\star})^{3/4}S^{3/4}\iota^{3/2}}{\left(1-\gamma\right)^{11/4}}+\frac{C^{\star}S\iota^{2}\log T_{k}}{\left(1-\gamma\right)^{2}}\\
 & \lesssim\frac{C^{\star}St_{\mathsf{mix}}\iota}{\left(1-\gamma\right)^{2}}+\sqrt{\frac{C^{\star}S\iota^{2}T_{k}}{\left(1-\gamma\right)^{4}}}\sqrt{\Delta_{k-1}}+\sqrt{\frac{C^{\star}S\iota T_{k}}{\left(1-\gamma\right)^{3}}}+\frac{C^{\star}S\iota^{3}}{\left(1-\gamma\right)^{4}}, 
\end{align*}
where the last line has invoked the AM-GM inequality:  
\begin{align*}
2T_{k}^{1/4}\frac{(C^{\star})^{3/4}S^{3/4}\iota^{3/2}}{\left(1-\gamma\right)^{11/4}} & =2\frac{T_{k}^{1/4}(C^{\star})^{1/4}S^{1/4}}{\left(1-\gamma\right)^{3/4}}\cdot\frac{(C^{\star})^{1/2}S^{1/2}\iota^{3/2}}{\left(1-\gamma\right)^{2}}\leq\frac{T_{k}^{1/2}(C^{\star})^{1/2}S^{1/2}}{\left(1-\gamma\right)^{3/2}}+\frac{C^{\star}S\iota^{3}}{\left(1-\gamma\right)^{4}}.
\end{align*}

Putting all of the above results together, we can conclude that
\begin{align}
\Lambda_{k} & \leq\frac{1}{T_{k}}\alpha_{0}\leq\frac{1}{T_{k}}\left(\alpha+\xi+\theta+\psi+\phi\right)\nonumber \\
 & \lesssim\sqrt{\frac{C^{\star}S\iota^{2}}{T_{k}\left(1-\gamma\right)^{4}}}\sqrt{\Delta_{k-1}}+\sqrt{\frac{C^{\star}S\iota}{T_{k}\left(1-\gamma\right)^{3}}}+\frac{C^{\star}S\iota^{3}}{T_{k}\left(1-\gamma\right)^{4}}+\frac{C^{\star}St_{\mathsf{mix}}\iota}{T_{k}\left(1-\gamma\right)^{2}}+\frac{C^{\star}t_{\mathsf{mix}}\iota^{2}}{T_{k}\left(1-\gamma\right)^{3}}.
	\label{eq:Lambda_k-bound}
\end{align}

\subsection{Step 2: bounding $\Delta_{k}$ by induction}

Thus far, we have established an intimate connection between $\Lambda_{k}$ and $\Delta_k$ (see~\eqref{eq:Lambda_k-bound}). 
In order to bound $\Delta_{k-1}$, we find it helpful to look at an auxiliary test distribution
\[
\widetilde{\rho}=\frac{d_{\rho}^{\star}-\left(1-\gamma\right)\rho}{\gamma}
\]
instead of $\rho$. The following property about $\widetilde{\rho}$ plays an important role in the subsequent analysis.
\begin{lemma}\label{lemma:d-stat-star} Suppose that $1/2\leq \gamma <1$. It holds that 
\begin{equation}
\sum_{j=0}^{\infty}\left[\gamma\left(1+\frac{1}{H}\right)^{3}\right]^{j}\left\langle \widetilde{\rho}(P_{\pi^{\star}})^{j},V\right\rangle \lesssim\frac{1}{1-\gamma}\left\langle d_{\rho}^{\star},V\right\rangle \log\frac{ST}{\delta}+\frac{\delta}{ST^{4}\left(1-\gamma\right)}\left\Vert V\right\Vert _{\infty}.\label{eq:d-star-star-useful}
\end{equation}
\end{lemma}
\begin{proof}See Appendix \ref{appendix:proof-lemma-d-stat-star}.\end{proof}

Armed with Lemma \ref{lemma:d-stat-star}, we can repeat the same
analysis in Step 1 to bound each $\Delta_{k}$. The difference between
(\ref{eq:useful}) and (\ref{eq:d-star-star-useful}) requires us
to replace $d_{\rho}^{\star}$ in Step 1 with $d_{\rho}^{\star} \log(ST/\delta)$,
which leads to 
\begin{align*}
\Delta_{k} & \lesssim\sqrt{\frac{C^{\star}S\iota^{4}}{T_{k}\left(1-\gamma\right)^{4}}}\sqrt{\Delta_{k-1}}+\sqrt{\frac{C^{\star}S\iota^{3}}{T_{k}\left(1-\gamma\right)^{3}}}+\frac{C^{\star}S\iota^{4}}{T_{k}\left(1-\gamma\right)^{4}}
	+\frac{C^{\star}St_{\mathsf{mix}}\iota^{2}}{T_{k}\left(1-\gamma\right)^{2}}+\frac{C^{\star}t_{\mathsf{mix}}\iota^{3}}{T_{k}\left(1-\gamma\right)^{3}} 
\end{align*}
for each $1\leq k\leq K$. The above inequality can be expressed as follows
\[
\Delta_{k}\leq\alpha_{k}\Delta_{k-1}^{1/2}+\beta_{k},
\]
where 
\[
\alpha_{k}=C\sqrt{\frac{C^{\star}S\iota^{4}}{T_{k}\left(1-\gamma\right)^{4}}}=2^{-k}A 
\qquad\text{with}\qquad A=C\sqrt{\frac{C^{\star}S\iota^{4}}{\left(1-\gamma\right)^{4}}}
\]
and 
\[
\beta_{k}=C\sqrt{\frac{C^{\star}S\iota^{3}}{T_{k}\left(1-\gamma\right)^{3}}}+C\frac{C^{\star}S\iota^{4}}{T_{k}\left(1-\gamma\right)^{4}}+C\frac{C^{\star}St_{\mathsf{mix}}\iota^{2}}{T_{k}\left(1-\gamma\right)^{2}}+C\frac{C^{\star}t_{\mathsf{mix}}\iota^{3}}{T_{k}\left(1-\gamma\right)^{3}}
\]
for some sufficiently large constant $C>0$. In addition, it is also observed that 
\[
\Delta_{0}\leq\frac{1}{1-\gamma}.
\]
By induction, for each $1\leq j\leq K$ we have 
\begin{align*}
\Delta_{j} & \leq\underbrace{\beta_{j}}_{\eqqcolon\delta_{j}}+\underbrace{\alpha_{j}\beta_{j-1}^{1/2}}_{\eqqcolon\delta_{j-1}}+\underbrace{\alpha_{j}\alpha_{j-1}^{1/2}\beta_{j-2}^{1/4}}_{\eqqcolon\delta_{j-2}}+\cdots+\underbrace{\alpha_{j}\alpha_{j-1}^{1/2}\alpha_{j-2}^{1/4}\cdots\alpha_{2}^{1/2^{j-2}}\beta_{1}^{1/2^{j-1}}}_{\eqqcolon\delta_{1}}\\
 & \quad+\underbrace{\alpha_{j}\alpha_{j-1}^{1/2}\alpha_{j-2}^{1/4}\cdots\alpha_{2}^{1/2^{j-2}}\alpha_{1}^{1/2^{j-1}}\Delta_{0}^{1/2^{j}}}_{\eqqcolon\delta_{0}}.
\end{align*}
In the sequel, we bound each term $\delta_{i}$, $0\leq i\leq j$ separately. 
\begin{itemize}
\item Let us begin with $\delta_{0}$, which can be calculated as follows
\[
\alpha_{j}\alpha_{j-1}^{1/2}\alpha_{j-2}^{1/4}\cdots\alpha_{2}^{1/2^{j-2}}\alpha_{1}^{1/2^{j-1}}=A^{2-1/2^{j-1}}2^{-j-\frac{j-1}{2}-\frac{j-2}{4}-\cdots-\frac{1}{2^{j-1}}}.
\]
Note that 
\begin{align*}
j+\frac{j-1}{2}+\frac{j-2}{4}+\cdots+\frac{1}{2^{j-1}} & =\sum_{k=0}^{j-1}\frac{j-k}{2^{k}}=j\sum_{k=0}^{j-1}\frac{1}{2^{k}}-\sum_{k=0}^{j-1}\frac{k}{2^{k}}\\
 & =j\left(2-\frac{1}{2^{j-1}}\right)-2+\frac{j+1}{2^{j-1}}\\
 & =2j-2+\frac{1}{2^{j-1}}, 
\end{align*}
where the penultimate line holds since 
\[
\sum_{k=0}^{j-1}\frac{k}{2^{k}}=\sum_{k=1}^{j-1}\frac{k}{2^{k-1}}-\sum_{k=0}^{j-1}\frac{k}{2^{k}}=\sum_{k=0}^{j-2}\frac{k+1}{2^{k}}-\sum_{k=0}^{j-1}\frac{k}{2^{k}}=\sum_{k=0}^{j-2}\frac{1}{2^{k}}-\frac{j-1}{2^{j-1}}=2-\frac{j+1}{2^{j-1}}.
\]
Therefore we have 
\begin{align*}
\delta_{0} & =A^{2-1/2^{j-1}}4^{-j+1-1/2^{j}}\Delta_{0}^{1/2^{j}}\asymp\frac{1}{4^{j}}\left[C\sqrt{\frac{C^{\star}S\iota^{4}}{\left(1-\gamma\right)^{4}}}\right]^{2-1/2^{j-1}}\left(\frac{1}{1-\gamma}\right)^{1/2^{j}}\\
 & \lesssim\frac{1}{T_{j}}\left[\frac{C^{\star}S\iota^{4}}{\left(1-\gamma\right)^{4}}\right]^{1-1/2^{j}}\left(\frac{1}{1-\gamma}\right)^{1/2^{j}}\lesssim\frac{C^{\star}S\iota^{4}}{T_{j}\left(1-\gamma\right)^{4}}.
\end{align*}

\item Next, we develop a uniform bound on every $\delta_{i}$, $1\leq i\leq j-1$.
We first observe that 
\begin{align*}
\alpha_{j}\alpha_{j-1}^{1/2}\alpha_{j-2}^{1/4}\cdots\alpha_{j+1}^{1/2^{j-i-1}}\beta_{i}^{1/2^{j-i}} & =A^{2-1/2^{j-i-1}}2^{-j-\frac{j-1}{2}-\frac{j-2}{4}-\cdots-\frac{i+1}{2^{j-i-1}}}\beta_{i}^{1/2^{j-i}},
\end{align*}
and
\begin{align*}
j+\frac{j-1}{2}+\frac{j-2}{4}+\cdots+\frac{j-j}{2^{j}} & =\sum_{k=0}^{j-i}\frac{j-i}{2^{k}}=j\sum_{k=0}^{j-i}\frac{1}{2^{k}}-\sum_{k=0}^{n-i}\frac{k}{2^{k}}\\
 & =j\left(2-\frac{1}{2^{j-i}}\right)-2+\frac{j+2-i}{2^{j-i}}\\
 & =2j-2+\frac{2-i}{2^{j-i}}\\
 & \geq2j-\frac{i}{2^{j-i}}-2,
\end{align*}
where the penultimate line holds since 
\[
\sum_{k=0}^{j-i}\frac{k}{2^{k}}=2-\frac{j+2-i}{2^{j-i}}.
\]
These properties allow one to derive
\begin{align*}
\delta_{i} & \leq A^{2-1/2^{j-i-1}}4^{-j+i/2^{j-i+1}+1}\beta_{i}^{1/2^{j-i}}\asymp4^{-j+i/2^{j-i+1}}\left(\frac{C^{\star}S\iota^{4}}{\left(1-\gamma\right)^{4}}\right)^{1-1/2^{j-i}}\beta_{i}^{1/2^{j-i}}\\
 & \asymp\frac{1}{T_{j}}\left(\frac{C^{\star}S\iota^{4}}{\left(1-\gamma\right)^{4}}\right)^{1-1/2^{j-i}}\left(\sqrt{\frac{C^{\star}S\iota^{3}}{\left(1-\gamma\right)^{3}}}\right)^{1/2^{j-i}}\\
 & \quad+\frac{4^{-j/2^{j-i+1}}}{T_{j}}\left(\frac{C^{\star}S\iota^{4}}{\left(1-\gamma\right)^{4}}\right)^{1-1/2^{j-i}}\left(\frac{C^{\star}S\iota^{4}}{\left(1-\gamma\right)^{4}}\right)^{1/2^{j-i}}\\
 & \quad+\frac{4^{-j/2^{j-i+1}}}{T_{j}}\left(\frac{C^{\star}S\iota^{4}}{\left(1-\gamma\right)^{4}}\right)^{1-1/2^{j-i}}\left(\frac{C^{\star}St_{\mathsf{mix}}\iota^{2}}{\left(1-\gamma\right)^{2}}\right)^{1/2^{j-i}}\\
 & \quad+\frac{4^{-j/2^{j-i+1}}}{T_{j}}\left(\frac{C^{\star}S\iota^{4}}{\left(1-\gamma\right)^{4}}\right)^{1-1/2^{j-i}}\left(\frac{C^{\star}t_{\mathsf{mix}}\iota^{3}}{\left(1-\gamma\right)^{3}}\right)^{1/2^{j-i}}\\
 & \lesssim\frac{C^{\star}S\iota^{4}}{T_{j}\left(1-\gamma\right)^{4}}+\frac{1}{T_{j}}\sqrt{\frac{C^{\star}S\iota^{3}}{\left(1-\gamma\right)^{3}}}+\frac{C^{\star}St_{\mathsf{mix}}\iota^{2}}{T_{j}\left(1-\gamma\right)^{2}}+\frac{C^{\star}t_{\mathsf{mix}}\iota^{3}}{T_{j}\left(1-\gamma\right)^{3}}.
\end{align*}
Here, the last line has used the weighted AM-GM inequality that $\alpha x+\beta y\geq(\alpha+\beta)x^{\alpha/(\alpha+\beta)}y^{\beta/(\alpha+\beta)}$
for all $\alpha,\beta,x,y>0$. 
\end{itemize}
Armed with the above results, we can readily conclude that 
\begin{align}
\Delta_{j} & \leq\delta_{0}+\delta_{j}+j\max_{1\leq i\leq j-1}\delta_{i}=\delta_{0}+\beta_{j}+j\max_{1\leq i\leq j-1}\delta_{i}\nonumber \\
 & \lesssim\frac{C^{\star}S\iota^{4}j}{T_{j}\left(1-\gamma\right)^{4}}+\frac{j}{T_{j}}\sqrt{\frac{C^{\star}S\iota^{3}}{\left(1-\gamma\right)^{3}}}+\frac{C^{\star}St_{\mathsf{mix}}\iota^{2}j}{T_{j}\left(1-\gamma\right)^{2}}+\frac{C^{\star}t_{\mathsf{mix}}\iota^{3}j}{T_{j}\left(1-\gamma\right)^{3}}+\sqrt{\frac{C^{\star}S\iota^{3}}{T_{k}\left(1-\gamma\right)^{3}}}\nonumber \\
 & \lesssim\sqrt{\frac{C^{\star}S\iota^{3}}{T_{j}\left(1-\gamma\right)^{3}}}+\frac{C^{\star}S\iota^{4}j}{T_{j}\left(1-\gamma\right)^{4}}+\frac{C^{\star}St_{\mathsf{mix}}\iota^{2}j}{T_{j}\left(1-\gamma\right)^{2}}+\frac{C^{\star}t_{\mathsf{mix}}\iota^{3}j}{T_{j}\left(1-\gamma\right)^{3}}, \label{eq:Delta-j-bound}
\end{align}
where the last line relies on the fact that $T_{j}^{1/2}\asymp2^{j}\gtrsim j$.

\subsection{Step 3: putting all this together}

Recall from (\ref{eq:Lambda_k-bound}) that
\begin{align*}
\Lambda_{K} & \lesssim\sqrt{\frac{C^{\star}S\iota^{2}}{T\left(1-\gamma\right)^{4}}}\sqrt{\Delta_{K-1}}+\sqrt{\frac{C^{\star}S\iota}{T\left(1-\gamma\right)^{3}}}+\frac{C^{\star}S\iota^{3}}{T\left(1-\gamma\right)^{4}}+\frac{C^{\star}St_{\mathsf{mix}}\iota}{T\left(1-\gamma\right)^{2}}+\frac{C^{\star}t_{\mathsf{mix}}\iota^{2}}{T\left(1-\gamma\right)^{3}},
\end{align*}
which invokes the fact that $T_{K}\asymp T$. From (\ref{eq:Delta-j-bound})
and the fact that $T_{K-1}\asymp T$, $K\asymp\log T$, we see that
\[
\Delta_{K-1}\lesssim\sqrt{\frac{C^{\star}S\iota^{3}}{T\left(1-\gamma\right)^{3}}}+\frac{C^{\star}S\iota^{4}\log T}{T\left(1-\gamma\right)^{4}}+\frac{C^{\star}St_{\mathsf{mix}}\iota^{2}\log T}{T\left(1-\gamma\right)^{2}}+\frac{C^{\star}t_{\mathsf{mix}}\iota^{3}\log T}{T\left(1-\gamma\right)^{3}},
\]
which in turn allows one to deduce that
\begin{align*}
\sqrt{\frac{C^{\star}S\iota^{2}}{T\left(1-\gamma\right)^{4}}}\sqrt{\Delta_{k-1}} & \lesssim\underbrace{\left(\frac{C^{\star}S\iota^{2}}{T\left(1-\gamma\right)^{4}}\right)^{1/2}\left(\frac{C^{\star}S\iota^{3}}{T\left(1-\gamma\right)^{3}}\right)^{1/4}}_{\eqqcolon\zeta_{1}}+\underbrace{\left(\frac{C^{\star}S\iota^{2}}{T\left(1-\gamma\right)^{4}}\right)^{1/2}\left(\frac{C^{\star}S\iota^{4}\log T}{T\left(1-\gamma\right)^{4}}\right)^{1/2}}_{\eqqcolon\zeta_{2}}\\
 & \quad+\underbrace{\left(\frac{C^{\star}S\iota^{2}}{T\left(1-\gamma\right)^{4}}\right)^{1/2}\left(\frac{C^{\star}St_{\mathsf{mix}}\iota^{2}\log T}{T\left(1-\gamma\right)^{2}}\right)^{1/2}}_{\eqqcolon\zeta_{3}}+\underbrace{\left(\frac{C^{\star}S\iota^{2}}{T\left(1-\gamma\right)^{4}}\right)^{1/2}\left(\frac{C^{\star}t_{\mathsf{mix}}\iota^{3}\log T}{T\left(1-\gamma\right)^{3}}\right)^{1/2}}_{\eqqcolon\zeta_{4}}\\
 & \lesssim\sqrt{\frac{C^{\star}S\iota}{T\left(1-\gamma\right)^{3}}}+\frac{C^{\star}S\iota^{3}\log T}{T\left(1-\gamma\right)^{4}}+\frac{C^{\star}St_{\mathsf{mix}}\iota}{T\left(1-\gamma\right)^{2}}+\frac{C^{\star}t_{\mathsf{mix}}\iota^{2}}{T\left(1-\gamma\right)^{3}}. 
\end{align*}
Here, the last step follows by applying the AM-GM inequality as follows:
\begin{align*}
\zeta_{1} & \lesssim \frac{C^{\star}S\iota^{3}}{T\left(1-\gamma\right)^{4}} + \sqrt{\frac{C^{\star}S\iota}{T\left(1-\gamma\right)^{3}}},\\
\zeta_{2} & \lesssim\frac{C^{\star}S\iota^{3}\sqrt{\log T}}{T\left(1-\gamma\right)^{4}},\\
\zeta_{3} & \lesssim\frac{C^{\star}S\iota^{3}\log T}{T\left(1-\gamma\right)^{4}}+\frac{C^{\star}St_{\mathsf{mix}}\iota}{T\left(1-\gamma\right)^{2}}\\
\zeta_{4} & \lesssim\frac{C^{\star}S\iota^{3}\log T}{T\left(1-\gamma\right)^{4}}+\frac{C^{\star}t_{\mathsf{mix}}\iota^{2}}{T\left(1-\gamma\right)^{3}}.
\end{align*}
The above bounds taken collectively demonstrate that
\[
\Lambda_{K}\lesssim\sqrt{\frac{C^{\star}S\iota}{T\left(1-\gamma\right)^{3}}}+\frac{C^{\star}S\iota^{3}\log T}{T\left(1-\gamma\right)^{4}}+\frac{C^{\star}St_{\mathsf{mix}}\iota}{T\left(1-\gamma\right)^{2}}+\frac{C^{\star}t_{\mathsf{mix}}\iota^{2}}{T\left(1-\gamma\right)^{3}}.
\]

To finish up, combine the preceding bound with \eqref{eq:Lambda-k-final-UB-135-VR} to reach
\begin{align*}
V^{\star}\left(\rho\right)-V^{\widehat{\pi}}\left(\rho\right) & =\big\langle \rho,V^{\star}-V^{\widehat{\pi}}\big\rangle 
\overset{\text{(i)}}{\leq}\left\langle \rho,V^{\star}-V_{T_{K}}\right\rangle \overset{\text{(ii)}}{\leq}\frac{1}{T_K}\sum_{t=1}^{T_K}\left\langle \rho,V^{\star}-V_{t}\right\rangle =\Lambda_{K}\\
 & \lesssim\sqrt{\frac{C^{\star}S\iota}{T\left(1-\gamma\right)^{3}}}+\frac{C^{\star}S\iota^{3}\log T}{T\left(1-\gamma\right)^{4}}+\frac{C^{\star}St_{\mathsf{mix}}\iota}{T\left(1-\gamma\right)^{2}}+\frac{C^{\star}t_{\mathsf{mix}}\iota^{2}}{T\left(1-\gamma\right)^{3}}, 
\end{align*}
where (i) holds true according to Lemma \ref{lemma:vr-basics}, and
(ii) follows due to the monotonicity of $V_{t}$ in $t$. 
This concludes the proof of Theorem~\ref{theorem:3}.

\section{Auxiliary lemmas for Theorem \ref{theorem:3}}

\subsection{Proof of Lemma \ref{lemma:vr-basics} \label{appendix:proof-vr-basics}}

Consider any state-action pair $(s,a)\in \cS\times \cA$, and let $n=n_{t}(s,a)$. For each $1\leq i\leq T_{k}$, define
\[
k_{i}\coloneqq\min\Big\{ \big\{ 0\leq j<T_{k} \mid j>k_{i-1},\left(s_{j},a_{j}\right)=(s,a)\big\} ,\, T_{k}\Big\} ,
\]
and denote $k_{0}=0$. 
Clearly, each $k_{i}$ is a stopping time. From the update rule in
Algorithm \ref{alg:Q-3-transition}, we can write 
\[
Q_{t}\left(s,a\right)=\sum_{i=1}^{n}\eta_{i}^{n}\left[r\left(s,a\right)+\gamma V_{k_{i}}\left(s_{k_{i}+1}\right)-\gamma\overline{V}\left(s_{k_{i}+1}\right)+\gamma\widetilde{P}_{s,a}\overline{V}-b_{i}\left(s,a\right)\right].
\]
This taken together with the elementary fact $\sum_{i=1}^{n}\eta_{i}^{n}=1$ gives 
\begin{align}
\left(Q^{\star}-Q_{t}\right)\left(s,a\right) & =\left(r+\gamma PV^{\star}\right)\left(s,a\right)-\sum_{i=1}^{n}\eta_{i}^{n}\left[r\left(s,a\right)+\gamma V_{k_{i}}\left(s_{k_{i}+1}\right)-\gamma\overline{V}\left(s_{k_{i}+1}\right)+\gamma\widetilde{P}_{s,a}\overline{V}-b_{i}\left(s,a\right)\right]\nonumber \\
 & =\gamma\sum_{i=1}^{n}\eta_{i}^{n}P_{s,a}\left(V^{\star}-V_{k_{i}}\right)+\underbrace{\gamma\sum_{i=1}^{n}\eta_{i}^{n}\left(\left(P-P_{k_{i}}\right)\left(V_{k_{i}}-\overline{V}\right)\right)\left(s,a\right)}_{\eqqcolon\,\alpha_{1}}\nonumber \\
 & \quad+\underbrace{\gamma\sum_{i=1}^{n}\eta_{i}^{n}\left(\big(P-\widetilde{P}\big)\overline{V}\right)\left(s,a\right)}_{\eqqcolon\,\alpha_{2}}+\sum_{i=1}^{n}\eta_{i}^{n}b_{i}(s,a).\label{eq:Q-error-decom-vr}
\end{align}
From the update rules in Algorithms \ref{alg:Q-3-transition}-\ref{alg:Q-3-inner}
as well as Lemma \ref{lemma:step-size}, we see that 
\begin{align}
\sum_{i=1}^{n}\eta_{n}^{i}b_{i}(s,a) & \in\left[\widetilde{\beta}_{n}\left(s,a\right),2\widetilde{\beta}_{n}\left(s,a\right)\right], \label{eq:sum-b-vr}
\end{align}
where 
\begin{align*}
\widetilde{\beta}_{n}\left(s,a\right) & \coloneqq C_{\mathsf{b}}\sqrt{\frac{H\iota}{n}\left\{ \sigma_{n}^{\mathsf{adv}}\big(s,\pi^{\star}\left(s\right)\big)-\left[\mu_{n}^{\mathsf{adv}}\big(s,\pi^{\star}\left(s\right)\big)\right]^{2}\right\} }+C_{\mathsf{b}}\frac{H^{3/4}\iota^{3/4}}{n^{3/4}\left(1-\gamma\right)}+C_{\mathsf{b}}\frac{H\iota}{n\left(1-\gamma\right)}\\
 & \quad+C_{\mathsf{b}}\sqrt{\frac{\iota}{n^{\mathsf{ref}}\left(s,a\right)}\left\{ \sigma^{\mathsf{ref}}\left(s,a\right)-\left[\mu^{\mathsf{ref}}\left(s,a\right)\right]^{2}\right\} }+C_{\mathsf{b}}\frac{\iota^{3/4}}{\left(1-\gamma\right)\left[n^{\mathsf{ref}}\left(s,a\right)\right]^{3/4}}+C_{\mathsf{b}}\frac{\iota}{\left(1-\gamma\right)n^{\mathsf{ref}}\left(s,a\right)}.
\end{align*}

From now on, we shall focus on the case where $a=\pi^{\star}(s)$. 
The terms $\alpha_{1}$ and $\alpha_{2}$ are controlled separately in the following. 
\begin{itemize}
\item Regarding $\alpha_{1}$, we first define a filtration $\{\mathcal{F}_{i}\}_{i=0}^{T_{k}-1}$
as 
\[
\mathcal{F}_{i}\coloneqq\sigma\Big\{ \left\{ \big(s_{j}^{k},a_{j}^{k}\big):1\leq j\leq k_{i}\right\} ,\cup_{j=1}^{k}\mathcal{D}_{j}^{\mathsf{ref}},\cup_{j=1}^{k-1}\mathcal{D}_{j}\Big\} .
\]
Here, $(s_i^j,a_i^j)$ (resp.~$(s_i^{j,\mathsf{ref}},a_i^{j,\mathsf{ref}})$) is defined to be the $i$-th state-action pair used to  update the Q-function estimate (resp.~construct the empirical transition kernel) within the $j$-th epoch;
%, for all $j\in[K]$ and $0\leq i<T_k$ (resp.~$0\leq i<T_k^{\mathsf{ref}}$), 
and we set
\begin{align}
\mathcal{D}_{j}^{\mathsf{ref}} & \coloneqq\left\{ \big(s_{i}^{j,\mathsf{ref}},a_{i}^{j,\mathsf{ref}}\big):0\leq i<T_{j}^{\mathsf{ref}}\right\} ,\qquad\mathcal{D}_{j}\coloneqq\left\{ \big(s_{i}^{j},a_{i}^{j}\big):0\leq i<T_{j}\right\} .\label{eq:D-definition}
\end{align}
It is straightforward to check that for any $1\leq\tau\leq T$, 
\[
\Big\{ \ind_{k_{i}<T}\left(\left(P-P_{k_{i}}\right)\left(V_{k_{i}}-\overline{V}\right)\right)\big(s,\pi^{\star}(s)\big)\Big\} _{i=1}^{\tau}
\]
is a martingale difference sequence with respect to $\{\mathcal{F}_{i}\}_{i\geq0}$.
Then we can invoke the Freedman inequality to obtain: for
any fixed $s\in\mathcal{S}$ and $\tau\in[T]$, with probability exceeding
$1-\delta/(ST)$, 
\begin{align*}
\left|\sum_{i=1}^{\tau}\ind_{k_{i}\leq T_{k}}\eta_{i}^{\tau}\left(\left(P-P_{k_{i}}\right)\left(V_{k_{i}}-\overline{V}\right)\right)\big(s,\pi^{\star}(s)\big)\right| & \lesssim\sqrt{\sum_{i=1}^{\tau}\left(\eta_{i}^{\tau}\right)^{2}\mathsf{Var}_{s,\pi^{\star}(s)}\left(V_{k_{i}}-\overline{V}\right)\iota}+\frac{1}{1-\gamma}\max_{1\leq i\leq\tau}\eta_{i}^{\tau}\iota\\
 & \lesssim\sqrt{\frac{H\iota}{\tau}\sum_{i=1}^{\tau}\eta_{i}^{\tau}\mathsf{Var}_{s,\pi^{\star}(s)}\left(V_{k_{i}}-\overline{V}\right)}+\frac{H\iota}{\left(1-\gamma\right)\tau}.
\end{align*}
Invoke the union bound to show that with probability at least $1-\delta$, the above inequality holds simultaneously for all $\tau\in[T]$
and $s\in\mathcal{S}$. Replacing $\tau$ with $n=n_{t}(s,\pi^{\star}(s))$ yields that, with
probability exceeding $1-\delta$,
\begin{equation}
\left|\alpha_{1}\right|\lesssim\sqrt{\frac{H\tau}{n}\sum_{i=1}^{n}\eta_{i}^{n}\mathsf{Var}_{s,\pi^{\star}(s)}\left(V_{k_{i}}-\overline{V}\right)}+\frac{H\tau}{\left(1-\gamma\right)n}\label{eq:bernstein}
\end{equation}
holds for all $s\in\mathcal{S}$ and $t\in[T_{k}]$, where
$n=n_{t}(s,\pi^{\star}(s))$. In addition, the update rules in Algorithm \ref{alg:Q-3} tell us
that 
\begin{align}
	\mu_{n}^{\mathsf{adv}}\left(s,a\right) &=\sum_{i=1}^{n}\eta_{i}^{n}\left[V_{k_{i}}\left(s_{k_{i}+1}\right)-\overline{V}\left(s_{k_{i}+1}\right)\right]=\sum_{i=1}^{n}\eta_{i}^{n}\Big(P_{k_{i}}\big(V_{k_{i}}-\overline{V}\big)\Big)\left(s,a\right) ; \\
	\sigma_{n}^{\mathsf{adv}}\left(s,a\right) &=\sum_{i=1}^{n}\eta_{i}^{n}\left[V_{k_{i}}\left(s_{k_{i}+1}\right)-\overline{V}\left(s_{k_{i}+1}\right)\right]^{2}=\sum_{i=1}^{n}\eta_{i}^{n}\left(P_{k_{i}}\left(V_{k_{i}}-\overline{V}\right)^{2}\right)\left(s,a\right).\label{eq:sigma-n-adv-expression}
\end{align}
Recognizing that 
\[
\sum_{i=1}^{n}\eta_{i}^{n}\mathsf{Var}_{s,a}\left(V_{k_{i}}-\overline{V}\right)=\sum_{i=1}^{n}\eta_{i}^{n}P_{s,a}\left(V_{k_{i}}-\overline{V}\right)^{2}-\sum_{i=1}^{n}\eta_{i}^{n}\left[P_{s,a}\left(V_{k_{i}}-\overline{V}\right)\right]^{2}, 
\]
we can connect $\mathsf{Var}_{s,a}\left(V_{k_{i}}-\overline{V}\right)$ with $\mu_{n}^{\mathsf{adv}}$ and $\sigma_{n}^{\mathsf{adv}}$ as follows
\begin{align*}
	& \sum_{i=1}^{n}\eta_{i}^{n}\mathsf{Var}_{s,\pi^{\star}(s)}\left(V_{k_{i}}-\overline{V}\right)-\left\{ \sigma_{n}^{\mathsf{adv}}\big(s,\pi^{\star}(s)\big)-\left[\mu_{n}^{\mathsf{adv}}\big(s,\pi^{\star}(s)\big)\right]^{2}\right\} \\
 & =\sum_{i=1}^{n}\eta_{i}^{n}P_{s,\pi^{\star}(s)}\left(V_{k_{i}}-\overline{V}\right)^{2}-\sum_{i=1}^{n}\eta_{i}^{n}\left[\left(P\left(V_{k_{i}}-\overline{V}\right)\right)\big(s,\pi^{\star}(s)\big)\right]^{2}\\
 & \quad\quad-\sum_{i=1}^{n}\eta_{i}^{n}\left(P_{k_{i}}\left(V_{k_{i}}-\overline{V}\right)^{2}\right)\big(s,\pi^{\star}(s)\big)+\left[\sum_{i=1}^{n}\eta_{i}^{n}\left(P_{k_{i}}\left(V_{k_{i}}-\overline{V}\right)\right)\big(s,\pi^{\star}(s)\big)\right]^{2}\\
	& =\underbrace{\sum_{i=1}^{n}\eta_{i}^{n}\left(\left(P-P_{k_{i}}\right)\left(V_{k_{i}}-\overline{V}\right)^{2}\right)\big(s,\pi^{\star}(s)\big)}_{\eqqcolon\,\alpha_{1,1}}+\underbrace{\left[\sum_{i=1}^{n}\eta_{i}^{n}\left(P_{k_{i}}\left(V_{k_{i}}-\overline{V}\right)\right)\big(s,\pi^{\star}(s)\big)\right]^{2}-\sum_{i=1}^{n}\eta_{i}^{n}\left[P_{s,\pi^{\star}(s)}\left(V_{k_{i}}-\overline{V}\right)\right]^{2}}_{\eqqcolon\,\alpha_{1,2}},
\end{align*}%
leaving us with two terms to control. 
\begin{itemize}
	\item
The first term $\alpha_{1,1}$ can be bounded by the Azuma-Hoeffding
inequality. We can employ similar arguments as used when proving (\ref{eq:hoeffding})
and invoke the Azuma-Hoeffding inequality to show that: with probability
exceeding $1-\delta/S$ , for all $s\in\mathcal{S}$ and $t\in[T_{k}]$,
it holds that
\begin{align}
	|\alpha_{1,1}| & \lesssim\sqrt{\frac{H\iota}{n\left(1-\gamma\right)^{4}}}.
	\label{eq:alpha-11-bound-Lemma6-VR}
\end{align}

\item 
Moving on to the second term $\alpha_{1,2}$, we invoke the identity $\sum_{i=1}^{n}\eta_{i}^{n}=1$ to deduce that
\begin{align*}
\alpha_{1,2} & =\left[\sum_{i=1}^{n}\eta_{i}^{n}\left(P_{k_{i}}\left(V_{k_{i}}-\overline{V}\right)\right)\big(s,\pi^{\star}(s)\big)\right]^{2}-\left(\sum_{i=1}^{n}\eta_{i}^{n}\right)\sum_{i=1}^{n}\eta_{i}^{n}\left[P_{s,\pi^{\star}(s)}\left(V_{k_{i}}-\overline{V}\right)\right]^{2}\\
 & \overset{\text{(i)}}{\leq}\left[\sum_{i=1}^{n}\eta_{i}^{n}\left(P_{k_{i}}\left(V_{k_{i}}-\overline{V}\right)\right)\big(s,\pi^{\star}(s)\big)\right]^{2}-\left[\sum_{i=1}^{n}\eta_{i}^{n}P_{s,\pi^{\star}(s)}\left(V_{k_{i}}-\overline{V}\right)\right]^{2}\\
 & =\left[\sum_{i=1}^{n}\eta_{i}^{n}\left(\left(P_{k_{i}}-P\right)\left(V_{k_{i}}-\overline{V}\right)\right)\big(s,\pi^{\star}(s)\big)\right]\left[\sum_{i=1}^{n}\eta_{i}^{n}\left(\left(P_{k_{i}}+P\right)\left(V_{k_{i}}-\overline{V}\right)\right)\big(s,\pi^{\star}(s)\big)\right]\\
 & \overset{\text{(ii)}}{\leq}\frac{1}{1-\gamma} 
	\left| \sum_{i=1}^{n}\eta_{i}^{n}\left(\left(P_{k_{i}}-P\right)\left(V_{k_{i}}-\overline{V}\right)\right)\big(s,\pi^{\star}(s)\big) \right| \\
 & \overset{\text{(iii)}}{\lesssim}\sqrt{\frac{H\iota}{n\left(1-\gamma\right)^{4}}}.
\end{align*}
Here, (i) arises from the Cauchy-Schwarz inequality; (ii) follows from the fact that
$0\leq V_{k_{i}}-\overline{V}\leq1/(1-\gamma)$ and the identity $\sum_{i=1}^{n}\eta_{i}^{n}=1$; and (iii) follows
by repeating the argument used to establish (\ref{eq:hoeffding}) and invoking the
Azuma-Hoeffding inequality (which we omite here for the sake of brevity). 
\end{itemize}
With the preceding bounds in place, we conclude that with probability
exceeding $1-O(\delta)$, 
\begin{align*}
\sum_{i=1}^{n}\eta_{n}^{i}\mathsf{Var}_{s,\pi^{\star}(s)}\left(V_{k_{i}}-\overline{V}\right) & \leq\sigma_{n}^{\mathsf{adv}}\big(s,\pi^{\star}(s)\big)-\left[\mu_{n}^{\mathsf{adv}}\big(s,\pi^{\star}(s)\big)\right]^{2}+\alpha_{1,1}+\alpha_{1,2}\\
 & \leq\sigma_{n}^{\mathsf{adv}}\big(s,\pi^{\star}(s)\big)-\left[\mu_{n}^{\mathsf{adv}}\big(s,\pi^{\star}(s)\big)\right]^{2}+O\left(\sqrt{\frac{H\iota}{n\left(1-\gamma\right)^{4}}}\right)
\end{align*}
holds for all $s\in\mathcal{S}$ and $t\in[T_{k}]$. 
Putting the above results together 
and using the fact $\sigma_{n}^{\mathsf{adv}}\big(s,\pi^{\star}(s)\big)\geq \left[\mu_{n}^{\mathsf{adv}}\big(s,\pi^{\star}(s)\big)\right]^{2}$ (due to Jensen's inequality) 
reveal that with probability exceeding $1-O(\delta)$, 
\begin{align*}
\left|\alpha_{1}\right| & \lesssim\sqrt{\frac{H\iota}{n}\sum_{i=1}^{n}\eta_{i}^{n}\mathsf{Var}_{s,\pi^{\star}(s)}\left(V_{k_{i}}-\overline{V}\right)}+\frac{H\iota}{\left(1-\gamma\right)n}\\
	& \lesssim\sqrt{\frac{H\iota}{n}\left\{ \sigma_{n}^{\mathsf{adv}}\big(s,\pi^{\star}\left(s\right)\big)-\left[\mu_{n}^{\mathsf{adv}}\big(s,\pi^{\star}(s)\big)\right]^{2}+O\left(\sqrt{\frac{H\iota}{n\left(1-\gamma\right)^{4}}}\right)\right\} }+\frac{H\iota}{n\left(1-\gamma\right)}\\
 & \lesssim\sqrt{\frac{H\iota}{n}\left\{ \sigma_{n}^{\mathsf{adv}}\big(s,\pi^{\star}\left(s\right)\big)-\left[\mu_{n}^{\mathsf{adv}}\big(s,\pi^{\star}(s)\big)\right]^{2}\right\} }+\frac{H^{3/4}\iota^{3/4}}{n^{3/4}\left(1-\gamma\right)}+\frac{H\iota}{n\left(1-\gamma\right)} 
\end{align*}
holds for all $s\in\mathcal{S}$
and $t\in[T_{k}]$.

\item Regarding $\alpha_{2}$, we first recall that $n^{\mathsf{ref}}(s,a)$
denotes the number of visit to $(s,a)$ among the samples used to compute $\widetilde{P}$. 
Let $k_{0}=-1$, and for each $1\leq i\leq T_{k}^{\mathsf{ref}}$, define
\[
k_{i}\coloneqq\min\Big\{ \left\{ 0\leq k<T_{k}^{\mathsf{ref}} \mid k>k_{i-1},\left(s_{k},a_{k}\right)=\left(s,a\right)\right\} ,T_{k}^{\mathsf{ref}}\Big\} .
\]
Akin to how we establish (\ref{eq:bernstein}), we can use the Freedman
inequality to show that: for any fixed $s\in\mathcal{S}$, with probability
exceeding $1-\delta/S$, 
\begin{align*}
	\left|\alpha_{2}\right| & =\left|\gamma\left(\big(P-\widetilde{P}\big)\overline{V}\right)\big(s,\pi^{\star}(s)\big)\right|\\
	& =\gamma\left|\frac{1}{n^{\mathsf{ref}}\big(s,\pi^{\star}(s)\big)}\sum_{i=0}^{T_{k}^{\mathsf{ref}}}\left(P-P_{i}^{\mathsf{ref}}\right)_{s,a}\overline{V}\ind\big\{(s_{i}^{\mathsf{ref}},a_{i}^{\mathsf{ref}})=\big(s,\pi^{\star}(s)\big)\big\}\right|\\
	& =\gamma\left|\frac{1}{n^{\mathsf{ref}}\big(s,\pi^{\star}(s)\big)}\sum_{i=0}^{n^{\mathsf{ref}}\left(s,\pi^\star(s)\right)}\left(P-P_{k_{i}}^{\mathsf{ref}}\right)_{s,\pi^{\star}(s)}\overline{V}\right|\\
	& \lesssim\sqrt{\frac{\mathsf{Var}_{s,\pi^{\star}(s)}(\overline{V})}{n^{\mathsf{ref}}\big(s,\pi^{\star}(s)\big)}\iota}+\frac{\iota}{\left(1-\gamma\right)n^{\mathsf{ref}}\big(s,\pi^{\star}\left(s\right)\big)}.
\end{align*}
Here, the first line results from the identity $\sum_{i=1}^{n}\eta_{i}^{n}=1$.
It follows from the update rule in Algorithm \ref{alg:Q-3} that 
\begin{align}
	\mu^{\mathsf{ref}}\left(s,a\right)&=\frac{1}{n^{\mathsf{ref}}\left(s,a\right)}\sum_{i=1}^{n^{\mathsf{ref}}\left(s,a\right)}\overline{V}\left(s_{k_{i}+1}\right)=\frac{1}{n^{\mathsf{ref}}\left(s,a\right)}\sum_{i=1}^{n^{\mathsf{ref}}\left(s,a\right)}\left(P_{k_{i}}\overline{V}\right)\left(s,a\right) , \\
	\sigma^{\mathsf{ref}}\left(s,a\right) &=\frac{1}{n^{\mathsf{ref}}\left(s,a\right)}\sum_{i=1}^{n^{\mathsf{ref}}\left(s,a\right)}\overline{V}^{2}\left(s_{k_{i}+1}\right)=\frac{1}{n^{\mathsf{ref}}\left(s,a\right)}\sum_{i=1}^{n^{\mathsf{ref}}\left(s,a\right)}\big(P_{k_{i}}\overline{V}^{2}\big)\left(s,a\right),
\end{align}
allowing us to deduce that
\begin{align*}
 & \left|\mathsf{Var}_{s,\pi^{\star}(s)}(\overline{V})-\sigma^{\mathsf{ref}}\big(s,\pi^{\star}(s)\big)+\left[\mu^{\mathsf{ref}}\big(s,\pi^{\star}(s)\big)\right]^{2}\right|\\
 & \quad=\left|P_{s,\pi^{\star}(s)}(\overline{V}^{2})-(P_{s,\pi^{\star}(s)}\overline{V})^{2}-\sigma^{\mathsf{ref}}\big(s,\pi^{\star}(s)\big)+\left[\mu^{\mathsf{ref}}\big(s,\pi^{\star}(s)\big)\right]^{2}\right|\\
	& \quad\leq\underbrace{\left|P_{s,\pi^{\star}(s)}(\overline{V}^{2})-\frac{1}{n^{\mathsf{ref}}\big(s,\pi^{\star}\left(s\right)\big)}\sum_{i=1}^{n^{\mathsf{ref}}\left(s,\pi^{\star}(s)\right)}\left(P_{k_{i}}\overline{V}^{2}\right)\big(s,\pi^{\star}(s)\big)\right|}_{\eqqcolon\alpha_{2,1}}\\
 &\quad\quad+\underbrace{\left|\left[\frac{1}{n^{\mathsf{ref}}\big(s,\pi^{\star}(s)\big)}\sum_{i=1}^{n^{\mathsf{ref}}\left(s,\pi^{\star}(s)\right)}\left(P_{k_{i}}\overline{V}\right)\big(s,\pi^{\star}(s)\big)\right]^{2}-\left(P_{s,\pi^{\star}(s)}\overline{V}\right)^{2}\right|}_{\eqqcolon\alpha_{2,2}}.
\end{align*}
Using the similar argument in proving (\ref{eq:hoeffding}) and the
Azuma-Hoeffding inequality, we can show that with probability exceeding $1-\delta/S$, 
\[
\alpha_{2,1}\lesssim\frac{1}{\left(1-\gamma\right)^{2}}\sqrt{\frac{\iota}{n^{\mathsf{ref}}\left(s,a\right)}}.
\]
The second term $\alpha_{2,2}$ can be bounded by 
\begin{align*}
	\alpha_{2,2} & =\left|\left[\frac{1}{n^{\mathsf{ref}}(s,\pi^{\star}\left(s\right))}\sum_{i=1}^{n^{\mathsf{ref}}(s,\pi^{\star}(s))}\left(\left(P_{k_{i}}-P\right)\overline{V}\right)\big(s,\pi^{\star}(s)\big)\right]\left[\frac{1}{n^{\mathsf{ref}}\big(s,\pi^{\star}(s)\big)}\sum_{i=1}^{n^{\mathsf{ref}}(s,\pi^{\star}(s))}\left(\left(P_{k_{i}}+P\right)\overline{V}\right)\big(s,\pi^{\star}(s)\big)\right]\right|\\
	& \leq\frac{2}{1-\gamma}\left|\frac{1}{n^{\mathsf{ref}}\big(s,\pi^{\star}(s)\big)}\sum_{i=1}^{n^{\mathsf{ref}}(s,\pi^{\star}(s))}\left(\left(P_{k_{i}}-P\right)\overline{V}\right)\big(s,\pi^{\star}(s)\big)\right|\\
 & \lesssim\frac{1}{\left(1-\gamma\right)^{2}}\sqrt{\frac{\iota}{n^{\mathsf{ref}}\big(s,\pi^{\star}(s)\big)}}.
\end{align*}
Here, the penultimate line follows from the fact that $0\leq\overline{V}(s)\leq1/(1-\gamma)$ for all $s\in \cS$,
whereas the last line can be proved by using the similar argument used to establish
(\ref{eq:hoeffding}) and invoking the Azuma-Hoeffding inequality. These bounds taken collectively allow us to derive
\begin{align}
\mathsf{Var}_{s,\pi^{\star}(s)}(\overline{V}) & =\sigma^{\mathsf{ref}}\big(s,\pi^{\star}(s)\big)-\left[\mu^{\mathsf{ref}}\big(s,\pi^{\star}(s)\big)\right]^{2}+O\left(\alpha_{2,1}+\alpha_{2,2}\right) \nonumber \\
 & =\sigma^{\mathsf{ref}}\big(s,\pi^{\star}(s)\big)-\left[\mu^{\mathsf{ref}}\big(s,\pi^{\star}(s)\big)\right]^{2}+O\left(\frac{1}{\left(1-\gamma\right)^{2}}\sqrt{\frac{\iota}{n^{\mathsf{ref}}\big(s,\pi^{\star}(s)\big)}}\right).\label{eq:sigma-ref-upper-bound}
\end{align}
Consequently, it is immediately seen that: with probability exceeding $1-O(\delta)$,
\begin{align*}
\left|\alpha_{2}\right| & \lesssim\sqrt{\frac{\iota}{n^{\mathsf{ref}}\big(s,\pi^{\star}(s)\big)}\left[\sigma^{\mathsf{ref}}\big(s,\pi^{\star}(s)\big)-\left[\mu^{\mathsf{ref}}\big(s,\pi^{\star}(s)\big)\right]^{2}+O\left(\frac{1}{\left(1-\gamma\right)^{2}}\sqrt{\frac{1}{n^{\mathsf{ref}}\big(s,\pi^{\star}(s)\big)}}\right)\right]}+\frac{\iota}{\left(1-\gamma\right)n^{\mathsf{ref}}\big(s,\pi^{\star}(s)\big)}\\
 & \lesssim\sqrt{\frac{\iota}{n^{\mathsf{ref}}\big(s,\pi^{\star}(s)\big)}\left\{ \sigma^{\mathsf{ref}}\big(s,\pi^{\star}(s)\big)-\left[\mu^{\mathsf{ref}}\big(s,\pi^{\star}(s)\big)\right]^{2}\right\} }+\frac{\iota^{3/4}}{\left(1-\gamma\right)\left[n^{\mathsf{ref}}\big(s,\pi^{\star}(s)\big)\right]^{3/4}}+\frac{\iota}{\left(1-\gamma\right)n^{\mathsf{ref}}\big(s,\pi^{\star}(s)\big)} 
\end{align*}
holds simultaneously for all $s\in\mathcal{S}$. 
\end{itemize}

With the above bounds on $\alpha_{1}$ and $\alpha_{2}$ in place, 
we can take these together with (\ref{eq:sum-b-vr}) to obtain
\begin{align*}
	0\leq\sum_{i=1}^{n}\eta_{i}^{n}b_{i}\big(s,\pi^{\star}(s)\big)+\alpha_{1}+\alpha_{2} & \leq3\widetilde{\beta}_{n}\big(s,\pi^{\star}(s)\big)=\beta_{n}\big(s,\pi^{\star}(s)\big), 
\end{align*}
with the proviso that $C_{\mathsf{b}}>0$ is sufficiently large. 
Substitution into (\ref{eq:Q-error-decom-vr}) then gives: with probability exceeding
$1-O(\delta)$, 
\begin{align*}
\left(Q^{\star}-Q_{t}\right)\big(s,\pi^{\star}(s)\big) & \leq\gamma\sum_{i=1}^{n}\eta_{i}^{n}P_{s,\pi^{\star}(s)}\left(V^{\star}-V_{k_{i}}\right)+\beta_{n}\big(s,\pi^{\star}(s)\big)
\end{align*}
holds for all $s\in\mathcal{S}$ and $t\in[T_{k}]$. 

The second part of the lemma --- namely, $V_{t}(s)\leq V^{\pi_{t}}(s)\leq V^{\star}(s)$
for all $s\in\mathcal{S}$ and $t\in[T_{k}]$ --- can be proved in a way similar to the proof of the second part of Lemma \ref{lemma:lcb-h-basics}. We omit it here for brevity.

\subsection{Proof of Lemma \ref{lemma:variance-bound} \label{appendix:proof-variance-bound}}

In view of (\ref{eq:sigma-n-adv-expression}), we can deduce that
\begin{align*}
\sigma_{n}^{\mathsf{adv}}\big(s,\pi^{\star}(s)\big) & =\sum_{i=1}^{n}\eta_{i}^{n}\left(P_{k_{i}}\left(V_{k_{i}}-\overline{V}\right)^{2}\right)\big(s,\pi^{\star}(s)\big)\\
 & =\sum_{i=1}^{n}\eta_{i}^{n}P_{s,\pi^{\star}(s)}\left(V_{k_{i}}-\overline{V}\right)^{2}+\sum_{i=1}^{n}\eta_{i}^{n}\left(\left(P_{k_{i}}-P\right)\left(V_{k_{i}}-\overline{V}\right)^{2}\right)\big(s,\pi^{\star}(s)\big)\\
 & \leq P_{s,\pi^{\star}(s)}\left(V^{\star}-\overline{V}\right)^{2}+\sum_{i=1}^{n}\eta_{i}^{n}\left(\left(P_{k_{i}}-P\right)\left(V_{k_{i}}-\overline{V}\right)^{2}\right)\big(s,\pi^{\star}(s)\big)\\
 & \leq P_{s,\pi^{\star}(s)}\left(V^{\star}-\overline{V}\right)^{2}+O\left(\sqrt{\frac{H\iota}{n\left(1-\gamma\right)^{4}}}\right).
\end{align*}
Here, the penultimate line follows from the fact that $V_{t}$ is non-decreasing in $t$ 
and $\overline{V}\leq V_{t}\leq V^{\star}$, while the last inequality
invokes the upper bound on $\alpha_{1,1}$ (cf.~\eqref{eq:alpha-11-bound-Lemma6-VR}) derived in Lemma \ref{lemma:vr-basics}.
In addition, we observe that
\[
\sigma^{\mathsf{ref}}\big(s,\pi^{\star}(s)\big)-\left[\mu^{\mathsf{ref}}\big(s,\pi^{\star}(s)\big)\right]^{2}=\mathsf{Var}_{s,\pi^{\star}(s)}(\overline{V})+O\left(\frac{1}{\left(1-\gamma\right)^{2}}\sqrt{\frac{\iota}{n^{\mathsf{ref}}\big(s,\pi^{\star}(s)\big)}}\right), 
\]
which follows directly from (\ref{eq:sigma-ref-upper-bound}).

Next, we turn to bounding the sum $\sum_{s,a}d_{\rho}^{\star}(s,a)\mathsf{Var}_{s,a}(\overline{V})$,
which can be decomposed into 
\begin{align}
\sum_{s,a}d_{\rho}^{\star}\left(s,a\right)\mathsf{Var}_{s,a}(\overline{V}) & =\underbrace{\sum_{s\in\mathcal{S},a\in\mathcal{A}}d_{\rho}^{\star}\left(s,a\right)\mathsf{Var}_{s,a}(V^{\star})}_{\eqqcolon\alpha_{1}}+\underbrace{\sum_{s,a}d_{\rho}^{\star}\left(s,a\right)\left[\mathsf{Var}_{s,a}(\overline{V})-\mathsf{Var}_{s,a}(V^{\star})\right]}_{\eqqcolon\alpha_{2}}.
	\label{eq:sum-d-Var-alpha-1-alpha-2-VR}
\end{align}
This leaves us with two terms $\alpha_{1}$ and $\alpha_{2}$ to control.
\begin{itemize}
	\item With regards to $\alpha_{1}$, we first define a vector $v=[v_s]_{s\in \cS}\in\mathbb{R}^{S}$ obeying
\[
	v_{s}\coloneqq\mathsf{Var}_{s,\pi^{\star}(s)}\left(V^{\star}\right) \qquad\text{for all } s\in\mathcal{S}, 
\]
which clearly satisfies
\begin{align}
v & =P_{\pi^{\star}}\left[V^{\star}\circ V^{\star}\right]-\left(P_{\pi^{\star}}V^{\star}\right)\circ\left(P_{\pi^{\star}}V^{\star}\right)\nonumber \\
 & =P_{\pi^{\star}}\left(V^{\star}\circ V^{\star}\right)-\frac{1}{\gamma^{2}}\left(r-V^{\star}\right)\circ\left(r-V^{\star}\right)\nonumber \\
 & =P_{\pi^{\star}}\left(V^{\star}\circ V^{\star}\right)-\frac{1}{\gamma^{2}}r\circ r-\frac{1}{\gamma^{2}}V^{\star}\circ V^{\star}+2V^{\star}\circ r\nonumber \\
 & \leq\frac{1}{\gamma^{2}}\left(\gamma^{2}P_{\pi^{\star}}-I\right)\left(V^{\star}\circ V^{\star}\right)+2V^{\star}\circ r.
	\label{eq:v-bound}
\end{align}
%
%where for any vectors $u,v$ with the same dimension, $u\circ v$
%is the vector of the same dimension that satisies $(u\circ v)_{i}=u_{i}v_{i}$.
Here,  the second identity follows from the Bellman optimality equation
$V^{\star}=r+\gamma P_{\pi^{\star}}V^{\star}$. Recognizing that $d_{\rho}^{\star}(s,a)=d_{\rho}^{\star}(s)\ind\{a=\pi^{\star}(s)\}$
and $d_{\rho}^{\star}= (1-\gamma) \rho (I-\gamma P_{\pi^{\star}})^{-1}$, we obtain
\begin{align*}
	\alpha_{1} & =\sum_{s\in\mathcal{S}}d_{\rho}^{\star}\left(s\right)\mathsf{Var}_{s,\pi^{\star}(s)}(V^{\star})=\left\langle d_{\rho}^{\star},v\right\rangle =\left(1-\gamma\right)\rho\left(I-\gamma P_{\pi^{\star}}\right)^{-1}v \\
 & \leq \left(1-\gamma\right) \|\rho\|_1 \left\Vert \left(I-\gamma P_{\pi^{\star}}\right)^{-1}v\right\Vert _{\infty} 
   = \left(1-\gamma\right)\left\Vert \left(I-\gamma P_{\pi^{\star}}\right)^{-1}v\right\Vert _{\infty}\\
 & \overset{\text{(i)}}{\leq}\frac{1-\gamma}{\gamma^{2}}\left\Vert \left(I-\gamma P_{\pi^{\star}}\right)^{-1}\left(\gamma^{2}P_{\pi^{\star}}-I\right)\left(V^{\star}\circ V^{\star}\right)\right\Vert _{\infty}+ 2\left(1-\gamma\right)\left\Vert \left(I-\gamma P_{\pi^{\star}}\right)^{-1}\left(V^{\star}\circ r\right)\right\Vert _{\infty}\\
 & =\frac{1-\gamma}{\gamma^{2}}\left\Vert \left(I-\gamma P_{\pi^{\star}}\right)^{-1}\left[\left(1-\gamma\right)I+\gamma\left(I-\gamma P_{\pi^{\star}}\right)\right]\left(V^{\star}\circ V^{\star}\right)\right\Vert _{\infty}+ 2\left(1-\gamma\right)\left\Vert V^{\star}\right\Vert _{\infty}\left\Vert \left(I-\gamma P_{\pi^{\star}}\right)^{-1}r\right\Vert _{\infty}\\
 & \overset{\text{(ii)}}{\leq}\frac{\left(1-\gamma\right)^{2}}{\gamma^{2}}\left\Vert \left(I-\gamma P_{\pi^{\star}}\right)^{-1}\left(V^{\star}\circ V^{\star}\right)\right\Vert _{\infty}+\frac{1-\gamma}{\gamma}\left\Vert V^{\star}\circ V^{\star}\right\Vert _{\infty}+ 2\left(1-\gamma\right)\left\Vert V^{\star}\right\Vert _{\infty}^2\\
 & \leq\frac{\left(1-\gamma\right)^{2}}{\gamma^{2}}\frac{1}{1-\gamma}\left\Vert V^{\star}\right\Vert _{\infty}^{2}+\frac{1-\gamma}{\gamma}\left\Vert V^{\star}\right\Vert _{\infty}^{2}+2\left(1-\gamma\right)\left\Vert V^{\star}\right\Vert _{\infty}^{2}\\
 & \overset{\text{(iii)}}{\leq}\frac{8}{1-\gamma}.
\end{align*}
Here, (i) follows from (\ref{eq:v-bound}); (ii) relies on the triangle
inequality as well as the Bellman optimality equation $V^{\star}=(I-\gamma P_{\pi^{\star}})^{-1}r$;
and (iii) arises from the property $0\leq V^{\star}(s)\leq1/(1-\gamma)$ for all $s\in \cS$ as well as the
assumption that $\gamma\geq1/2$.

\item Regarding $\alpha_{2}$, we make the observation that
\begin{align}
	\alpha_{2} & =\sum_{s\in\mathcal{S},a\in\mathcal{A}}d_{\rho}^{\star}\left(s,a\right)\left\{ P_{s,a}\overline{V}^{2}-\left[P_{s,a}\overline{V}\right]^{2}-P_{s,a}\big(V^{\star2}\big)+\left[P_{s,a}V^{\star}\right]^{2}\right\} \notag\\
	& =\sum_{s\in\mathcal{S},a\in\mathcal{A}}d_{\rho}^{\star}\left(s,a\right)\left\{ P_{s,a}\overline{V}^{2}-P_{s,a} \big(V^{\star2}\big) \right\} +\sum_{s,a}d_{\rho}^{\star}\left(s,a\right)\left\{ \left[P_{s,a}V^{\star}\right]^{2}-\left[P_{s,a}\overline{V}\right]^{2}\right\} \notag\\
 & \leq
	%=\sum_{s\in\mathcal{S},a\in\mathcal{A}}d_{\rho}^{\star}\left(s,a\right)P_{s,a}\left[\left(\overline{V}-V^{\star}\right)\left(\overline{V}+V^{\star}\right)\right] +
	\sum_{s,a}d_{\rho}^{\star}\left(s,a\right)P_{s,a}\left(V^{\star}-\overline{V}\right)P_{s,a}\left(V^{\star}+\overline{V}\right) \notag\\
 & \leq\frac{2}{1-\gamma}\sum_{s\in\mathcal{S},a\in\mathcal{A}}d_{\rho}^{\star}\left(s,a\right)P_{s,a}\left(V^{\star}-\overline{V}\right), 
	\label{eq:alpha2-UB-crude-VR}
\end{align}
where the third line holds since $\overline{V}^2\leq V^{\star2}$, and the last line is valid since $\|P_{s,a}\|_1=1$ and $\|\overline{V}\|_{\infty}\leq \|V^{\star}\|_{\infty}\leq \frac{1}{1-\gamma}$.  
Recognizing that 
\begin{equation}
	d_{\rho}^{\star} = (1-\gamma) \rho + \gamma d_{\rho}^{\star} P_{\pi^{\star}} ,
	\label{eq:d-rho-star-equilibrium-VR}
		%\left(s\right)=\sum_{a\in\mathcal{A}}d_{\rho}^{\star}\left(s,a\right)=\left(1-\gamma\right)\rho\left(s\right)+\gamma\sum_{s'\in\mathcal{S},a'\in\mathcal{A}}P\left(s\mymid s',a'\right)d_{\rho}^{\star}\left(s',a'\right),
\end{equation}
we can use the fact $d_{\rho}^{\star}\left(s,a\right)=d_{\rho}^{\star}\left(s\right) \ind\{a=\pi^{\star}(s)\}$ to derive
\begin{align*}
	\sum_{s\in\mathcal{S},a\in\mathcal{A}}d_{\rho}^{\star}\left(s,a\right)P_{s,a}\big( V^{\star}-\overline{V} \big) 
	& =\sum_{s\in \cS} d_{\rho}^{\star}(s) P_{s,\pi^{\star}(s)}\big( V^{\star}-\overline{V} \big) 
	 = \big\langle  d_{\rho}^{\star} P_{\pi^{\star}} ,  V^{\star}-\overline{V} \big\rangle \\ 
	& = \bigg\langle \frac{d_{\rho}^{\star}-\left(1-\gamma\right)\rho}{\gamma}, V^{\star}-\overline{V} \bigg\rangle\\
 	& =\left\langle \widetilde{\rho},V^{\star}-\overline{V}\right\rangle =\Delta_{k-1}. 
\end{align*}
Substitution into \eqref{eq:alpha2-UB-crude-VR} leads to
\[
	\alpha_{2}\leq\frac{2}{1-\gamma}\Delta_{k-1}.
\]

\item Take the preceding bounds on $\alpha_{1}$ and $\alpha_{2}$ together with \eqref{eq:sum-d-Var-alpha-1-alpha-2-VR} to
yield
\begin{align*}
\sum_{s,a}d_{\rho}^{\star}\left(s,a\right)\mathsf{Var}_{s,a}(\overline{V}) & \leq\alpha_{1}+\alpha_{2}\leq\frac{8}{1-\gamma}+\frac{2}{1-\gamma}\Delta_{k-1}.
\end{align*}

\end{itemize}

Finally, we turn attention to $\sum_{s,a}d_{\rho}^{\star}\left(s,a\right)\mathsf{Var}_{s,a}(V^{\star}-\overline{V})$. 
This sum can be bounded as follows
\begin{align*}
\sum_{s,a}d_{\rho}^{\star}\left(s,a\right)\mathsf{Var}_{s,a}(V^{\star}-\overline{V}) & \leq\sum_{s\in\mathcal{S},a\in\mathcal{A}}d^{\star}\left(s,a\right)P_{s,a}\left(V^{\star}-\overline{V}\right)^{2}
	 = \sum_{s\in\mathcal{S} }d^{\star}\left(s\right)P_{s,\pi^{\star}(s)}\left(V^{\star}-\overline{V}\right)^{2}  \\
	& = \bigg\langle \frac{d_{\rho}^{\star}-\left(1-\gamma\right)\rho}{\gamma},  \big(V^{\star}-\overline{V} \big)^2 \bigg\rangle\\ 	
 & =\sum_{s\in\mathcal{S}}\frac{d^{\star}\left(s\right)-\left(1-\gamma\right)\rho\left(s\right)}{\gamma}\left(V^{\star}-\overline{V}\right)^{2}\left(s\right)\\
	& = \big\langle \widetilde{\rho}, \big(V^{\star}-\overline{V}\big)^{2} \big\rangle 
	\leq  \big\| V^{\star}-\overline{V}\big\|_{\infty}   \big\langle \widetilde{\rho}, V^{\star}-\overline{V} \big\rangle \\
	& \leq\frac{1}{1-\gamma}\Delta_{k-1}, 
\end{align*}
where the first identity holds since $d^{\star}\left(s,a\right)=d^{\star}\left(s\right) \ind \{a=\pi^{\star}(s)\}$, 
and the second line invokes \eqref{eq:d-rho-star-equilibrium-VR}.

\subsection{Proof of Lemma \ref{lemma:d-stat-star}\label{appendix:proof-lemma-d-stat-star}}
Recall that 
\[
	\widetilde{\rho}\coloneqq\frac{d_{\rho}^{\star}-\left(1-\gamma\right)\rho}{\gamma} ,
\]
which is clearly also a probability vector. To prove the lemma, we find it helpful to introduce the following occupancy distribution induced by $\widetilde{\rho}$:
\begin{align*}
d_{\widetilde{\rho}}^{\star}\left(s\right) & \coloneqq\left(1-\gamma\right)\sum_{t=0}^{\infty}\gamma^{t}\mathbb{P}\left(s_{t}=s\mid\pi^{\star},s_{0}\sim\widetilde{\rho}\right).
\end{align*}
Repeating the argument used to establish (\ref{eq:useful}), we can easily see that: for any vector $V\in\mathbb{R}^{d}$
with non-negative entries, it holds that
\begin{align}
\sum_{j=0}^{\infty}\left[\gamma\left(1+\frac{1}{H}\right)^{3}\right]^{j}\left\langle \widetilde{\rho}P_{\pi^{\star}}^{j},V\right\rangle \lesssim\frac{1}{1-\gamma}\left\langle d_{\widetilde{\rho}}^{\star},V\right\rangle +\frac{\delta}{ST^{4}\left(1-\gamma\right)}\left\Vert V\right\Vert _{\infty}.
	\label{eq:sum-lemma-d-star-star-VF}
\end{align}
Consequently, it boils down to analyzing the distribution $d_{\widetilde{\rho}}^{\star}$.

For any integer $K\geq0$, employ the identity $d_{\rho}^{\star}=\rho \sum_{i=0}^{\infty} \gamma^i (P_{\pi^{\star}})^i$ to deduce that
\begin{align*}
d_{\widetilde{\rho}}^{\star}\left(s\right) & =\left(1-\gamma\right)\left[\sum_{i=0}^{\infty}\gamma^{i}\frac{d_{\rho}^{\star}-\left(1-\gamma\right)\rho}{\gamma}(P_{\pi^{\star}})^{i}\right](s)\\
 & \overset{\text{(i)}}{=} \frac{\left(1-\gamma\right)^{2}}{\gamma}\left[\sum_{i=0}^{\infty}\sum_{j=0}^{\infty}\gamma^{i+j}\rho(P_{\pi^{\star}})^{i+j}\right](s)-\frac{\left(1-\gamma\right)^{2}}{\gamma}\left[\sum_{i=0}^{\infty}\gamma^{i}\rho(P_{\pi^{\star}})^{i}\right](s)\\
 & =\frac{\left(1-\gamma\right)^{2}}{\gamma}\left[\rho\left(\sum_{l=0}^{\infty}\gamma^{l}(P_{\pi^{\star}})^{l}\right)^{2}\right](s)-\frac{\left(1-\gamma\right)^{2}}{\gamma}\left[\sum_{i=0}^{\infty}\gamma^{i}\rho(P_{\pi^{\star}})^{i}\right](s)\\
 & =\frac{\left(1-\gamma\right)^{2}}{\gamma}\left[\sum_{i=0}^{\infty}\left(i+1\right)\gamma^{i}\rho(P_{\pi^{\star}})^{i}\right](s)-\frac{\left(1-\gamma\right)^{2}}{\gamma}\left[\sum_{i=0}^{\infty}\gamma^{i}\rho(P_{\pi^{\star}})^{i}\right](s)\\
 & =\frac{\left(1-\gamma\right)^{2}}{\gamma}\left[\sum_{i=0}^{\infty}i\gamma^{i}\rho(P_{\pi^{\star}})^{i}\right](s)\\
 & =\frac{\left(1-\gamma\right)^{2}}{\gamma}\left[\sum_{i=0}^{K-1}i\gamma^{i}\rho(P_{\pi^{\star}})^{i}\right](s)+\frac{\left(1-\gamma\right)^{2}}{\gamma}\left[\sum_{i=K}^{\infty}i\gamma^{i}\rho(P_{\pi^{\star}})^{i}\right](s)\\
 & \leq K\frac{\left(1-\gamma\right)^{2}}{\gamma}\left[\sum_{i=0}^{K-1}\gamma^{i}\rho(P_{\pi^{\star}})^{i}\right](s)+\frac{\left(1-\gamma\right)^{2}}{\gamma}\left[\sum_{i=K}^{\infty}i\gamma^{i}\rho(P_{\pi^{\star}})^{i}\right](s)\\
 & \overset{\text{(ii)}}{\leq}2K\left(1-\gamma\right)d_{\rho}^{\star}(s)+\underbrace{\frac{\left(1-\gamma\right)^{2}}{\gamma}\left[\sum_{i=K}^{\infty}i\gamma^{i}\rho(P_{\pi^{\star}})^{i}\right](s)}_{\eqqcolon\, e(s)}.
\end{align*}
Here, (i) and (ii) make use of the identity $d_{\rho}^{\star}=(1-\gamma)\rho\sum_{j=0}^{\infty}\gamma^{j}(P_{\pi^{\star}})^{j}$
and the assumption that $\gamma\geq1/2$. By choosing 
\[
K\coloneqq\left\lceil \frac{C_{K}}{1-\gamma}\log\frac{ST}{\delta}\right\rceil 
\]
for some constant $C_{K}>0$, we can guarantee that 
\begin{equation}
d_{\widetilde{\rho}}^{\star}\left(s\right)\leq 4C_{K}d_{\rho}^{\star}(s) \log\frac{ST}{\delta} +e\left(s\right).
\label{eq:d-star-star-bound}
\end{equation}
This inequality further motivates us to bound $e(s)$. 
Towards this, note that  $e(s)$ satisfies 
\begin{align}
\sum_{s\in\mathcal{S}}e\left(s\right) & =\frac{\left(1-\gamma\right)^{2}}{\gamma}\sum_{i=K}^{\infty}i\gamma^{i}\rho(P_{\pi^{\star}})^{i}1 \notag\\
 & =\frac{\left(1-\gamma\right)^{2}}{\gamma}\sum_{i=K}^{\infty}i\gamma^{i}=\frac{1-\gamma}{\gamma}\left(\sum_{i=K}^{\infty}i\gamma^{i}-\gamma\sum_{i=K}^{\infty}i\gamma^{i}\right)\notag\\
 & =\frac{1-\gamma}{\gamma}\left(\sum_{i=K}^{\infty}i\gamma^{i}-\sum_{i=K+1}^{\infty}\left(i-1\right)\gamma^{i}\right)\notag\\
 & =\frac{1-\gamma}{\gamma}\left(K\gamma^{K}+\sum_{i=K+1}^{\infty}\gamma^{i}\right)\notag\\
 & =\frac{1-\gamma}{\gamma}\left(K\gamma^{K}+\frac{\gamma^{K+1}}{1-\gamma}\right)\notag\\
 & \leq2C_{K}\gamma^{K-1}\log\frac{ST}{\delta}+\gamma^{K}
	\lesssim\frac{\delta}{ST^{4}}
	\label{eq:sum-es-bound-VR}
\end{align}
with $1$ the all-one vector, 
where the second line holds since $\rho(P_{\pi^{\star}})^{i}$ remains a probability vector (and hence $\rho(P_{\pi^{\star}})^{i}1=1$).
Here, the last line follows from our assumption that $\gamma\geq1/2$
and the fact that 
\[
\gamma^{K}=e^{K\log\left[1-\left(1-\gamma\right)\right]}\leq e^{-K\left(1-\gamma\right)}\leq e^{-C_{K}\log(ST/\delta)}=\left(\frac{\delta}{ST}\right)^{-C_{K}}\leq\frac{\delta^{2}}{S^{2}T^{5}} ,
\]
provided that $C_{K}\geq5$.

%Consequently, we arrive at 
%\begin{equation}
%d_{\widetilde{\rho}}^{\star}\left(s\right)\lesssim d_{\rho}^{\star}\left(s\right)\log\frac{ST}{\delta}+e\left(s\right),
%	\label{eq:d-star-star-bound}
%\end{equation}
%where 
%\begin{align*}
%\sum_{s\in\mathcal{S}}e\left(s\right) & \leq\frac{\delta}{ST^{4}}.
%\end{align*}
%%
%

We are now ready to establish the claim of this lemma. 
Substituting the bounds (\ref{eq:d-star-star-bound}) and \eqref{eq:sum-es-bound-VR} into \eqref{eq:sum-lemma-d-star-star-VF} leads to
\begin{align*}
\left\langle d_{\widetilde{\rho}}^{\star},V\right\rangle  & \lesssim\left\langle d_{\rho}^{\star},V\right\rangle \log\frac{ST}{\delta}+\sum_{s\in\mathcal{S}}e\left(s\right)\left\Vert V\right\Vert _{\infty}\\
 & \lesssim\left\langle d_{\rho}^{\star},V\right\rangle \log\frac{ST}{\delta}+\frac{\delta}{ST^{4}}\left\Vert V\right\Vert _{\infty}.
\end{align*}
As a result, one can readily conclude that 
\[
\sum_{j=0}^{\infty}\left[\gamma\left(1+\frac{1}{H}\right)^{3}\right]^{j}\left\langle \widetilde{\rho}P_{\pi^{\star}}^{j},V\right\rangle \lesssim\frac{1}{1-\gamma}\left\langle d_{\rho}^{\star},V\right\rangle \log\frac{ST}{\delta}+\frac{\delta}{ST^{4}\left(1-\gamma\right)}\left\Vert V\right\Vert _{\infty}.
\]